\documentclass[doublecolumn]{IEEEtran}

\usepackage{ragged2e}
\usepackage{cite}
\usepackage{cuted}
\usepackage{caption}
\usepackage{comment}
\usepackage{multirow}
\usepackage{makecell}
\ifCLASSOPTIONcompsoc
\usepackage[caption=false, font=normalsize, labelfont=sf, textfont=sf]{subfig}
\else
\usepackage[caption=false, font=footnotesize]{subfig}

\ifCLASSINFOpdf
  \usepackage[pdftex]{graphicx}
\else
  \usepackage[dvips]{graphicx}
\fi
\usepackage{amsmath}
\usepackage{amssymb}
\usepackage{amsthm}
\makeatletter

\newcommand{\Rmnum}[1]{\expandafter\@slowromancap\romannumeral #1@}
\makeatother

\newtheorem{lemma}{Lemma}
\newtheorem{remark}{Remark}
\newtheorem{assumption}{Assumption}

\newtheorem{proposition}{Proposition}

\usepackage{graphicx}

\usepackage{bbding}

\usepackage{extarrows}
\usepackage{float}
\usepackage{bm}
\usepackage{algorithmic}
\usepackage{array}
\usepackage[table]{xcolor}
\usepackage{booktabs}
\usepackage{algorithm}
\usepackage{algorithmic}
\usepackage{ wasysym }

\usepackage{threeparttable}

\usepackage{threeparttable}

\usepackage[explicit]{titlesec}
 
\newenvironment{shrinkfix}
{ \bgroup
\addtolength\abovedisplayshortskip{-0.5ex}
\addtolength\abovedisplayskip{-0.5ex}
\addtolength\belowdisplayshortskip{-0.5ex}
\addtolength\belowdisplayskip{-0.5ex}}
{\egroup\ignorespacesafterend}

\ifCLASSOPTIONcompsoc
  \usepackage[caption=false,font=normalsize,labelfont=sf,textfont=sf]{subfig}
\else
  \usepackage[caption=false,font=footnotesize]{subfig}
\fi

\ifCLASSOPTIONcaptionsoff
 \usepackage[nomarkers]{endfloat}
 \let\MYoriglatexcaption\caption
 \renewcommand{\caption}[2][\relax]{\MYoriglatexcaption[#2]{#2}}
\fi

\usepackage{url}

\allowdisplaybreaks[4]

\begin{document}

    \title{A Unified Framework for Fair Spectral Clustering With Effective Graph Learning 
}
\author{
Xiang~Zhang,~\IEEEmembership{Student~Member,~IEEE,} ~Qiao~Wang,~\IEEEmembership{Senior~Member,~IEEE}

\thanks{The authors are with the School of Information Science and Engineering, Southeast University, Nanjing 210096, China (e-mail: xiangzhang369@seu.edu.cn; qiaowang@seu.edu.cn).
 
}

}

\maketitle

\begin{abstract}
We consider the problem of spectral clustering under group fairness constraints, where samples from each sensitive group are approximately proportionally represented in each cluster. Traditional fair spectral clustering (FSC) methods consist of two consecutive stages, i.e.,  performing fair spectral embedding on a \textit{given} graph and conducting $k$means to obtain discrete cluster labels. However, in practice, the graph is usually unknown, and we need to construct the underlying graph from potentially noisy data, the quality of which inevitably affects subsequent fair clustering performance. Furthermore, performing FSC through separate steps breaks the connections among these steps, leading to suboptimal results. To this end, we first theoretically analyze the effect of the constructed graph on FSC. Motivated by the analysis, we propose a novel graph construction method with a node-adaptive graph filter to learn graphs from noisy data. Then, all independent stages of conventional FSC are integrated into a single objective function, forming an end-to-end framework that inputs raw data and outputs discrete cluster labels. An algorithm is developed to jointly and alternately update the variables in each stage. Finally, we conduct extensive experiments on synthetic, benchmark, and real data, which show that our model is superior to state-of-the-art fair clustering methods.

\end{abstract}


\section{Introduction}
\label{sec:introduction}
\IEEEPARstart{C}{lustering} is an unsupervised task that aims to group samples with common attributes and separate dissimilar samples. It has numerous practical applications, e.g., image processing  \cite{lei2018superpixel}, remote sensing \cite{xie2018unsupervised}, and bioinformatics \cite{kiselev2019challenges}. Existing clustering methods include $k$means \cite{likas2003global}, spectral clustering (SC) \cite{von2007tutorial}, hierarchical clustering \cite{xie2020hierarchical}. Among these methods, SC is a graph-based method utilizing the topological information of data and usually obtains better performance when handling complex high-dimensional datasets \cite{von2007tutorial}.

Recently, many concerns have arisen regarding fairness when performing clustering algorithms. For example, in loan applications, applicants are grouped into several clusters to support cluster-specific loan policies. However, clustering results could be affected by sensitive factors such as race and gender \cite{chouldechova2018frontiers}, even if the clustering algorithms do not consider sensitive attributes. Unfair clustering can lead to discriminatory outcomes, such as a specific group being more likely to be denied a loan. Therefore, there is a growing need for fair clustering methods unbiased by sensitive attributes. In the literature,  \cite{chierichetti2017fair} first introduces the notion of group fairness into clustering.  As illustrated in Fig.\ref{Fig-illustration}, given data with sensitive attributes, fair clustering aims to partition the data into clusters, where samples in every sensitive group are approximately proportionally represented in each cluster \cite{chierichetti2017fair}. In this way, every sensitive group is treated fairly. Following \cite{chierichetti2017fair}, \cite{bera2019fair} generalizes the definition of fair clustering, \cite{backurs2019scalable} proposes a scalable fair clustering algorithm, and \cite{ziko2021variational} applies the variational method to fair clustering. Furthermore, fairness constraints are also incorporated into deep clustering methods that leverage deep neural networks to partition data \cite{zeng2023deep,li2020deep}.

\begin{figure}[t] 
    \centering
       \includegraphics[width=0.9\linewidth]{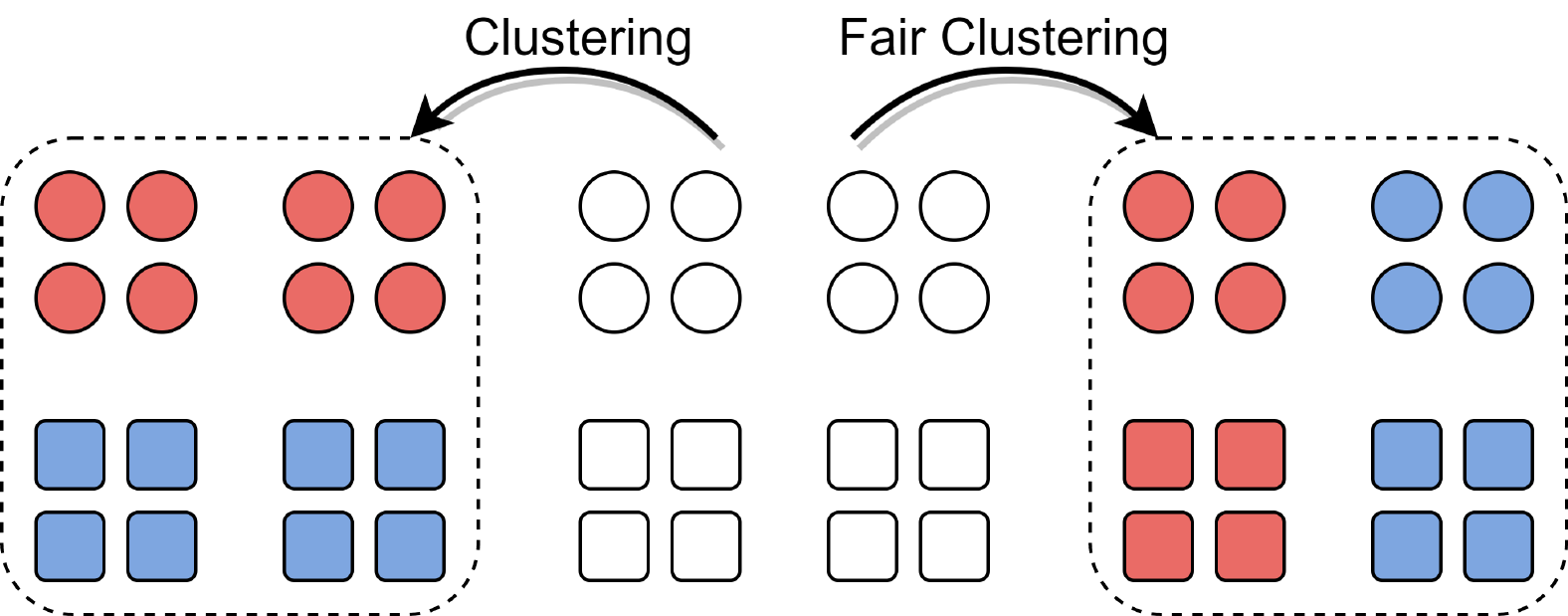}
    	\caption{The illustration of fair clustering. Given data points of two sensitive groups (squares and circles),  fair clustering partitions them into two clusters (blue and red), where samples of each group are proportionally represented in each cluster.
    	}
    	\label{Fig-illustration}
\end{figure}


Here, we consider the problem of fair spectral clustering (FSC). The first work discussing FSC is \cite{kleindessner2019guarantees}, which designs a fairness constraint for SC according to the definition of group fairness in  \cite{chierichetti2017fair}. A scalable algorithm is proposed in \cite{wang2023scalable} to solve the model in \cite{kleindessner2019guarantees}, and \cite{li2023spectral} considers group fairness of normalized-cut graph partitioning. In \cite{gupta2021protecting}, individual fairness is considered in SC, which utilizes a representation graph to encode sensitive attributes and requires the neighbors of a node in the graph to be approximately proportionally represented in the clusters. Recently, \cite{wang2022ifig} proposes a fair multi-view SC method. However, existing FSC models are built on a \textit{given} similarity graph, which may not be available in practice. Thus, before proceeding with FSC algorithms, it is necessary to construct a graph from raw data. That is, a complete FSC method typically consists of three subtasks. First, a similarity graph is constructed from raw data. Second, spectral embedding under fairness constraints is performed on the graph to obtain a continuous cluster indicator matrix. Third, conducting $k$means on the continuous matrix to obtain discrete cluster labels.

Although feasible, the traditional FSC paradigm still has the following problems to be addressed.
(\romannumeral1) The quality of the constructed graph inevitably affects subsequent fair clustering performance, but this has not been explored theoretically. Additionally, noisy observations make it more difficult to construct accurate graphs. (\romannumeral2) The post-processing discretization $k$means is sensitive to the initial cluster centers and could cause far deviation from the true discrete results \cite{huang2013spectral}. (\romannumeral3) Performing the subtasks separately breaks the connections among graph construction, fair spectral embedding, and discretization, leading to suboptimal fair clustering results. For example, independent graph construction may fail to find the optimal graph for fair clustering \cite{kang2018unified}. Furthermore, independent spectral embedding is inferior to joint optimization of graph construction and spectral embedding \cite{kang2017twin}.

To address the above issues, we propose a unified FSC model based on group fairness, which is an end-to-end framework that inputs observed data and outputs discrete cluster labels. Specifically, we first theoretically analyze how the estimated graphs affect FSC, demonstrating that accurate graphs are crucial to improve fair clustering performance. Motivated by the analysis, we propose a novel graph construction method to learn graphs from observed data under the smoothness assumption. Our approach incorporates a node-adaptive graph filter to denoise and produce smooth signals from potentially noisy data. Second, we introduce the group fairness constraint into traditional spectral embedding to guarantee fair clustering results. Third, we utilize spectral rotation instead of $k$means as the discretization operation since it can produce discrete results with smaller discrepancies from the true labels. Finally, all subtasks are integrated into a single objective function to avoid the sub-optimality caused by separate optimization.


In summary, the contributions of this study are as follows.

\begin{enumerate}
\item[$\bullet$]
We theoretically analyze the impact of the estimated graph on fair clustering errors, justifying the necessity of an accurate graph to improve FSC performance. Motivated by the analysis, we propose a graph construction method to learn accurate graphs as inputs to FSC.

\item[$\bullet$]
We propose a  unified FSC model integrating graph construction, fair spectral embedding, and discretization into a single objective function. Our model is an end-to-end framework that inputs observed data and outputs discrete fair clustering results and a similarity graph.

\item[$\bullet$]
We develop an algorithm to solve the objective function of our model.  Compared with separate optimization, our algorithm updates all variables jointly and alternately, leading to an overall optimal solution for all subtasks.

\item[$\bullet$]
We conduct extensive experiments on synthetic, benchmark, and real data to test the proposed FSC model. Experimental results demonstrate that our model outperforms state-of-the-art fair clustering models. 

\end{enumerate}

\textbf{Organization:} The rest of this paper is organized as follows. Section \ref{sec:related work} presents some related works. Background information is introduced in Section \ref{sec:background}.  We propose our unified FSC framework in Section \ref{sec:formulation}. Then, the proposed algorithm is provided in Section \ref{sec:algorithm}. We conduct experiments to test the proposed FSC method in Section \ref{sec:Experiments}. Finally, concluding remarks are presented in Section \ref{sec:Conclusion}.

\textbf{Notations:}
Throughout this paper, vectors, matrices, and sets are written in bold lowercase, bold uppercase letters, and calligraphic uppercase letters,  respectively.  Given a matrix $\mathbf{B}$, $\mathbf{B}_{[i,:]}, \mathbf{B}_{[:,j]}$, and $  \mathbf{B}_{[ij]}$ denote the $i$-th row, the $i$-th column, and the $(i,j)$ entry of $\mathbf{B}$, respectively. $\mathbf{B}\geq0$ means all elements of $\mathbf{B}$ are non-negative. Furthermore,  $\mathrm{diag}(\mathbf{B})$  and $\mathrm{diag}_{\mathrm{0}}(\mathbf{B})$ mean converting the diagonal elements of $\mathbf{B}$ to a vector and  setting the diagonal entries of  $\mathbf{B}$ to zeros. The vectors $\mathbf{1}$, $\mathbf{0}$, and matrix $\mathbf{I}$ represent all-one vectors, all-zero vectors, and identity matrices, respectively. Moreover, $\lVert \cdot \rVert_{\mathrm{F}}$, $\lVert \cdot \rVert_{1,1}$, and $\lVert \cdot \rVert_{q}$ are the Frobenius norm, element-wise $\ell_1$ norm, and $\ell_q$ norm of a vector (matrix), respectively.  The notations $\dag$, $\circ$, and  $\mathrm{Tr}(\cdot)$ are pseudo inverse, Hadamard product, and trace operator, respectively. Given a set $\mathcal{B}$, $|\mathcal{B}|$ is the number of elements in $\mathcal{B}$. Finally, $\mathbb{R}$ and $\mathbb{S}$ are the domain of real values and symmetric matrices whose dimensions depend on the context.

\section{Related Work}
\label{sec:related work}
\subsection{Graph Learning Methods For (Fair) SC}
Graph learning (GL) aims to infer the graph topology behind observed data, a prerequisite step for (fair) SC when similarity graphs are unavailable. Traditionally, graphs are constructed via some direct rules, such as $k$-nearest-neighborhood ($k$-NN), $\varepsilon$-nearest-neighborhood ($\varepsilon$-NN) \cite{huang2015new}, and sample correlation methods like Pearson correlation (PC). These methods may be limited in capturing similarity relationships between data pairs \cite{peng2023jgsed}. Thus, many works attempt to learn graphs from data adaptively, including the sparse representation (SR) method \cite{elhamifar2013sparse} and the low-rank representation method \cite{liu2012robust}. The emergence of adaptive neighbourhood graph learning (ANGL) \cite{nie2014clustering} provides a new way that uses the probability of two samples being adjacent to measure the similarity between them. In \cite{gao2022possibilistic},  a possibilistic neighbourhood graph is proposed, an improved version of \cite{nie2014clustering}. Recently, with the rise of graph signal processing (GSP) \cite{shuman2013emerging}, many works attempt to learn graphs from the perspective of signal processing. One of the widely-used GSP-based GL methods postulates that signals are smooth over the corresponding graphs \cite{dong2016learning}. Intuitively, a smooth graph signal means the signal values of two connected nodes are similar\cite{kalofolias2016learn}, which is also a fundamental principle of SC \cite{von2007tutorial}. Many methods are dedicated to learning graphs from smooth signals \cite{dong2019learning}. However, limited to our understanding, applying smoothness-based GL to SC has yet to be thoroughly explored, let alone FSC.

\subsection{Unified SC Models}
Many works focus on establishing a unified model for SC, which can be roughly divided into three categories. The first one integrates graph construction and spectral embedding \cite{kang2017twin, nie2020self, nie2014clustering}. They use an independent discretization step as post-processing. The second one is based on a given similarity graph. They integrate spectral embedding and discretization \cite{pang2018spectral, yang2016unified, huang2021new}. The third category integrates all three stages into a single objective function \cite{kang2018unified, peng2023jgsed, han2018discrete, tang2022unified, zhang2022one}. Our model differs from these models in two main ways. (\romannumeral1) Our framework utilizes a new graph construction method. (\romannumeral2) We further consider fairness issues in clustering tasks.


\section{Background}
\label{sec:background}
This section presents background information, including SC under group fairness constraints and spectral rotation.

\subsection{SC Under Group Fairness Constraints}
Given an undirected graph  $\mathcal{G}=\{\mathcal{V},\mathcal{E}\}$ of $D$ vertices, where $\mathcal{V}$ and $\mathcal{E}$ are the sets of vertices and edges of $\mathcal{G}$, respectively, its adjacency matrix $\mathbf{W} \in \mathbb{S}^{D\times D}$ is a symmetric matrix with zero diagonal entries and non-negative off-diagonal elements if the graph has non-negative edge weights and no self-loops. The Laplacian matrix  of $\mathcal{G}$ is defined as  $\mathbf{L} = \mathbf{D} - \mathbf{W}$, where $\mathbf{D}\in\mathbb{S}^{D\times D}$ is a diagonal matrix satisfying $\mathbf{D}_{[ii]} = \sum_{j=1}^D \mathbf{W}_{[ij]}$. Unnormalized SC aims to partition $D$ nodes into $K$ disjoint clusters $\mathcal{C}_1,...,\mathcal{C}_K$, where  $\mathcal{V} = \mathcal{C}_1 \cup ... \cup \mathcal{C}_K $, and $\mathcal{C}_k$ is the set containing nodes in the $k$–th cluster. The problem of unnormalized SC  is equivalent to minimizing the $\mathrm{RatioCut}$ objective function \cite{von2007tutorial}, i.e., 
\begin{shrinkfix}
\begin{align}
\mathrm{RatioCut} (\mathcal{C}_1,...,\mathcal{C}_K)= \sum_{k=1}^K\frac{\mathrm{Cut}(\mathcal{C}_k, \mathcal{V} \setminus \mathcal{C}_k)}{|\mathcal{C}_k|},
    \label{eq-prelim-ratiocut}
\end{align}
\end{shrinkfix}
where $ \mathcal{V} \setminus \mathcal{C}_k$ contains all nodes in $\mathcal{V}$ except those in $\mathcal{C}_k$, and 
\begin{shrinkfix}
\begin{align}
\mathrm{Cut}(\mathcal{C}_k, \mathcal{V} \setminus \mathcal{C}_k) = \sum_{i\in\mathcal{C}_k,j\in\mathcal{V} \setminus \mathcal{C}_k } \mathbf{W}_{[ij]}.
    \label{eq-prelim-cut}
\end{align}
\end{shrinkfix}
Let $\widetilde{\mathbf{U}} \in \mathbb{R}^{D\times K}$ be 
\begin{shrinkfix}
\begin{align}
\widetilde{\mathbf{U}}_{[ik]}= 
\begin{cases}
\frac{1}{\sqrt{|\mathcal{C}_k|}} & i\in\mathcal{C}_k\\
0 & i\notin\mathcal{C}_k
\end{cases}.
    \label{eq-prelim-U}
\end{align}
\end{shrinkfix}
Then, minimizing the $\mathrm{RatioCut}$ objective function \eqref{eq-prelim-ratiocut} is equivalent to solving the following problem \cite{von2007tutorial} 
\begin{shrinkfix}
\begin{align}
 \underset{\widetilde{\mathbf{U}}}{\min}\; \mathrm{Tr}(\widetilde{\mathbf{U}}^{\top}\mathbf{L}\widetilde{\mathbf{U}}), \;\;\mathrm{s.t.}\; \widetilde{\mathbf{U}} \;\text{is of form \eqref{eq-prelim-U}}.
    \label{eq-prelim-5}
\end{align}
\end{shrinkfix}
 Due to the discrete constraint of \eqref{eq-prelim-U}, problem \eqref{eq-prelim-5} is NP-hard. In practice, problem 
 \eqref{eq-prelim-5} is usually relaxed to
\begin{shrinkfix}
\begin{align}
\underset{\mathbf{U}}{\min}\; \mathrm{Tr}(\mathbf{U}^{\top}\mathbf{L}\mathbf{U}), \;\;\mathrm{s.t.}\;\mathbf{U}^{\top}\mathbf{U} = \mathbf{I}, 
    \label{eq-prelim-6}
\end{align}
\end{shrinkfix}
where $\mathbf{U}\in\mathbb{R}^{D\times K}$ is a relaxed continuous clustering label matrix, and $\mathbf{U}^{\top}\mathbf{U} = \mathbf{I}$ is adopted to avoid trivial solutions. 
The process of solving \eqref{eq-prelim-6} is called spectral embedding.  
After obtaining $\mathbf{U}^*$, a common practice is to apply $k$means to the rows of $\mathbf{U}^*$ to yield final discrete clustering labels $\mathbf{Q}$, where $\mathbf{Q} \in\{0,1\}^{D\times K}$ is a binary cluster indicator matrix. The only non-zero element of the $i$-th row of $\mathbf{Q}$ indicates the cluster membership of the $i$-th node of $\mathcal{G}$.

Fair spectral clustering groups the vertices of $\mathcal{G}$ by considering fairness. If the nodes of $\mathcal{G}$ belong to $S$ sensitive groups $\mathcal{D}_1,...,\mathcal{D}_S$, where $\mathcal{D}_s$ contains the nodes of the $s$-th sensitive group, we define the $\mathrm{Balance}$ of cluster $\mathcal{C}_k$ as \cite{chierichetti2017fair}
\begin{shrinkfix}
\begin{align}
\mathrm{Balance}(\mathcal{C}_k) = \underset{s\neq s^{\prime}\in [S]}{\min}\; \frac{\left\lvert \mathcal{D}_{s} \cap \mathcal{C}_k\right\rvert}{\left\lvert \mathcal{D}_{s^{\prime}} \cap \mathcal{C}_k\right\rvert} \in [0,1],
    \label{eq-prelim-7}
\end{align}
\end{shrinkfix}
where $[S]:=\{1,...,S\}$. The higher the $\mathrm{Balance}$ of each cluster, the fairer the clustering \cite{chierichetti2017fair}. It is not difficult to check that $\underset{k\in[K]}{\min}\mathrm{Balance}(\mathcal{C}_k)\leq \underset{s\neq s^{\prime}\in [S]}{\min}\;|\mathcal{D}_s|/|\mathcal{D}_{s^{\prime}}|$. Thus, this notion of fairness is asking for a clustering where \textit{the fraction of different sensitive groups in each cluster is approximately the same as that of the entire dataset $\mathcal{V}$} \cite{kleindessner2019guarantees}, which is also called group fairness. To incorporate this fairness notion into SC, a group-membership vector $\mathbf{f}_s\in \{0,1\}^D$ of $\mathcal{D}_s$ is defined, where $(\mathbf{f}_s)_{[i]} = 1$ if $i\in \mathcal{D}_s$ and $(\mathbf{f}_s)_{[i]} = 0$ otherwise, for $s\in [S]$ and $i\in [D]$. Then, we have the following lemma.
\begin{lemma}
(Fairness constraint as linear constraint on $\widetilde{\mathbf{U}}$ \cite{kleindessner2019guarantees})
    Let $\mathcal{V} = \mathcal{C}_1 \cup ... \cup \mathcal{C}_K $ be a clustering that is
encoded as in \eqref{eq-prelim-U}. We have, for every $k\in[K]$
\begin{shrinkfix}
\begin{align}
\forall s\in[S] : \frac{|\mathcal{D}_s\cap\mathcal{C}_k|}{|\mathcal{C}_k|} = \frac{|\mathcal{D}_s|}{D} \Leftrightarrow \mathbf{F}^{\top}\widetilde{\mathbf{U}} = \mathbf{0}, 
    \label{eq-prelim-lemma-fair}
\end{align}
\end{shrinkfix}
where  $\mathbf{F} \in \mathbb{R}^{D\times (S-1)}$ is a matrix satisfying $\mathbf{F}_{[:,s]} = \mathbf{f}_s - (|\mathcal{D}_s|/D)\cdot\mathbf{1}, s\in [S-1]$. 
    \label{lemma-fair}
\end{lemma}
Lemma \ref{lemma-fair} states that the proportional representation of all sensitive attribute samples in each cluster can be interpreted as a linear constraint $\mathbf{F}^{\top}\widetilde{\mathbf{U}} = \mathbf{0}$. Under this fairness constraint, unnormalized SC is equivalent to the following problem  
\begin{shrinkfix}
\begin{align}
 &\underset{\widetilde{\mathbf{U}}}{\min}\; \mathrm{Tr}(\widetilde{\mathbf{U}}^{\top}\mathbf{L}\widetilde{\mathbf{U}}),\;\;\mathrm{s.t.}\; \widetilde{\mathbf{U}} \;\text{is of form \eqref{eq-prelim-U}},\;\mathbf{F}^{\top}\widetilde{\mathbf{U}} = \mathbf{0}.
    \label{eq-prelim-8-1}
\end{align}
\end{shrinkfix}
Similarly, we can relax \eqref{eq-prelim-8-1} to
\begin{shrinkfix}
\begin{align}
 &\underset{\mathbf{U}}{\min}\; \mathrm{Tr}(\mathbf{U}^{\top}\mathbf{L}\mathbf{U}), \;\; \mathrm{s.t.}\;\mathbf{U}^{\top}\mathbf{U} = \mathbf{I}, \;\mathbf{F}^{\top}\mathbf{U} = \mathbf{0}.
    \label{eq-prelim-8}
\end{align}
\end{shrinkfix}
Following traditional SC, existing FSC models perform $k$means on rows of $\mathbf{U}$ to obtain discrete cluster labels $\mathbf{Q}$.

\subsection{Spectral Rotation}
\label{sec:spectral rotation}
Spectral rotation \cite{huang2013spectral} is an alternative to $k$means for obtaining discrete clustering results from continuous labels $\mathbf{U}$, which is formulated as
\begin{shrinkfix}
\begin{align}
        &\underset{\mathbf{Q}, \mathbf{R}}{\min}\; \lVert \mathbf{Q} - \mathbf{U}\mathbf{R}\rVert_{\mathrm{F}}^2,
        &\mathrm{s.t.}\;\mathbf{R}^{\top}\mathbf{R} = \mathbf{I}, \mathbf{Q}\in \mathcal{I},
    \label{eq-form-1}
\end{align}
\end{shrinkfix}
where the set $\mathcal{I}$ contains all discrete cluster indicator matrices, and $\mathbf{R}\in \mathbb{R}^{K\times K}$ is an orthonormal matrix. According to the spectral solution invariance property \cite{huang2013spectral}, if $\mathbf{U}$ is a solution of \eqref{eq-prelim-6}, $\mathbf{U}\mathbf{R}$ is another solution. A suitable $\mathbf{R}$ could facilitate $\mathbf{U}\mathbf{R}$ as close to $\mathbf{Q}$ as possible. In contrast, $k$means is performed directly on $\mathbf{U}$ obtained from spectral embedding, which may far deviate from the real discrete results. Thus, spectral rotation usually achieves superior performance than $k$means  \cite{huang2013spectral}.

\section{Model Formulation}
\label{sec:formulation}


In this section, we first theoretically analyze the impact of the constructed graph on FSC, which justifies an accurate graph for improving FSC performance. Then, we propose a novel graph construction method to learn graphs from potentially noisy observed data. Next, we integrate graph construction, fair spectral embedding, and discretization into an end-to-end framework. Finally, we analyze the connections between our model and existing works.

\subsection{Why We Need An Accurate Graph?}
We first introduce a variant of the stochastic block model \cite{holland1983stochastic} (vSBM) to generate random graphs with cluster structures and sensitive attributes \cite{kleindessner2019guarantees}. This model assumes that there are two or more meaningful ground-truth clusterings of the observed data, and only one of them is fair. Assume that $\mathcal{V}$ comprises $S$ sensitive groups and is partitioned into $K$ clusters such that $|\mathcal{D}_s \cap \mathcal{C}_k|/|\mathcal{C}_k| = \zeta_s, s\in [S], k\in[K]$, for $\zeta_s\in(0,1)$ with $\sum_{s=1}^S \zeta_s = 1$. Based on the clusters and sensitive groups, we construct a random graph by connecting two vertices $i$ and $j$ with a probability $\mathrm{Pr}(i,j)$ that depends on the clusters and sensitive groups of $i$ and $j$. We define     
\begin{shrinkfix}
    \begin{align} 
    \mathrm{Pr}(i,j) =
\begin{cases}
a, & \pi_C(i) =  \pi_C(j), \; \pi_S(i) =  \pi_S(j)\\
b, & \pi_C(i) \neq  \pi_C(j),  \; \pi_S(i) =  \pi_S(j)\\
c, &  \pi_C(i) =  \pi_C(j),  \; \pi_S(i) \neq  \pi_S(j) \\
d, &  \pi_C(i) \neq  \pi_C(j),  \; \pi_S(i) \neq  \pi_S(j),
\end{cases}
    \label{eq-exper-1}
    \end{align}
\end{shrinkfix}
where $\pi_C: [D]\to [K]$ and $\pi_S: [D]\to [S]$ are two functions that assign a node $i\in\mathcal{V}$ to one of the clusters and sensitive groups, respectively. Let $\mathbf{L}^*$ be the real graph Laplacian matrix generated by the vSBM method and $\widehat{\mathbf{L}}$ be the Laplacian matrix estimated by any graph construction method. The matrix $\widehat{\mathbf{L}}$ is used as the input to fair spectral embedding in \eqref{eq-prelim-8}, and spectral rotation is utilized to obtain discrete cluster labels. Our goal is to derive a fair clustering error bound related to the estimation error between $\widehat{\mathbf{L}}$ and $\mathbf{L}^*$. Let us make some assumptions.
\begin{assumption}
    Let $\widehat{\mathbf{U}}$ be a continuous cluster indicator matrix estimated from $\widehat{\mathbf{L}}$ via \eqref{eq-prelim-8}. For a given constant $\epsilon>0$, the $\widehat{\mathbf{Q}}$ and $\widehat{\mathbf{R}}$ estimated by spectral rotation satisfies
\begin{shrinkfix}
    \begin{align} 
        \lVert \widehat{\mathbf{Q}} -  \widehat{\mathbf{U}}\widehat{\mathbf{R}}\rVert_{\mathrm{F}}^2 \leq (1+\epsilon) \underset{\mathbf{Q}\in\mathcal{I}, \mathbf{R}^{\top}\mathbf{R} = \mathbf{I}}{\min}\lVert {\mathbf{Q}} -  \widehat{\mathbf{U}}{\mathbf{R}}\rVert_{\mathrm{F}}^2.
    \label{eq-theo-res2-1}
    \end{align}
\end{shrinkfix}    
    \label{assump-1}
      \vspace{-1em}
\end{assumption}
\begin{assumption}
The ground-truth clustering and sensitive partitions of $\mathcal{V}$ satisfy
\begin{shrinkfix}
    \begin{align} 
    |\mathcal{D}_s| = \frac{D}{S}, \; |\mathcal{C}_k| = \frac{D}{K}, \;\frac{|\mathcal{D}_s\cap\mathcal{C}_k|}{|\mathcal{C}_k|} = \frac{1}{S}.
    \label{eq-theo-res2-2}
    \end{align}
\end{shrinkfix}   
    \label{assump-2}
      \vspace{-1em}
\end{assumption}
Assumption \ref{assump-1} is similar to the $(1+\epsilon)-$approximation of $k$means \cite{lei2015consistency}, which provides the estimation accuracy of spectral rotation. Assumption \ref{assump-2} is the same as that in Theorem 1 of \cite{kleindessner2019guarantees}, which is only made to facilitate theoretical analysis. In practice, Assumption 2 may be violated, which, however, does not affect the effectiveness of FSC algorithms \cite{kleindessner2019guarantees}. Based on the two assumptions, we have the following proposition.

\begin{proposition}  
Let $\mathbf{L}^*$ be the real Laplacian matrix of the random graph generated by the vSBM method with $a>b>c>d$ satisfying $a>r_1 \ln{D}/D$ for some $r_1 >0$, and $\widehat{\mathbf{L}}$ be the estimated Laplacian matrix from observed data. Assume that we run fair spectral embedding \eqref{eq-prelim-8} on  $\widehat{\mathbf{L}}$ and perform $(1+\epsilon)$ spectral rotation \eqref{eq-form-1} to obtain discrete cluster labels. Besides, let $\widehat{\pi}_C(i)$ be the assigned cluster label (after proper permutation) of node $i$, and define $\mathcal{M}_k:= \left\{i\in\mathcal{C}_k:\widehat{\pi}_C(i) \neq k  \right\}$ as the set of misclassified vertices of cluster $k$. Under Assumptions \ref{assump-1}-\ref{assump-2}, for every $r_2>0$, there exist constants $\widehat{C} = \widehat{C}(r_1,r_2)$ and $\widetilde{C} = \widetilde{C}(r_1,r_2)$ such that if
 \begin{shrinkfix}
    \begin{align} 
    \frac{aK^3\ln{D}}{D(c-d)^2} \leq\frac{\widehat{C}}{1+\epsilon},
    \label{eq-theo-res2-3-0}
    \end{align}
    \end{shrinkfix}
then with probability at least $1-D^{-r_2}$, the number of misclassified vertices, $\sum_{k=1}^K {|\mathcal{M}_k|}$, is at most 
 \begin{shrinkfix}
    \begin{align} 
  \underbrace{\frac{\widetilde{C} (1+\epsilon) aK^2\ln{D}}{(c-d)^2 }}_{ \mathrm{related\;to\;the\;vSBM\;}} + \underbrace
{\frac{512(4+2\epsilon)K^2}{D(c-d)^2}\lVert \mathbf{Z}^{\top}\mathbf{L}^*\mathbf{Z} - \mathbf{Z}^{\top}\widehat{\mathbf{L}}\mathbf{Z} 
    \rVert_{\mathrm{F}}^2}_{\mathrm{ related\;to\;graph\; estimation}},
    \label{eq-theo-res2-3}
    \end{align}
    \end{shrinkfix}  
where $\mathbf{Z}\in\mathbb{R}^{D\times (D-S+1)}$ is a matrix whose columns form the orthonormal basis of the nullspace of $\mathbf{F}^{\top}$.
\label{theo-clustering performance}
\end{proposition}
\begin{proof}
The proof is inspired by \cite{kleindessner2019guarantees}, but has two main differences. First, spectral rotation instead of $k$means is used to obtain discrete labels. Second, fair spectral embedding is based on an estimated graph rather than a known graph generated by the vSBM method. See Appendix \ref{appendix2} for details.
    \label{proof-theo-clustering performance}
\end{proof}

According to \cite{kleindessner2019guarantees}, the meaning of ``the number of misclassified vertices is at most $D_m$'' is that there exists a permutation of cluster indices such that the clustering results up to this permutation successfully predict all cluster labels but $D_{m}$ many vertices. Note that the error bound consists of two parts. The first one is caused by the difference between the expected and real graph produced by the vSBM method, which is similar to \cite{kleindessner2019guarantees}. 
The second part is related to the estimation error of graph construction methods. The fair constraint affects clustering performance via $\mathbf{Z}$, which is a matrix determined by sensitive group-membership matrix $\mathbf{F}$. For convenience,  $\mathbf{Z}^{\top}\mathbf{L}\mathbf{Z}$ is dubbed fair graph. Generally, the error bound in \eqref{eq-theo-res2-3} depends on $K$, $D$ and $\epsilon$. If we divide \eqref{eq-theo-res2-3} by $D$, we obtain the bound for the misclassification rate. The first part of the misclassification rate bound tends to zero as $D$ goes to infinity, meaning that if $\mathbf{L}^*$ is exactly estimated (the second part equals to zero), performing FSC via \eqref{eq-prelim-8} and spectral rotation is weakly consistent \cite{kleindessner2019guarantees}. However,  $\mathbf{L}^*$ usually cannot be estimated exactly, causing an additional error for subsequent fair clustering results. If the fair graph estimation error does not increase quadratically as $D$, the second part of the misclassification rate bound will also decay to zero. Proposition \ref{theo-clustering performance} illustrates that a well-estimated graph $\widehat{\mathbf{L}}$, which is close to $\mathbf{L}^*$, brings a small misclassification error bound. Thus, it motivates us to seek a more effective method to construct accurate graphs from observed data.

\subsection{The Proposed Graph Construction Method}
Given $N$ observed data $\mathbf{X}_{o}\in \mathbb{R}^{D\times N}$, we need to infer the underlying similarity graph topology as the input to FSC algorithms. However, contaminated data may lead to poor graph estimation performance,  as indicated in Proposition \ref{theo-clustering performance}, which degrades subsequent fair clustering performance. Therefore, we propose a method to learn graphs from potentially noisy data $\mathbf{X}_o$, which is formulated as 
\begin{shrinkfix}
    \begin{align}
    \underset{\mathbf{L}\in \mathcal{L},\mathbf{X},\bm{\upsilon}>0}{\mathrm{min}}\,\, &\frac{1} {N}\lVert 
    {\mathbf{\Upsilon}} (\mathbf{X}_o- \mathbf{X})\rVert_{\mathrm{F}}^2 + \frac{\xi} {N}\mathrm{Tr}(\mathbf{X}^{\top}\mathbf{L}\mathbf{X}) + \sum_{i=1}^D\frac{1}{\bm{\upsilon}_{[i]}} \notag\\
    &\underbrace{- \mathbf{1}^{\top}\log\left(\mathrm{diag}(\mathbf{L})\right) + {\beta} \lVert \mathrm{diag_0}(\mathbf{L})\rVert^2_{\mathrm{F}}}_{Reg(\mathbf{L})},
    \label{eq-prelim-1}
    \end{align}
\end{shrinkfix}
where  $ \mathcal{L}:= \left\{\mathbf{L}: \mathbf{L}\in \mathbb{S}^{D\times D},\, \mathbf{L}\mathbf{1} = \mathbf{0},\, \mathbf{L}_{[ij]}\leq 0,\,\, i \neq j \right\}$ contains all Laplacian matrices. Moreover, $\xi$ and $\beta$ are parameters, and $\bm{\upsilon}\in \mathbb{R}^{D}$ is a vector of adaptive weights. We let $\mathbf{\Upsilon} := \mathrm{diag}(\sqrt{\bm{\upsilon}})$, where $\sqrt{\bm{\upsilon}} = (\sqrt{\bm{\upsilon}_{[1]}},...,\sqrt{\bm{\upsilon}_{[D]}})^{\top}$. Eq.\eqref{eq-prelim-1} is a joint model of denoising and smoothness-based GL \cite{kalofolias2016learn}.

\textit{1) Denoising: } If $\mathbf{L}$ is fixed, the problem \eqref{eq-prelim-1} becomes
\begin{shrinkfix}
    \begin{align}
    \underset{\mathbf{X},\bm{\upsilon}}{\mathrm{min}}\,\, \frac{1} {N}\lVert {\mathbf{\Upsilon}}(\mathbf{X}_o- \mathbf{X})\rVert_{\mathrm{F}}^2 + \frac{\xi} {N}\mathrm{Tr}(\mathbf{X}^{\top}\mathbf{L}\mathbf{X})+ \sum_{i=1}^D\frac{1}{\bm{\upsilon}_{[i]}}
    \label{eq-prelim-1-1}.
    \end{align}
\end{shrinkfix}
The model is a node-adaptive graph filter, and $\bm{\upsilon}$ represents node weights. Specifically, given node weights $\bm{\upsilon}$, we have
\begin{shrinkfix}
    \begin{align}
    \underset{\mathbf{X}}{\mathrm{min}}\,\, \lVert {\mathbf{\Upsilon}}(\mathbf{X}_o- \mathbf{X})\rVert_{\mathrm{F}}^2 + {\xi}\mathrm{Tr}(\mathbf{X}^{\top}\mathbf{L}\mathbf{X})
    \label{eq-prelim-1-2}.
    \end{align}
\end{shrinkfix}
Taking the derivative
of \eqref{eq-prelim-1-2} and setting it to zero, we obtain
\begin{shrinkfix}
    \begin{align}
     \mathbf{X} = \left(\mathbf{\Upsilon}^{\top}\mathbf{\Upsilon} + {\xi}\mathbf{L}\right)^{-1}\mathbf{\Upsilon}^{\top}\mathbf{\Upsilon}\mathbf{X}_o 
     = {\left(\mathbf{I}+ {\xi}(\mathbf{\Upsilon}^{\top}\mathbf{\Upsilon})^{-1}\mathbf{L}\right)^{-1}}\mathbf{X}_o.  
     \label{eq-prelim-1-3}
    \end{align}
\end{shrinkfix}
We let $\mathbf{K} := \left(\mathbf{I}+ {\xi}(\mathbf{\Upsilon}^{\top}\mathbf{\Upsilon})^{-1}\mathbf{L}\right)^{-1}$, which is positive definite and has eigen-decomposition $\mathbf{K} = \mathbf{\Theta}\mathbf{\Lambda}\mathbf{\Theta}^{\top}$ with eigenvalues matrix $\mathbf{\Lambda}$ and eigenvectors matrix $\mathbf{\Theta}$. Moreover, $\mathbf{\Lambda} = \mathrm{diag}\left(\frac{1}{1+\xi \lambda_1}, ...., \frac{1}{1+\xi \lambda_D}\right)$, where $0 = \lambda_1\leq,...., \leq\lambda_D$ are the eigenvalues of $(\mathbf{\Upsilon}^{\top}\mathbf{\Upsilon})^{-1}\mathbf{L}$. From the perspective of graph spectral filtering (GFT) \cite{shuman2013emerging}, $\mathbf{K}\mathbf{X}_o = \mathbf{\Theta}\mathbf{\Lambda}\mathbf{\Theta}^{\top} \mathbf{X}_o$ can be interpreted as that the observed graph signals (columns of $\mathbf{X}_o$) are first transformed to the graph frequency domain via $\mathbf{\Theta}^{\top}$, attenuated GFT coefficients according to $\mathbf{\Lambda}$, and transformed back to the nodal domain via $\mathbf{\Theta}$. It is observed from $\mathbf{\Lambda}$ that the graph filter $\mathbf{K}$ is low-pass since the attenuation is stronger for larger eigenvalues. Thus, the graph filter can suppress the high-frequency component of raw data $\mathbf{X}_o$, alleviating the noises on the graph to output the ``noiseless" signal $\mathbf{X}$.

Our graph filter $\mathbf{K}$ differs from the  Auto-Regressive graph filter $\left(\mathbf{I}+ {\xi}\mathbf{L}\right)^{-1} $ \cite{li2019label} in that we assign each node an individual weight $\bm{\upsilon}_{[i]}, i=1,..., D$. The reason for using $\bm{\upsilon}$ is that the measurement noise of different nodes may be heterogeneous. If the $i$-th node signal (the  $i$-th row of $\mathbf{X}_o$) has a small noise scale, a large $\bm{\upsilon}_{[i]}$ should be assigned to the fidelity term of node $i$ in \eqref{eq-prelim-1-1} to ensure $\mathbf{X}_{[i,:]}$ is are close to the corresponding observation $(\mathbf{X}_o)_{[i,:]}$ \cite{pilavci2021graph}. When we cannot know the noise scale a priori, we can adaptively learn $\bm{\upsilon}$ from the data. Specifically, given $\mathbf{X}$, the problem \eqref{eq-prelim-1-1} becomes  
\begin{shrinkfix}
    \begin{align}
    \underset{\bm{\upsilon}>0}{\mathrm{min}}\,\, \frac{1} {N} \sum_{i=1}^D \bm{\upsilon}_{[i]} \lVert(\mathbf{X}_o)_{[i,:]} - \mathbf{X}_{[i,:]} \rVert_2^2 + \frac{1}{\bm{\upsilon}_{[i]}}
    \label{eq-prelim-1-4}.
    \end{align}
\end{shrinkfix}
Intuitively, solving \eqref{eq-prelim-1-4} will assign a large $\bm{\upsilon}_{[i]}$ to node $i$ if $\mathbf{X}_{[i,:]}$ is close to $(\mathbf{X}_o)_{[i,:]}$, as expected.

\textit{2) Graph learning: } If we have obtained the ``noiseless" signals $\mathbf{X}$ via the graph filter $\mathbf{K}$, the problem \eqref{eq-prelim-1} becomes
\begin{shrinkfix}
    \begin{align}
    \underset{\mathbf{L}\in \mathcal{L}}{\mathrm{min}}\,\, \frac{\xi} {N}\mathrm{Tr}(\mathbf{X}^{\top}\mathbf{L}\mathbf{X})+ {Reg(\mathbf{L})}.
    \label{eq-prelim-1-5}
    \end{align}
\end{shrinkfix}
The first Laplacian quadratic term of \eqref{eq-prelim-1-5} is equivalent to
\begin{shrinkfix}
    \begin{align}
    \frac{1}{N}\mathrm{Tr}(\mathbf{X}^{\top}\mathbf{L}\mathbf{X}) = \frac{1}{N} \sum_{n=1}^N \sum_{i,j=1}^D \mathbf{W}_{[ij]}\left(\mathbf{X}_{[in]}- \mathbf{X}_{[jn]}\right)^2,
    \label{eq-prelim-1-5-1}
    \end{align}
\end{shrinkfix}
which measures the average smoothness of data $\mathbf{X}$ over the graph $\mathbf{L}$ \cite{kalofolias2016learn}. The second term of \eqref{eq-prelim-1-5} contains regularizers that endow the learned graphs with desired properties. The  $\log$-degree term is to control node degree, and the Frobenius norm term is to control graph sparsity. Our model \eqref{eq-prelim-1-5} can learn a graph suitable for graph-based clustering tasks for the following reasons. (\romannumeral1) It is observed from \eqref{eq-prelim-1-5-1} that minimizing the smoothness is to seek a graph whose similar vertices (node signals) are closely connected, which is consistent with the fundamental principle of SC. (\romannumeral2) The $\log$-degree term can avoid isolated nodes, which is crucial for SC, especially for normalized SC \cite{von2007tutorial}. (\romannumeral3) The Frobenius norm term can lead to a sparse graph, which may remove redundant and noisy edges.

The model \eqref{eq-prelim-1-5} is similar to the ANGL method \cite{nie2014clustering} since both construct graphs by minimizing the smoothness. The main differences lie in three aspects. (\romannumeral1) Our model removes the sum-to-one constraint in the ANGL method\textemdash the degree of each node is forced to be one\textemdash since the constraint makes the output graphs sensitive to noisy points \cite{gao2022possibilistic}. Removing this constraint allows our model to capture more complex similarity relationships. (\romannumeral2) We add a $\log$-degree term to ensure the learned graph has no isolated nodes. (\romannumeral3) The input data of \eqref{eq-prelim-1-5} are those produced by a low-pass graph filter.

\textit{3) Discussion: }
We try to explain why our method is effective in constructing graphs from observed data. If data $\mathbf{X}_o$ have a clustering structure, they should follow the assumption of cluster and manifold, i.e., the data in the same cluster are close to each other. According to  \cite{pan2021multi}, smooth signals containing low-frequency parts tend to follow the cluster and manifold assumption. Thus, if $\mathbf{L}$ accurately represents the graph behind observed data, the denoising part of our model has two functions. First, it filters out the high-frequency components of the observed graph signals that correspond to noises. Second, it produces smooth signals that have a clearer clustering structure, which could facilitate subsequent clustering. To better illustrate the effectiveness of the node-adaptive filter, Fig.\ref{Fig-tnse} depicts the t-SNE \cite{van2008visualizing} results of our methods on the MNIST dataset, where four clusters correspond to four randomly selected digits. We can see that raw data are entangled together. In contrast, the denoised data $\mathbf{X}$ by the graph filter are clearly separated, meaning that the denoising part of our model can produce cluster-friendly signals. From the perspective of GL, our model \eqref{eq-prelim-1-5} learns a graph minimizing the smoothness of data, i.e., the nodes corresponding to similar signals are closely connected. Thus, the learned graph can effectively capture similarity relationships between data and preserve clustering structures. Consequently, the denoising operation and the smoothness-based GL can enhance each other collaboratively to bring a high-quality graph for subsequent fair clustering tasks.

\begin{figure}[t] 
    \centering
       \includegraphics[width=0.99\linewidth]{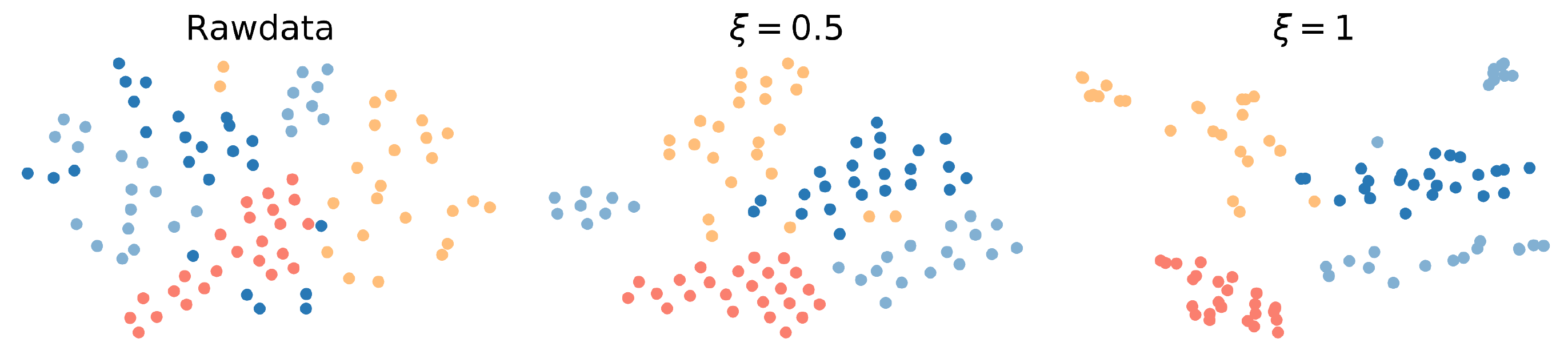}
    	\caption{The t-SNE results of MNIST with different $\xi$ values. 
    	}
    	\label{Fig-tnse}
\end{figure}

\begin{figure*}[t] 
    \centering
       \includegraphics[width=0.99\linewidth]{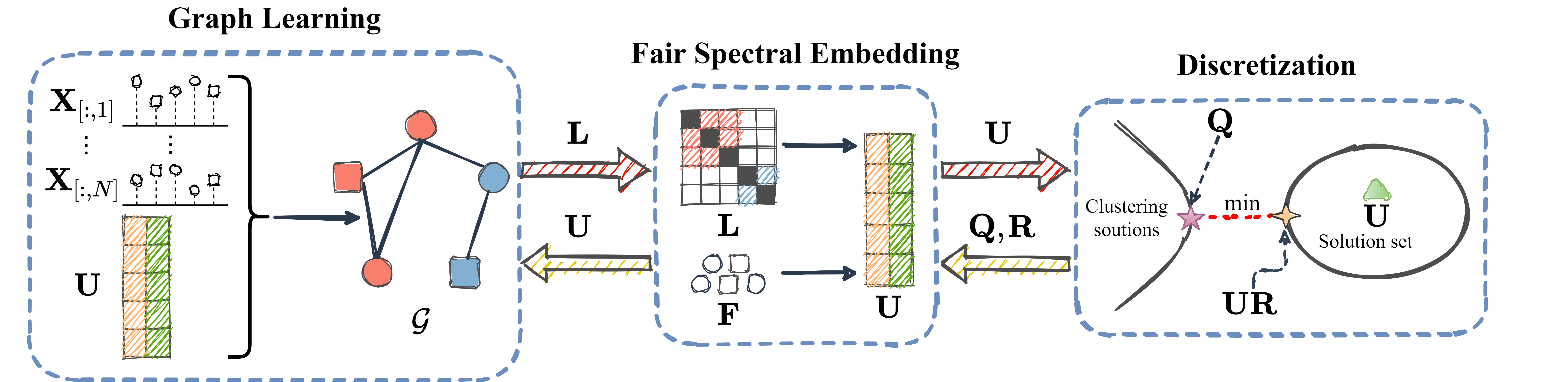}
    	\caption{The illustration of the proposed framework. 
    	}
    	\label{Fig-flowchart}
\end{figure*}

\subsection{The Unified FSC Model}
In this subsection, we build an end-to-end FSC framework that inputs observed data $\mathbf{X}_o$ and node attributes $\mathbf{F}$ and directly outputs discrete cluster labels.  As shown in Fig.\ref{Fig-flowchart}, our model consists of four modules, i.e., denoising, graph learning, fair spectral embedding, and discretization. First, we construct graphs from the observed data by the proposed method \eqref{eq-prelim-1}. Once we obtain $\mathbf{L}$, the Laplacian matrix together with $\mathbf{F}$ can be directly used to perform fair spectral embedding \eqref{eq-prelim-8} to obtain continuous clustering label matrix $\mathbf{U}$. Finally, we leverage spectral rotation \eqref{eq-form-1} instead of $k$means to obtain discrete cluster labels. In addition to the reason for superior performance as stated in Section \ref{sec:spectral rotation}, we utilize spectral rotation since it can be flexibly integrated into an end-to-end framework. Integrating all the above subtasks into a  single objective function, we obtain
\begin{shrinkfix}
\begin{align}
       & \underset{\mathbf{X},\mathbf{L}, \bm{\upsilon},\mathbf{Y},\mathbf{R},\mathbf{Q}}{\mathrm{min}}\,\, \frac{1} {N}\lVert 
    {\mathbf{\Upsilon}} (\mathbf{X}_o- \mathbf{X})\rVert_{\mathrm{F}}^2 + \frac{\xi}{N}\mathrm{Tr}(\mathbf{X}^{\top}\mathbf{L}\mathbf{X}) + Reg(\mathbf{L}) \notag\\
       &\;\;\;\;\;\;\;\;\;\;\;\;\;\;\;\;+ \sum_{i=1}^D\frac{1}{\bm{\upsilon}_{[i]}} + \mu \mathrm{Tr}(\mathbf{U}^{\top}\mathbf{L}\mathbf{U})+ \gamma \lVert \mathbf{Q} - \mathbf{U}\mathbf{R}\rVert_{\mathrm{F}}^2\notag\\
    & \mathrm{s.t.}\; \mathbf{L}\in \mathcal{L},\bm{\upsilon}>0, \mathbf{U}^{\top}\mathbf{U} = \mathbf{I}, \mathbf{F}^{\top}\mathbf{{U}} = \mathbf{0}, \mathbf{R}^{\top}\mathbf{R} = \mathbf{I}, \mathbf{Q} \in \mathcal{I},
    \label{eq-formulation-final}
\end{align}
\end{shrinkfix}
where $\mu$ and $\gamma$ are two parameters. All modules are not simply put together, but are bridged by two Laplacian quadratic terms, i.e., $\frac{\xi}{N}\mathrm{Tr}(\mathbf{X}^{\top}\mathbf{L}\mathbf{X})$ and $\mu \mathrm{Tr}(\mathbf{U}^{\top}\mathbf{L}\mathbf{U})$. First, $\frac{\xi}{N}\mathrm{Tr}(\mathbf{X}^{\top}\mathbf{L}\mathbf{X})$ can be viewed as the graph Tikhonov regularizer of the denoising task \eqref{eq-prelim-1-2} to output smooth signals \cite{li2019label}. On the other hand, it measures smoothness in the GL task to capture the similarity relationships between data. Second,  $\mu \mathrm{Tr}(\mathbf{U}^{\top}\mathbf{L}\mathbf{U})$ together with $\mathbf{F}^{\top}\mathbf{U} = \mathbf{0}$ performs fair spectral embedding as input to the discretization. It is also used to impose structural constraints on the constructed graph, which is discussed in the next subsection. The four modules are coupled with each other to achieve the overall optimal results for all subtasks.

To better understand how the fairness constraint works, we introduce a new variable matrix $\mathbf{Y}\in\mathbb{R}^{(D-S+1)\times K}$ and let $\mathbf{U} = \mathbf{Z}\mathbf{Y}$, where $\mathbf{Z}$ is the matrix defined in Proposition \ref{theo-clustering performance}. The matrix $\mathbf{F}$ encodes sensitive information, as does $\mathbf{Z}$. Then, problem \eqref{eq-formulation-final} can be rephrased in term of $\mathbf{Y}$ as 
\begin{shrinkfix}
\begin{align}
       & \underset{\mathbf{X},\mathbf{L}, \bm{\upsilon},\mathbf{Y},\mathbf{R},\mathbf{Q}}{\mathrm{min}}\,\, \frac{1} {N}\lVert 
    {\mathbf{\Upsilon}} (\mathbf{X}_o- \mathbf{X})\rVert_{\mathrm{F}}^2 + \frac{\xi}{N}\mathrm{Tr}(\mathbf{X}^{\top}\mathbf{L}\mathbf{X}) + Reg(\mathbf{L}) \notag\\
       &\;\;\;\;\;\;\;\;\;\;\;\;\;\;\;\;+ \sum_{i=1}^D\frac{1}{\bm{\upsilon}_{[i]}}+ \mu \mathrm{Tr}(\mathbf{Y}^{\top}\mathbf{Z}^{\top}\mathbf{L}\mathbf{Z}\mathbf{Y})+ \gamma \lVert \mathbf{Q} - \mathbf{Z}\mathbf{Y}\mathbf{R}\rVert_{\mathrm{F}}^2\notag\\
    & \mathrm{s.t.}\; \mathbf{L}\in \mathcal{L},\bm{\upsilon}>0, \mathbf{Y}^{\top}\mathbf{Y} = \mathbf{I}, \mathbf{R}^{\top}\mathbf{R} = \mathbf{I}, \mathbf{Q} \in \mathcal{I}.
    \label{eq-formulation-final-Y}
\end{align}
\end{shrinkfix}
In \eqref{eq-formulation-final-Y}, the fairness constraint $ \mathbf{F}^{\top}\mathbf{{U}} = \mathbf{0}$ is removed. We conduct spectral embedding on the fair graph $\mathbf{Z}^{\top}\mathbf{L}\mathbf{Z}$, which encodes graph topology and sensitive information simultaneously, instead of $\mathbf{L}$ to ensure fair clustering. The impact of fairness constraints is discussed in the next subsection.

The basic formulation \eqref{eq-formulation-final} is flexible and has many possible extensions. Here are some examples. (\romannumeral1) We can replace spectral rotation with improved spectral rotation \cite{zhong2023self} to further improve discretization performance. (\romannumeral2)  We can introduce self-weighted features into \eqref{eq-formulation-final} to determine the importance of different features in assigning
cluster labels \cite{nie2019semi}. (\romannumeral3) We can extend \eqref{eq-formulation-final} from unnormalized SC \eqref{eq-prelim-6} to normalized SC \cite{shi2000normalized}. (\romannumeral4) We can also incorporate individual fairness \cite{gupta2021protecting} into our unified model. We place the details of these extensions in the supplementary material for completeness. 

\begin{remark}
The above extensions may improve fair clustering performance. However, we focus on the basic formulation \eqref{eq-formulation-final} since our primary goal is to demonstrate the advantages of the proposed graph construction method and the unified framework rather than to propose a complex FSC model.
    \label{remark-1}
\end{remark}

\subsection{Connections to Existing Works}

\textit{1) Connections to community-based GL models: }If we only focus on GL and fair spectral embedding, Eq.\eqref{eq-formulation-final-Y} becomes
\begin{shrinkfix}
\begin{align}
       & \underset{\mathbf{L}, \mathbf{Y}}{\min}\; \frac{\xi}{N}\mathrm{Tr}(\mathbf{X}^{\top}\mathbf{L}\mathbf{X}) + Reg(\mathbf{L}) +   \mu \mathrm{Tr}(\mathbf{Y}^{\top}\mathbf{Z}^{\top}\mathbf{L}\mathbf{Z}\mathbf{Y})\notag\\
    & \mathrm{s.t.}\; \mathbf{L}\in \mathcal{L}, \mathbf{Y}^{\top}\mathbf{Y} = \mathbf{I},
 \label{eq-formulation-4}
\end{align}
\end{shrinkfix}
where $\mathbf{X}$ is regarded as the ``noiseless" data here. According to Ky Fan’s theorem \cite{fan1949theorem}, we have $
\underset{\mathbf{Y}^{\top}\mathbf{Y} = \mathbf{I}}{\min}\; \mathrm{Tr}(\mathbf{Y}^{\top}\mathbf{Z}^{\top}\mathbf{L}\mathbf{Z}\mathbf{Y})=\sum_{k=1}^K \widetilde{\lambda}_k$, where $\widetilde{\lambda}_k$ is the $k$ smallest eigenvalue of $\mathbf{Z}^{\top}\mathbf{L}\mathbf{Z}$. Thus, the problem \eqref{eq-formulation-4} can be rephrased as 
\begin{shrinkfix}
\begin{align}
       & \underset{\mathbf{L}\in \mathcal{L}}{\min}\; \frac{\xi}{N}\mathrm{Tr}(\mathbf{X}^{\top}\mathbf{L}\mathbf{X}) + Reg(\mathbf{L}) +   \mu \sum_{k=1}^K\widetilde{\lambda}_k.
 \label{eq-formulation-4-1}
\end{align}
\end{shrinkfix}
Note that $\mathbf{Z}^{\top}\mathbf{L}\mathbf{Z}$ is a semi-positive definite matrix, i.e.,  $\widetilde{\lambda}_k \geq 0$. Minimizing \eqref{eq-formulation-4-1} is equivalent to forcing $\sum_{k=1}^K\widetilde{\lambda}_k \to 0$ if $\mu$ is large enough. That is, \eqref{eq-formulation-4-1} encourages the fair graph to have $K$ connected components. Therefore, \eqref{eq-formulation-4-1} can be viewed as the community-based GL model, which has been widely studied. For example, \cite{kumar2020unified} lets the first $K$ smallest eigenvalues of Laplacian matrices be zero to obtain the community structures, which can be relaxed to the last term in \eqref{eq-formulation-4-1}. The works \cite{wu2022effective, nie2014clustering} force the rank of Laplacian matrices to $D-K$, which can also be interpreted as minimizing the sum of the first $K$ eigenvalues. Furthermore, \cite{pircalabelu2020community} adds a term $\mathrm{Tr}(\mathbf{\Xi}^{\top}\mathbf{L}\mathbf{\Xi})$ to impose community constraints, where $\mathbf{\Xi}\mathbf{\Xi}^{\top}$  contains the value 1 for within-community edges only and 0 everywhere else. Although closely related, \eqref{eq-formulation-4} differs from existing community-based GL models in two key aspects.  First, the basic GL models are different.  Our model is based on the smoothness-based GL, while \cite{kumar2020unified, pircalabelu2020community} are based on statistical GL models like Graphical Lasso. Besides, \cite{wu2022effective, nie2014clustering} are based on the ANGL method. Second, our model imposes the community constraint on the fair graph $\mathbf{Z}^{\top}\mathbf{L}\mathbf{Z}$ instead of on $\mathbf{L}$ like existing works.  Thus, the fairness constraint may affect the topology of the learned graph to obtain fair clustering. We will test the impact of the fairness constraint in the experimental section.

\textit{2) Connections to unified SC models: }
If we remove the denoising module and fairness constraint, our model becomes
\begin{shrinkfix}
\begin{align}
       & \underset{\mathbf{L},\mathbf{U},\mathbf{R},\mathbf{Q}}{\min}\; \frac{\xi}{N}\mathrm{Tr}(\mathbf{X}^{\top}\mathbf{L}\mathbf{X}) + Reg(\mathbf{L}) + \mu \mathrm{Tr}(\mathbf{U}^{\top}\mathbf{L}\mathbf{U})\notag\\
       &\;\;\;\;\;\;\;\;\;\;\;+ \gamma \lVert \mathbf{Q} - \mathbf{U}\mathbf{R}\rVert_{\mathrm{F}}^2\notag\\
    & \mathrm{s.t.}\; \mathbf{L}\in \mathcal{L}, \mathbf{U}^{\top}\mathbf{U} = \mathbf{I}, \mathbf{R}^{\top}\mathbf{R} = \mathbf{I}, \mathbf{Q} \in \mathcal{I}.
    \label{eq-ours-nofair}
\end{align}
\end{shrinkfix}
Again, $\mathbf{X}$ is treated as the observed data. The model \eqref{eq-ours-nofair} is an end-to-end SC model. Here, we discuss the connections between our model and those unified SC models integrating graph construction, spectral embedding, and discretization. As stated in Remark \ref{remark-1}, we focus on basic formulations without additional extensions. The first model we compare is \cite{kang2018unified}
\begin{shrinkfix}
\begin{align}
&\underset{\mathbf{W},\mathbf{U},\mathbf{Q}, \mathbf{R}} {\min} \lVert \mathbf{X} - \mathbf{W}^{\top} \mathbf{X} \rVert_{\mathrm{F}}^2 + \alpha_{U}  \lVert \mathbf{W} \rVert_{1,1}  + {\mu_U} \mathrm{Tr}(\mathbf{U}^{\top}\mathbf{L}\mathbf{U}) \notag\\ 
&\;\;\;\;\;\;\;\;\;\;\;\;\;+ \gamma_U \lVert \mathbf{Q} -  \mathbf{U} \mathbf{R} \rVert_{\mathrm{F}}^2\notag\\
&\mathrm{s.t.}\; \mathbf{W} \in \mathcal{W},  \mathbf{U}^{\top}\mathbf{U} = \mathbf{I}, \mathbf{R}^{\top}\mathbf{R} = \mathbf{I}, \mathbf{Q} \in \mathcal{I},
\label{eq-USPC-1}
\end{align}
\end{shrinkfix}
where $\alpha_{U}, \mu_{U}$, and $\gamma_{U}$ are constant parameters. Moreover, $ \mathcal{W} = \left\{ \mathbf{W} : \mathbf{W}\in \mathbb{S}^{D\times D}, \mathbf{W}\geq 0,  \mathrm{diag}(\mathbf{W}) = \mathbf{0}\right\}$ is the set containing all adjacency matrices. This is a unified SC model that leverages the sparse representation method \cite{elhamifar2013sparse} to construct graphs, which is different from our GL method.

Another unified SC model \cite{peng2023jgsed, han2018discrete} is concluded as  
\begin{shrinkfix}
\begin{align}
&\underset{\mathbf{W},\mathbf{U},\mathbf{Q}, \mathbf{R}}{\min}\sum_{i,j = 1}^D \lVert \mathbf{X}_{[i,:]} -\mathbf{X}_{[j,:]}\rVert_{2}^2 \mathbf{W}_{[i,j]} + \beta_{J} \mathbf{W}_{[i,j]}^2 \notag\\ 
&\;\;\;\;\;\;\;\;\;\;\;\;+ {\mu_J} \mathrm{Tr}(\mathbf{U}^{\top}\mathbf{L}\mathbf{U}) + \gamma_J \lVert \mathbf{Q} -  \mathbf{U} \mathbf{R} \rVert_{\mathrm{F}}^2\notag\\
&\mathrm{s.t.}\; \mathbf{W}\mathbf{1} = \mathbf{1}, \mathbf{W} \geq 0, \mathbf{L} = \mathbf{D} - \mathbf{W}, \mathbf{U}^{\top}\mathbf{U} = \mathbf{I}, \notag\\
&\;\;\;\;\; \;\;\mathbf{R}^{\top}\mathbf{R} = \mathbf{I}, \mathbf{Q} \in \mathcal{I},
\label{eq-J-1}
\end{align}
\end{shrinkfix}
where $\beta_{J}, \mu_{J}$, and $\gamma_{J}$ are constant parameters. The graph construction method in \eqref{eq-J-1} is the ANGL method \cite{nie2014clustering}. We have discussed the difference between our graph construction method and the ANGL in the previous subsection.

In summary, our model \eqref{eq-ours-nofair} differs from the existing unified SC models \eqref{eq-USPC-1}-\eqref{eq-J-1} mainly in the graph construction method. As Proposition \ref{theo-clustering performance} states, an accurate GL method can boost fair clustering performance. In the experimental section, we develop fair versions of \eqref{eq-USPC-1}-\eqref{eq-J-1} and compare them with our model \eqref{eq-formulation-final} to illustrate the superiority of our model.

\section{Model Optimization}
\label{sec:algorithm}
In this section, we first propose an algorithm for solving \eqref{eq-formulation-final}, followed by convergence and complexity analyses. 
\subsection{Optimization Algorithm}
Our algorithm alternately updates $\mathbf{L}, \mathbf{U}, \mathbf{R}$, $\mathbf{Q}$,  $\mathbf{X}$, and $\bm{\upsilon}$ in \eqref{eq-formulation-final}, i.e., updating one with the others fixed. For clarity, we omit the iteration index here. The following derivations are the updates in one iteration.

\textit{1) Update $\mathbf{L}$: } The sub-problem of updating  $\mathbf{L}$ is 
\begin{shrinkfix}
\begin{align}
       & \underset{\mathbf{L}\in \mathcal{L}}{\min}\; \frac{\xi}{N}\mathrm{Tr}(\mathbf{X}^{\top}\mathbf{L}\mathbf{X}) + Reg(\mathbf{L}) + \mu \mathrm{Tr}(\mathbf{U}^{\top}\mathbf{L}\mathbf{U}).
    \label{eq-opt-1}
\end{align}
\end{shrinkfix}
The problem can be rewritten in terms of $\mathbf{W}$ 
\begin{shrinkfix}
\begin{align}
&\underset{\mathbf{W}\in \mathcal{W}}{\min}\; \frac{1}{2}\lVert \mathbf{W}\circ\mathbf{P}\rVert_{1,1} + Reg_W(\mathbf{W}),
    \label{eq-opt-2}
\end{align}
\end{shrinkfix}
where 
\begin{shrinkfix}
\begin{align}
\mathbf{P}_{[ij]} = \frac{\xi}{N}\lVert \mathbf{X}_{[i,:]} - \mathbf{X}_{[j,:]}\rVert_2^2 + {\mu}\lVert \mathbf{U}_{[i,:]} - \mathbf{U}_{[j,:]}\rVert_2^2,
    \label{eq-opt-2-0}
\end{align}
\end{shrinkfix}
and $Reg_W(\mathbf{W}) = - \mathbf{1}^{\top}\mathrm{log}(\mathbf{W}\mathbf{1}) + {\beta}\lVert \mathbf{W}\rVert_{\mathrm{F}}^2$. By the definition of $\mathcal{W}$, the free variables of $\mathbf{W}$ are the upper triangle elements. Thus, we define a vector $\mathbf{w}\in\mathbb{R}^{P}, P:=\frac{ D(D-1)}{2}$,  satisfying that  $\mathbf{w} = \mathrm{Triu}(\mathbf{W})$, where $\mathrm{Triu}(\cdot): \mathbb{R}^{D\times D}\to\mathbb{R}^P$ is a function that converts the upper triangular elements of a matrix into a vector. Then, the problem \eqref{eq-opt-2} is equivalent to 
\begin{shrinkfix}
\begin{align}
&\underset{\mathbf{w}\geq 0 }{\min}\; \mathbf{p}^{\top}\mathbf{w} -  \mathbf{1}^{\top}\log(\mathbf{S}\mathbf{w}) + 2\beta \lVert \mathbf{w}\rVert_2^2,
    \label{eq-opt-2-1}
\end{align}
\end{shrinkfix}
where $\mathbf{p} = \mathrm{Triu}(\mathbf{P}$),  $\mathbf{S}\in \mathbb{R}^{D\times P}$ is a linear operator satisfying $\mathbf{S}\mathbf{w} = \mathbf{W}\mathbf{1}$ \cite{kalofolias2016learn}. The problem \eqref{eq-opt-2-1} is convex, and we employ the algorithm in \cite{saboksayr2021accelerated} to solve the problem. The complete algorithm flow is presented in the supplementary materials. After obtaining the estimated ${\mathbf{w}}$, we let ${\mathbf{W}} = \mathrm{iTriu}({\mathbf{w}})$, where $\mathrm{iTriu}(\cdot):\mathbb{R}^P\to \mathbb{R}^{D\times D}$ is the inverse $\mathrm{Triu}$ operation. The operation $\mathrm{iTriu}({\mathbf{w}})$ converts ${\mathbf{w}}$ into an adjacency matrix, where ${\mathbf{w}}$ corresponds to the upper triangle elements of ${\mathbf{W}}$. Finally, we calculate the Laplacian matrix from ${\mathbf{W}}$ and feed it into subsequent updates of other variables.

\textit{2) Update $\mathbf{U}$: } The sub-problem of updating $\mathbf{U}$ is 
\begin{shrinkfix}
\begin{align}
&\underset{\mathbf{U}}{\min}\;\mu \mathrm{Tr}\left(\mathbf{U}^{\top}\mathbf{L}\mathbf{U}\right) + \gamma \lVert \mathbf{Q} - \mathbf{U}\mathbf{R}\rVert_{\mathrm{F}}^2 \notag\\
    & \mathrm{s.t.}\; \mathbf{U}^{\top}\mathbf{U} = \mathbf{I}, \mathbf{F}^{\top}\mathbf{U}= \mathbf{0}. 
    \label{eq-opt-3}
\end{align}
\end{shrinkfix}
Like \eqref{eq-formulation-final-Y}, \eqref{eq-opt-3} can be cast into a problem of variable $\mathbf{Y}$
\begin{shrinkfix}
\begin{align}
&
\underset{\mathbf{Y}^{\top}\mathbf{Y} = \mathbf{I}}{\min}\;\mu\mathrm{Tr}\left(\mathbf{Y}^{\top}\mathbf{Z}^{\top}\mathbf{L}\mathbf{Z}\mathbf{Y}\right) + \gamma \lVert \mathbf{Q} - \mathbf{Z}\mathbf{Y}\mathbf{R}\rVert_{\mathrm{F}}^2 \notag \\
\Leftrightarrow & \underset{\mathbf{Y}^{\top}\mathbf{Y} = \mathbf{I}}{\min}\;\mu\mathrm{Tr}\left(\mathbf{Y}^{\top}\mathbf{Z}^{\top}\mathbf{L}\mathbf{Z}\mathbf{Y}\right) - 2\gamma\mathrm{Tr}(\mathbf{R}\mathbf{Q}^{\top}\mathbf{Z}\mathbf{Y}).
    \label{eq-opt-4}
\end{align}
\end{shrinkfix}
This is a typical quadratic optimization problem with orthogonal constraints. Let $\phi(\mathbf{Y})$ be the objective function of \eqref{eq-opt-4}. We have that $\phi(\mathbf{Y})$ is differential, and $ 
\nabla_{\mathbf{Y}}\,\phi(\mathbf{Y}) = 2\mu \mathbf{Z}^{\top}\mathbf{L}\mathbf{Z}\mathbf{Y} - 2\gamma \mathbf{Z}^{\top}\mathbf{Q}\mathbf{R}^{\top}$. Thus, the problem can be efficiently solved via the algorithm in \cite{wen2013feasible}. After obtaining $\mathbf{Y}$, we let $\textbf{U} = \textbf{Z}\textbf{Y}$.

\textit{3) Update $\mathbf{R}$: } The sub-problem of updating $\mathbf{R}$ is 
\begin{shrinkfix}
\begin{align}
&\underset{\mathbf{R}^{\top}\mathbf{R} = \mathbf{I}}{\min}\; \gamma \lVert \mathbf{Q} - \mathbf{U}\mathbf{R}\rVert_{\mathrm{F}}^2 \notag\\
\Leftrightarrow&\underset{\mathbf{R}^{\top}\mathbf{R} = \mathbf{I}}{\max}\; \mathrm{Tr}(\mathbf{Q}^{\top}\mathbf{U}\mathbf{R}).
    \label{eq-opt-5}
\end{align}
\end{shrinkfix}
It is the orthogonal Procrustes problem with a closed-form solution \cite{schonemann1966generalized}. Assuming that $\mathbf{\Theta}_{L}$ and $\mathbf{\Theta}_{R}$ are the left and right matrices of SVD of $\mathbf{Q}^{\top}\mathbf{U}$, the solution to \eqref{eq-opt-5} is \cite{schonemann1966generalized} 
\begin{shrinkfix}
\begin{align}
\mathbf{R} = \mathbf{\Theta}_{R}\mathbf{\Theta}_{L}^{\top}.
    \label{eq-opt-5-1}
\end{align}
\end{shrinkfix}

\textit{4) Update $\mathbf{Q}$: } The sub-problem of updating $\mathbf{Q}$ is 
\begin{shrinkfix}
\begin{align}
&\underset{\mathbf{Q}\in \mathcal{I}}{\min}\; \gamma \lVert \mathbf{Q} - \mathbf{U}\mathbf{R}\rVert_{\mathrm{F}}^2\notag\\
\Leftrightarrow&\underset{\mathbf{Q}\in \mathcal{I}}{\min}\; \gamma\mathrm{Tr}(\mathbf{Q}^{\top}\mathbf{Q}) -2\gamma \mathrm{Tr}(\mathbf{Q}^{\top}\mathbf{U}\mathbf{R}) \notag\\
\Leftrightarrow&\underset{\mathbf{Q}\in \mathcal{I}}{\max}\;  \mathrm{Tr}(\mathbf{Q}^{\top}\mathbf{U}\mathbf{R}).
    \label{eq-opt-6}
\end{align}
\end{shrinkfix}
The optimal solution to \eqref{eq-opt-6} is as follows:
\begin{shrinkfix}
\begin{align}
\mathbf{Q}_{[ik]} = 
\begin{cases}
    1 & k = \mathrm{argmax}_{j\in[K]} \;(\mathbf{U}\mathbf{R})_{[ij]},\\
    0 & \mathrm{others}.
\end{cases}
    \label{eq-opt-8}
\end{align}
\end{shrinkfix}

\textit{5) Update $\mathbf{X}$: } The sub-problem of updating $\mathbf{X}$ is \eqref{eq-prelim-1-1}, which has a closed solution \eqref{eq-prelim-1-3}. However, matrix inversion is computationally expensive with complexity $\mathcal{O}(D^3)$. Fortunately, $\left(\mathbf{\Upsilon}^{\top}\mathbf{\Upsilon} + {\xi}\mathbf{L}\right)^{-1}$ is symmetric, sparse, and positive definite. We can hence solve \eqref{eq-prelim-1-1} efficiently using conjugate gradient (CG) algorithm without matrix inverse \cite{axelsson1986rate}.

\textit{6) Update $\bm{\upsilon}$: } The sub-problem of updating  $\bm{\upsilon}$ is \eqref{eq-prelim-1-4}. Taking the derivative of \eqref{eq-prelim-1-4} and setting it to zero, we have
\begin{shrinkfix}
\begin{align}
\bm{\upsilon}_{[i]} = \frac{\sqrt{N}}{\lVert (\mathbf{X}_o)_{[i,:]} - \mathbf{X}_{[i,:]}\rVert_2}, \;\;i=1,...,D.
    \label{eq-opt-9}
\end{align}
\end{shrinkfix}

It is observed that the updates of $\mathbf{L}$, $\mathbf{U}$, $\mathbf{R}$, $\mathbf{Q}$, $\mathbf{X}$, and $\bm{\upsilon}$ are coupled with each other. Updating one variable depends on the other variables, leading to an overall optimal solution. The complete procedure is presented in Algorithm \ref{alg:1}.

\begin{algorithm}[t] 
\caption{The algorithm for problem \eqref{eq-formulation-final}} 
\begin{algorithmic}[1] 
\REQUIRE  Data matrix $\mathbf{X}_o \in \mathbb{R}^{D\times N}$, sensitive attributes related matrix $\mathbf{F}$ or $\mathbf{Z}$, the number of clusters $K$, model parameters $\xi, \beta, \mu$, and $\gamma$\\
\ENSURE  The learned $\mathbf{L}$ and discrete cluster labels  $\mathbf{Q}$
\STATE Initialize $\mathbf{L}$, $\mathbf{U}$, $\mathbf{Q} $, and $\mathbf{R}$ randomly, $\mathbf{X} = \mathbf{X}_o$, and $\bm{\upsilon} = \mathbf{1}$\\

\WHILE{not converged}
\STATE Calculate $\mathbf{P}$ by \eqref{eq-opt-2-0} and let $\mathbf{p} = \mathrm{Triu(\mathbf{P})}$

\STATE Update $\mathbf{w}$ by solving \eqref{eq-opt-2-1} 

\STATE Convert $\mathbf{W} = \mathrm{iTriu}(\mathbf{w})$ and calculate $\mathbf{L} = \mathbf{D} -\mathbf{W}$

\STATE Update $\mathbf{Y}$ by solving problem
\eqref{eq-opt-4} using the algorithm in \cite{wen2013feasible}, and let $\mathbf{U} = \mathbf{Z}\mathbf{Y}$

\STATE Update $\mathbf{R} = \mathbf{\Theta}_{R}\mathbf{\Theta}_{L}^{\top}$, where  $ \mathbf{\Theta}_{L}$ and  $\mathbf{\Theta}_{R}$ are the left and right matrices of SVD of $\mathbf{Q}^{\top}\mathbf{U}$

\STATE Update $\mathbf{Q}$ via \eqref{eq-opt-8} 

\STATE Update $\mathbf{X}$ by solving \eqref{eq-prelim-1-1}

\STATE Update $\bm{\upsilon}$ using \eqref{eq-opt-9}

\ENDWHILE
\end{algorithmic}
\label{alg:1} 
\end{algorithm}

\subsection{Convergence and Complexity Analysis}
\textit{1) Convergence analysis: } It is challenging to obtain a globally optimal solution to \eqref{eq-formulation-final} since it is not jointly convex for all variables. However, our algorithm for solving each sub-problem can reach its optimal solution. Specifically, when we update $\mathbf{L}$, the problem \eqref{eq-opt-2-0} is convex, and the corresponding algorithm is guaranteed to converge to the global optimum \cite{saboksayr2021accelerated}.  When updating $\mathbf{U}$, we use the algorithm in \cite{wen2013feasible} to solve the problem \eqref{eq-opt-4}, which can converge to the global optimum  \cite{wen2013feasible}. The updates of $\mathbf{Q}$, $\mathbf{R}$ and $\bm{\upsilon}$ have closed-form solutions. Despite updating $\mathbf{X}$ via \eqref{eq-prelim-1-2} has a closed solution, we update $\mathbf{X}$ using CG, which is guaranteed to converge \cite{axelsson1986rate}. In summary, the update of each variable converges in our algorithm. In reality, the whole algorithm converges well, which is verified experimentally in Section \ref{sec:Experiments}.

\textit{2) Complexity analysis: } In one iteration, our algorithm consists of six parts, which we analyze one by one below. As stated in \cite{saboksayr2021accelerated}, the update of $\mathbf{L}$ requires $\mathcal{O}(T_1D^2)$ costs, where $T_1$ is the average number of iterations of updating $\mathbf{w}$. The computational cost can be further reduced if the average number of neighbors per node is fixed; see \cite{kalofolias2018large} and analysis therein. The computational complexity of our algorithm for updating $\mathbf{U}$ is $\mathcal{O}(T_2(DK^2 + K^3))$  according to \cite{wen2013feasible}, where $T_2$ is the average number of iterations of the algorithm in \cite{wen2013feasible}. When updating $\mathbf{R}$, we perform SVD on $\mathbf{Q}^{\top}\mathbf{U}\in \mathbb{R}^{K\times K}$, which costs $\mathcal{O}(K^3)$.  {\color{black} The  updates of $\mathbf{Q}$ and $\bm{\upsilon}$ require $\mathcal{O}(DK^2)$ and $\mathcal{O}(DN)$, respectively.} Finally, the complexity of using CG to update $\mathbf{X}$ is $\mathcal{O}(T_3DN)$, where $T_3$ is the average number of iterations of the CG algorithm.



\section{Experiments}
\label{sec:Experiments}
In this section, we test our proposed model using synthetic, benchmark, and real-world data. First, some experimental setups are introduced.

\subsection{Experimental Setups}

\textit{1) Graph generation: } For synthetic data, we leverage the vSBM method to generate random graphs with sensitive attributes. Specifically, we let $ \zeta_s = \frac{1}{S}, a = 0.8, b = 0.2, c = 0.15$, and $d = 0.05$. After obtaining connections among nodes, we assign each edge a random weight between $[0.1,2]$. Finally, we normalize the edge weights to satisfy $\mathrm{Tr}(\mathbf{L}^*) = D$.

\textit{2) Signal generation: }We generate $N$ observed signals from the following Gaussian distribution  \cite{dong2016learning}
\begin{shrinkfix}
    \begin{align}
        (\mathbf{X}_o)_{[:,n]} \sim \mathcal{N}\left(\mathbf{0}, (\mathbf{L}^*)^{\dag} + \mathbf{\Sigma}_e\right),\;\;n =1,...,N,
    \label{signal-genegration}
    \end{align}
\end{shrinkfix}
where $\mathbf{\Sigma}_e = \mathrm{diag}(\sigma_1^2,...,\sigma_D^2)$ and $\sigma_i$ is the noise scale of the $i$-th node. As stated in \cite{dong2016learning}, signals generated in this way are smooth over the corresponding graph.

\textit{3) Evaluation metrics: } In topology inference, determining whether two vertices are connected can be regarded as a binary classification problem. Thus, we employ the F1-score ($\mathrm{FS}$)  to evaluate classification results
\begin{shrinkfix}
    \begin{align}
   \mathrm{FS} = \frac{2\mathrm{TP}}{2\mathrm{TP} + \mathrm{FN} + \mathrm{FP}},
    \label{FS}
    \end{align}
\end{shrinkfix}
where $\mathrm{TP}$ is true positive rate, $\mathrm{TN}$ is true negative rate, $\mathrm{FP}$ is false positive rate, and $\mathrm{FN}$ is false negative rate. We also use the estimation error (${\mathrm{EE}}$) to evaluate the learned graph 
\begin{shrinkfix}
    \begin{align}
   \mathrm{EE} = \lVert\mathbf{Z}^{\top}\widehat{\mathbf{L}}\mathbf{Z} - \mathbf{Z}^{\top}\mathbf{L}^*\mathbf{Z}\lVert_{\mathrm{F}}.
    \label{RE}
    \end{align}
\end{shrinkfix}
For a fair comparison of $\mathrm{EE}$, we normalize the learned graphs to $\mathrm{Tr}(\widehat{\mathbf{L}}) = D$. For fair clustering, we use the same two metrics as in \cite{kleindessner2019guarantees}: clustering error ($\mathrm{CE}$) and Balance ($\mathrm{Bal}$)
\begin{shrinkfix}
    \begin{align}
      &\mathrm{CE} = \frac{1}{D} \left|\{i: \widehat{\pi}_C(i) \neq {\pi}_C(i), i=1....,D\}\right|,\notag\\
      &\mathrm{Balance}\;(\mathrm{Bal}) = \frac{1}{K} \sum_{k=1}^K \mathrm{Balance}(C_k),
    \label{CE-Bal}
    \end{align}
\end{shrinkfix}
where $ \widehat{\pi}_C(i)$ is the estimated cluster label of node $i$ (after proper permutation), and ${\pi}_C(i)$ is the ground-truth. The metric $\mathrm{Balance}$ measures the average balance of all clusters.

\textit{4) Baselines: }The comparison baselines are list in Table.\uppercase\expandafter{\romannumeral1}. The model Fairlets is the fair version of $k$median \cite{chierichetti2017fair}. Models 3-5 are the implementations of \cite{kleindessner2019guarantees} using different graph construction methods. FGLASSO \cite{tarzanagh2021fair} is the only model that jointly performs graph construction and fair spectral embedding. FSRSC and  FJGSED are the fair versions of unified SC models \eqref{eq-USPC-1} and \eqref{eq-J-1}. The formulations and algorithms for the two models are placed in the supplementary material.

\begin{table}[t]
\renewcommand{\arraystretch}{1}
	\centering
  
	\begin{threeparttable}
        \tabcolsep = 0.6em
	\caption{Comparison baselines}
	{\footnotesize 
	\begin{tabular}{c|c|c|c|c|c}
	\Xhline{1.05pt}
	
 Index  &  Models & Graph-based & Fair  & End-to-End  & GL method  \\
     
    \hline

    1 &$k$means
     &\XSolidBrush  &\XSolidBrush  &\textemdash &\textemdash\\  	
   2 &Fairlets
     &\XSolidBrush  &\Checkmark &\XSolidBrush &\textemdash\\  	
  3   &CorrFSC
     &\Checkmark &\Checkmark &\XSolidBrush &PC\\  	
  4  & KNNFSC
     &\Checkmark &\Checkmark &\XSolidBrush &$k$-NN\\  		
  5   &$\varepsilon$NNFSC
     &\Checkmark &\Checkmark &\XSolidBrush &$\varepsilon$-NN\\  
  6  & FGLASSO
      &\Checkmark &\Checkmark &\XSolidBrush &GLASSO\\ 
 7   & FJGSED
    &\Checkmark &\Checkmark &\Checkmark &ANGL\\ 
 8  &  FSRSC
    &\Checkmark &\Checkmark &\Checkmark &SR\\ 
	\Xhline{1.2pt}

	\end{tabular} 
	}

 \end{threeparttable}
 
 \label{table-baselines-models}
     	\vspace{-1em}
\end{table}

\textit{5) Determination of parameters: } For our model, we first grid-search $\xi$ and $\beta$ corresponding to the best $\mathrm{FS}$ in the range $[0.001,0.1]$ for the graph learning task. Then, parameters $\mu$ and $\gamma$ are selected as those achieving the best $\mathrm{CE}$ in the range $[0.001, 1]$. All parameters of baselines are also selected as those achieving the best $\mathrm{CE}$ values.

  


     



        
        
 

\begin{table*}[t]
\renewcommand{\arraystretch}{1.2}
	\centering
 
	\begin{threeparttable}
        \tabcolsep = 0.5em
	\caption{The results of our model and the compared baselines under different cases.}
	{\small 
	\begin{tabular}{c|c|c|c|c|c|c|c|c|c|c|c|c|c|c|c|c}
	\Xhline{1.05pt}
	

	 & \multicolumn{4}{c|}{$\sigma_i\sim \mathcal{U}(0,0.2), N = 1000$} & \multicolumn{4}{c|}{$\sigma_i\sim \mathcal{U}(0.4,0.6), N = 1000$}& \multicolumn{4}{c|}{$\sigma_i\sim \mathcal{U}(0,0.2), N = 5000$} & \multicolumn{4}{c}{$\sigma_i\sim \mathcal{U}(0.4,0.6), N = 5000$}\\

	\cline{2-17}
  
     & $\mathrm{FS} \uparrow$ & $\mathrm{EE} \downarrow$  & $\mathrm{CE}\downarrow$  & $\mathrm{Bal} \uparrow$ 
     & $\mathrm{FS} \uparrow$ & $\mathrm{EE} \downarrow$ & $\mathrm{CE}\downarrow$ & $\mathrm{Bal} \uparrow$ 
     & $\mathrm{FS} \uparrow$  & $\mathrm{EE} \downarrow$ & $\mathrm{CE}\downarrow$  & $\mathrm{Bal} \uparrow$ 
     & $\mathrm{FS} \uparrow$ & $\mathrm{EE}\downarrow $  & $\mathrm{CE}\downarrow$ & $\mathrm{Bal} \uparrow$  \\
     
    \hline

    $k$means 
    &\textemdash&\textemdash &0.671 &0.191 
    &\textemdash &\textemdash&0.687  &0.149
    &\textemdash&\textemdash &0.635 &0.161 
    &\textemdash&\textemdash &0.667  &0.145	
     \\

    Fairlets  
    &\textemdash&\textemdash &0.658 &0.485 
    &\textemdash&\textemdash &0.665  &0.457 
    &\textemdash&\textemdash &0.611 &0.355     
    &\textemdash&\textemdash &0.623  &0.348\\

     CorrFSC 
     &0.472 &2.858 &0.567 &0.482
     &0.441 &3.016&0.578  &0.705
     &0.630 &2.529 &0.104 &0.874 
     &0.596 & 2.511 &0.156 &0.859\\

     KNNFSC 
     &0.105&\textemdash &0.687 &0.829 
     &0.103&\textemdash &0.703  &0.626
     &0.113 &\textemdash&0.729 &0.628 
     &0.098&\textemdash&0.682 &0.731\\

     EpsNNFSC 
     &0.086&\textemdash &0.729 &0.380 
     &0.094 &\textemdash&0.739  &0.333
     &0.091&\textemdash &0.718 &0.652 
     &0.065&\textemdash &0.724 &0.369\\

    FGLASSO 
    &0.482&3.902 &0.411 &0.616 
    &0.450&3.724 &0.406  &0.646
    &0.587&3.971 &0.291 &0.657 
    &0.574&3.533 &0.271  &\textbf{0.908}\\ 

    FJGSED 
    &0.271&28.159 &0.724 &0.359
    &0.263&22.626 &0.734  &0.240
    &0.325&23.576 &0.604 &0.579 
    &0.293&31.552 &0.734  &0.247\\

    FSRSC 
    &0.374&5.222 &0.724 &0.619
    &0.355&9.671 &0.733  &0.607
    &0.345&5.049 &0.729 &0.766 
    &0.512&10.024 &0.739  &0.663\\

    Ours  
    &\textbf{0.501}&\textbf{2.375} &\textbf{0.286} &\textbf{0.845}
    
    &\textbf{0.474}&\textbf{2.414} &\textbf{0.390} &\textbf{0.801}
    
    &\textbf{0.715}&\textbf{1.691} &\textbf{0.052} &\textbf{0.960} 
    
    &\textbf{0.674}&\textbf{2.174} &\textbf{0.142}  &{0.870}\\

	\Xhline{1.2pt}

	\end{tabular} 
	}

		\begin{tablenotes}
			\footnotesize
			\item $\uparrow$ means that higher value is better and $\downarrow$  means that lower value is better.
           \item $\sigma_i\sim \mathcal{U}(a1,a2), i =1,...,D$, means that the noise scale of the $i$-th node is from the uniform distribution $\mathcal{U}(a1,a2).$
        \end{tablenotes}

 \end{threeparttable}
 
 \label{table-baselines}
     	\vspace{-1em}
\end{table*}

\begin{figure*}[t] 
    \centering
	  \subfloat[Ground-truth]{
       \includegraphics[width=0.2\linewidth]{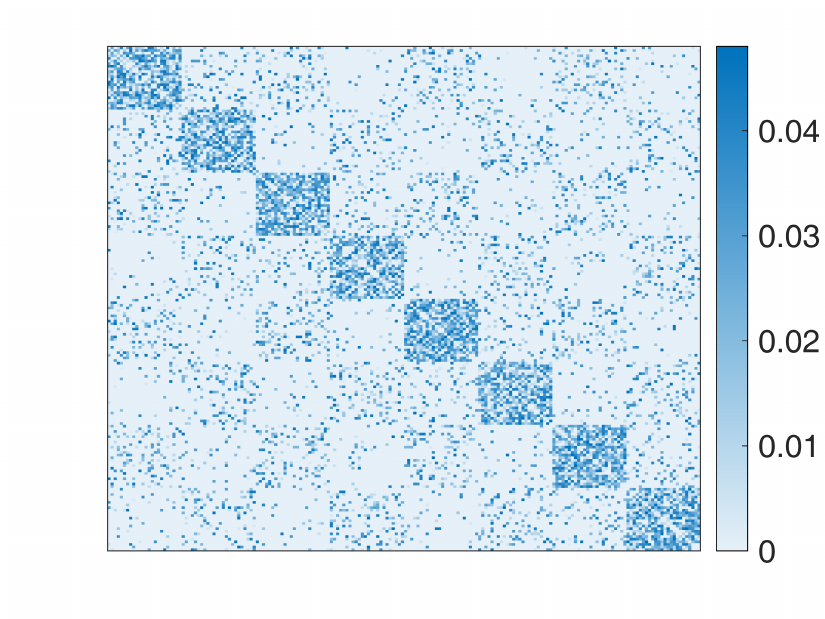}}
       \hspace{2em}
       \subfloat[CorrFSC]{  \includegraphics[width=0.2\linewidth]{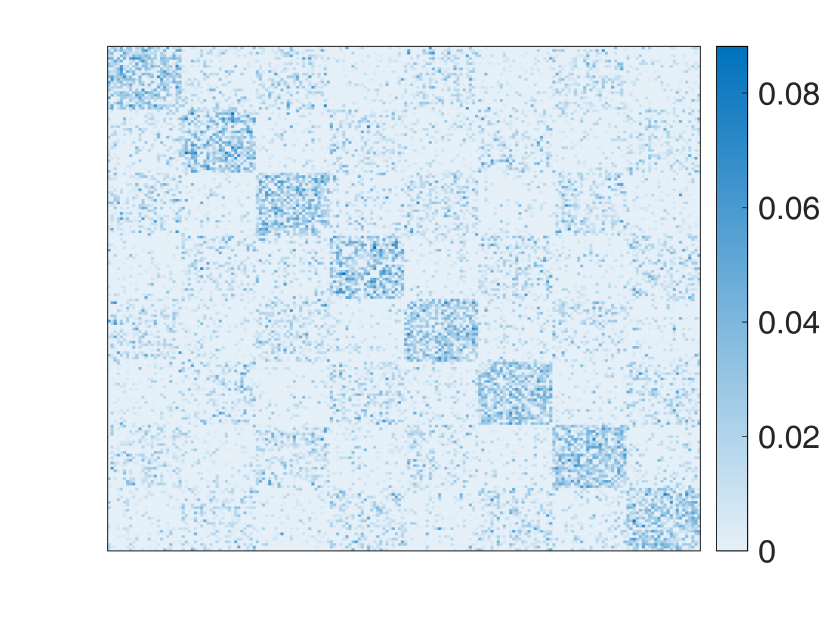}}
        \hspace{2em}
      \subfloat[KNNFSC]{\includegraphics[width=0.2\linewidth]{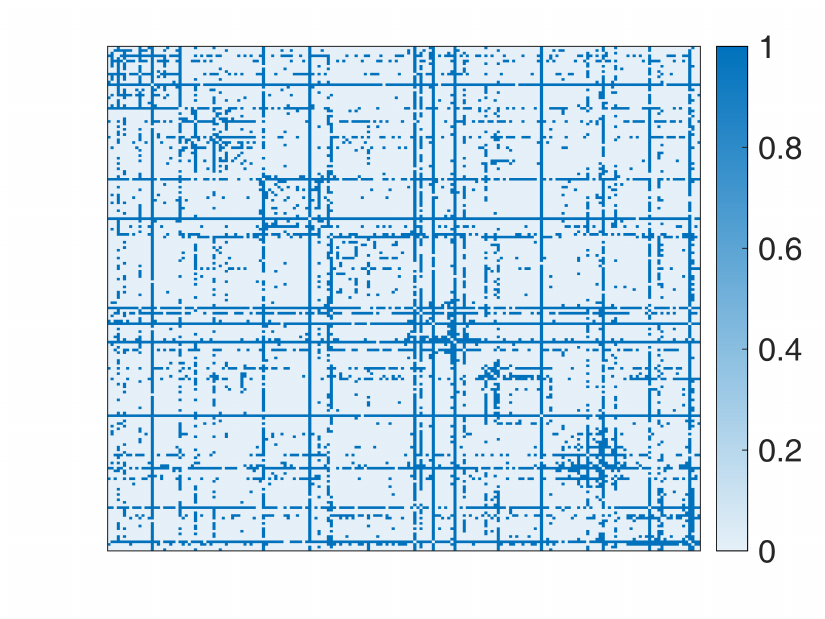}}
       \hspace{2em}
      \subfloat[EpsNNFSC]{\includegraphics[width=0.2\linewidth]{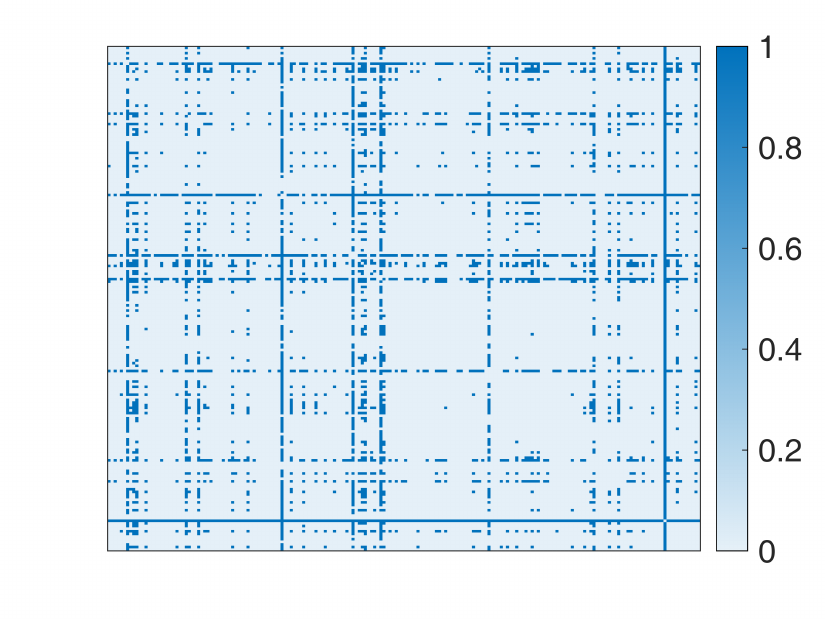}}

    	\vspace{-0.5em}
     \subfloat[FGLASSO]{\includegraphics[width=0.2\linewidth]{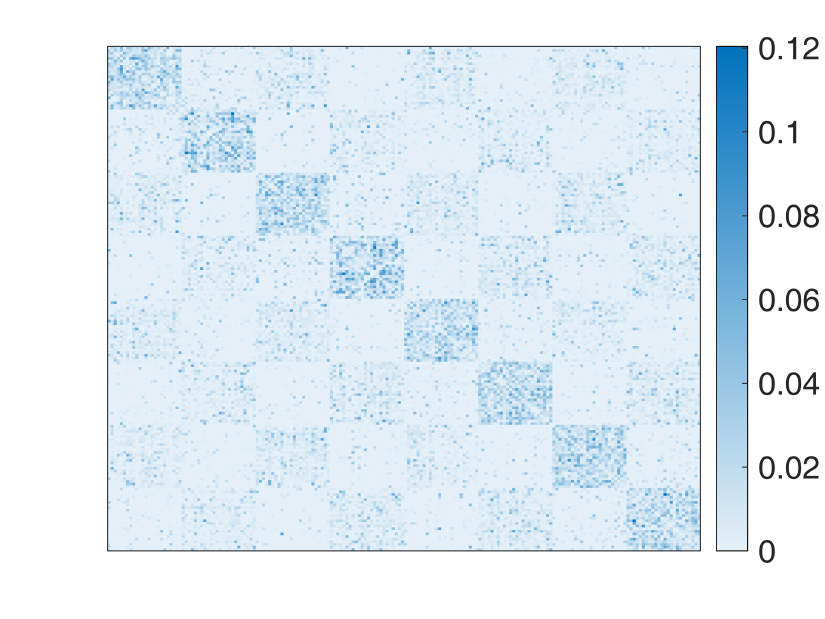}}
     \hspace{2em}
    \subfloat[FJGSED]{\includegraphics[width=0.2\linewidth]{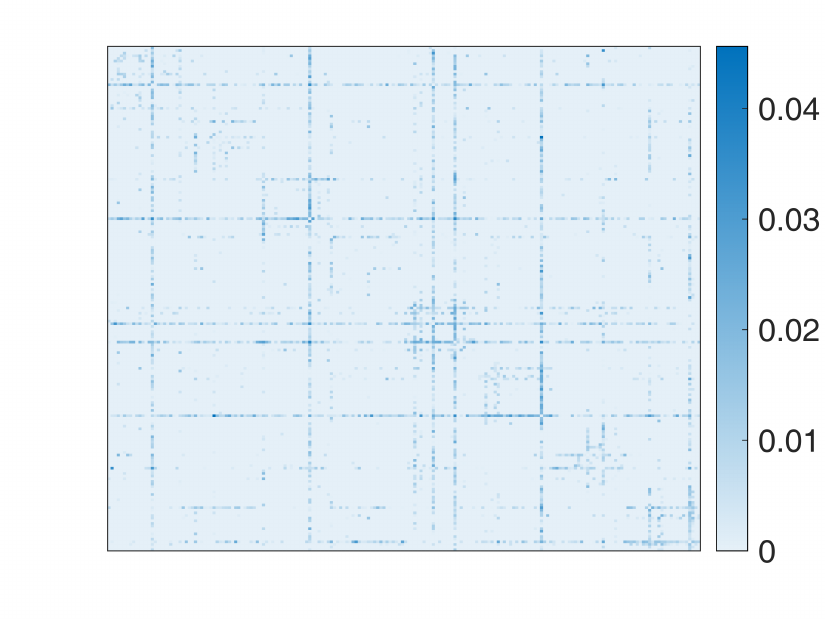}}
    \hspace{2em}
    \subfloat[FSRSC]{\includegraphics[width=0.2\linewidth]{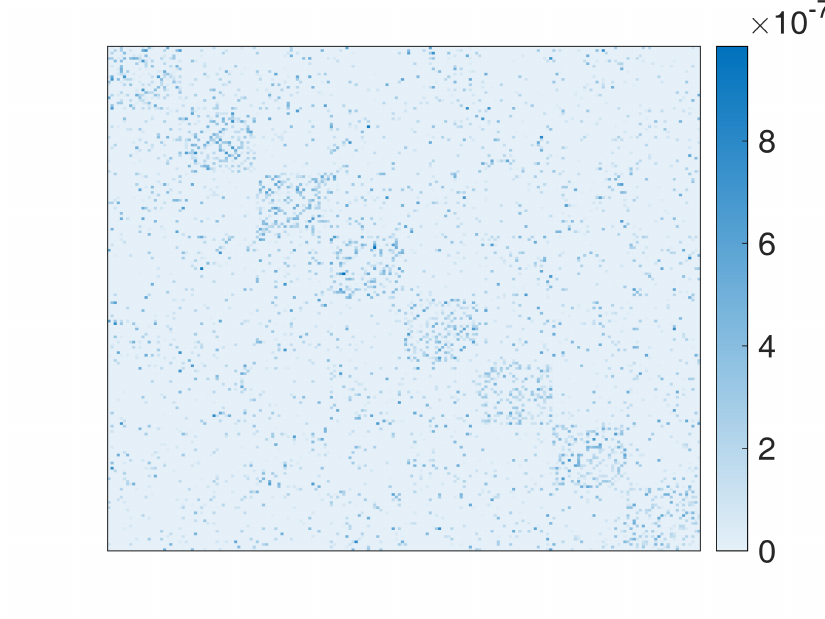}}
    \hspace{2em}
    \subfloat[Ours]{\includegraphics[width=0.2\linewidth]{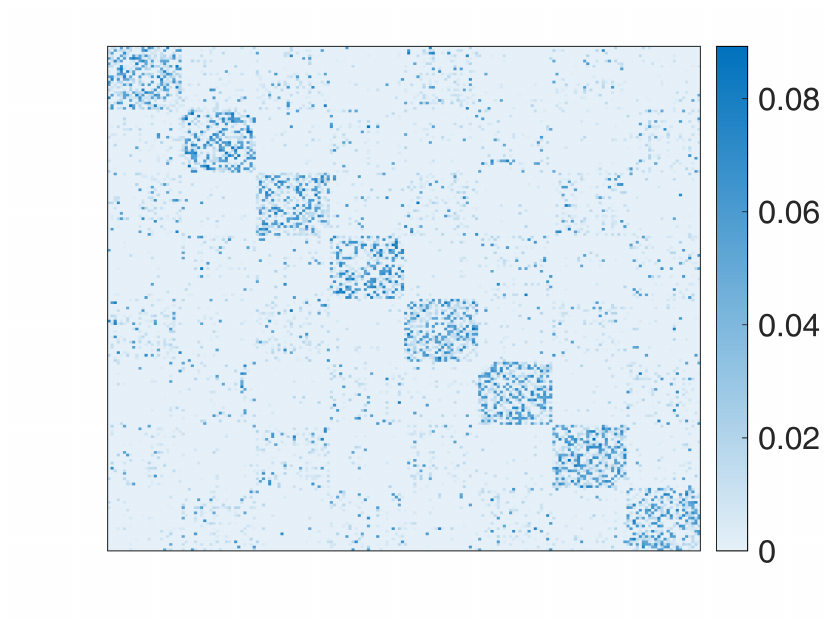}}
    	\caption{The visualization of the learned graphs (unnormalized weights) when $N =5000$ and $\sigma_i \sim \mathcal{U}(0.4,0.6)$.}
    	\label{Fig-visualization}
\end{figure*}

\subsection{Synthetic Data}

\textit{1) Model performance: }
We first compare our model with all baselines in four cases. We let $D= 192, K=4$, and $S = 2$. As listed in Table \uppercase\expandafter{\romannumeral2}, our model outperforms $k$means and Fairlets on clustering metrics because it exploits structured information behind raw data. The graphs established by KNNFSC and EpsNNFSC methods are not evaluated by $\mathrm{EE}$ since no edge weights are assigned. Among Models 3-5, CorrFSC achieves the best GL performance ($\mathrm{FS}$) as well as the best $\mathrm{CE}$ clustering performance ($\mathrm{CE}$). However, the graph construction performance of the three methods is inferior to our model, leading to unsatisfactory fair clustering results. Furthermore, compared with the three methods, our model unifies all seprate stages into a single optimization objective, avoiding suboptimality caused by separate optimization. The reason why our model outperforms FGLASSO could be that FGLASSO separately uses $k$means to obtain final cluster labels. Besides, our method could learn better graphs than FGLASSO. Although FJGSED and FSRSC also perform fair clustering in an end-to-end manner, our model obtains superior fair clustering performance due to more accurate graphs constructed by our method. Finally, our model has a node-adaptive graph filter to denoise observed signals. Thus, our model obtains the best graph construction performance under different levels of noise contamination.

We visualize the learned graphs of different methods in Fig.\ref{Fig-visualization}. We see that EpsNNFSC fails to capture the clustering structure, resulting in the worst fair clustering performance. The graph of KNNFSC tends to have imbalanced node degrees, and the graph of FSRSC has small edge weights. Compared with CorrFSC, FGLASSO, and FJGSED, the graph of our model has fewer noisy edges and clearer cluster structures.

\textit{2) The effect of $K$ and $S$: } We set $d = 192, N = 5000,  \sigma_i \sim\mathcal{U}(0.4,0.6)$. In the first case, we fix $S = 2$ and vary $K$ from 2 to 6. In the second case, we fix $K = 2$ and vary $S$ from 2 to 6. Fig.\ref{Fig-impact} displays that the fair clustering performance degrades with the increase of $K$ ($\mathrm{CE}$ increases and $\mathrm{Balance}$ decreases), which is consistent with Proposition 1. On the other hand, the fair clustering performance is less affected by $S$.

\begin{figure}[t] 
    \centering
	  \subfloat[]{
       \includegraphics[width=0.485\linewidth]{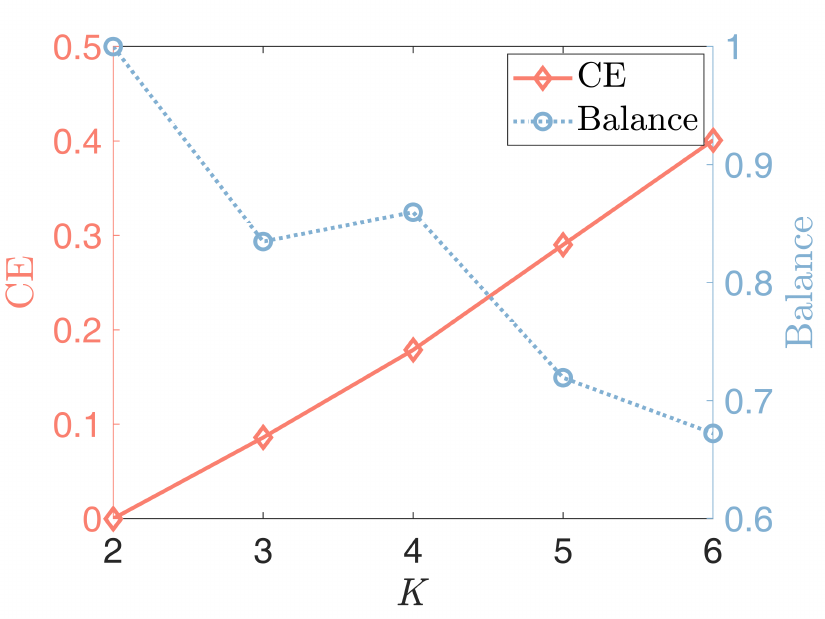}}
       \hspace{-1mm}
       \subfloat[]{
       \includegraphics[width=0.485\linewidth]{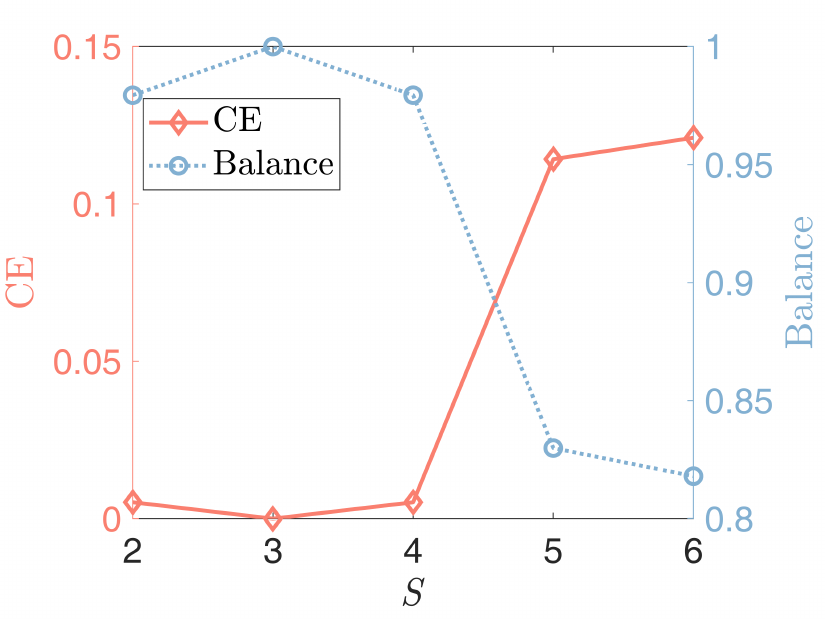}}
    	\caption{The effect of (a) $K$ and (b) $S$ on clustering results.}
    	\label{Fig-impact}
\end{figure}

\textit{3) The effect of $D$: } We let $N = 10^4,  \sigma_i \sim\mathcal{U}(0.4,0.6), K=4$ and $S = 2$. As depicted in Fig.\ref{Fig-node_num}, for a fixed data size,  $\mathrm{CE}$ first decreases and then increases as $D$. The reason may be that, as stated in Proposition \ref{theo-clustering performance}, the misclassification rate of FSC algorithms on the graph generated by the vSBM method decreases as $D$ if the underlying graph is exactly estimated. However, the quality of the estimated graph declines for large $D$ if $N$ is fixed. Thus, the second part of the error bound in Proposition \ref{theo-clustering performance} worsens. If the performance improvement brought by increasing $D$ is smaller than the degradation caused by the graph estimation error, fair clustering performance decreases when $D$ is large.  
\begin{figure}[t] 
    \centering
	  \subfloat[]{
    \includegraphics[width=0.475\linewidth]{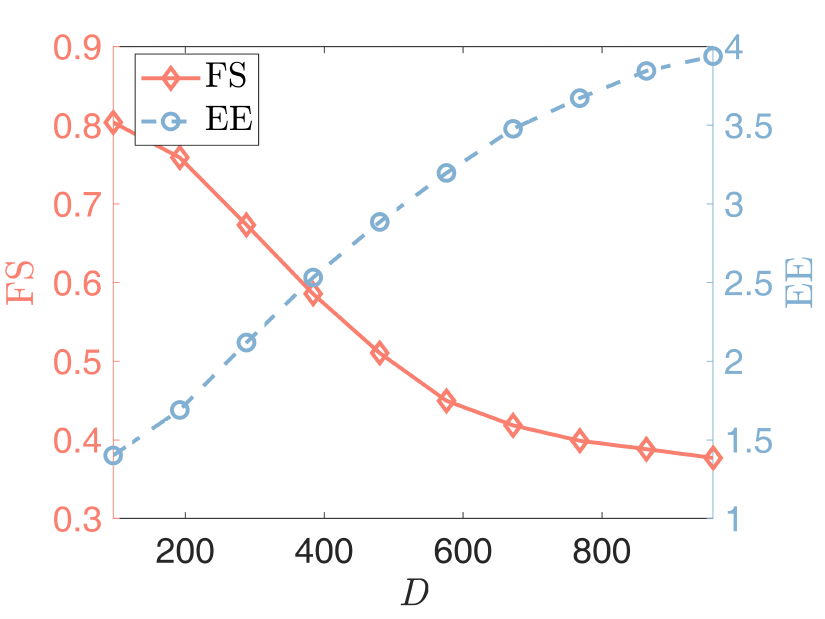}}
       \hspace{-1mm}
       \subfloat[]{ \includegraphics[width=0.475\linewidth]{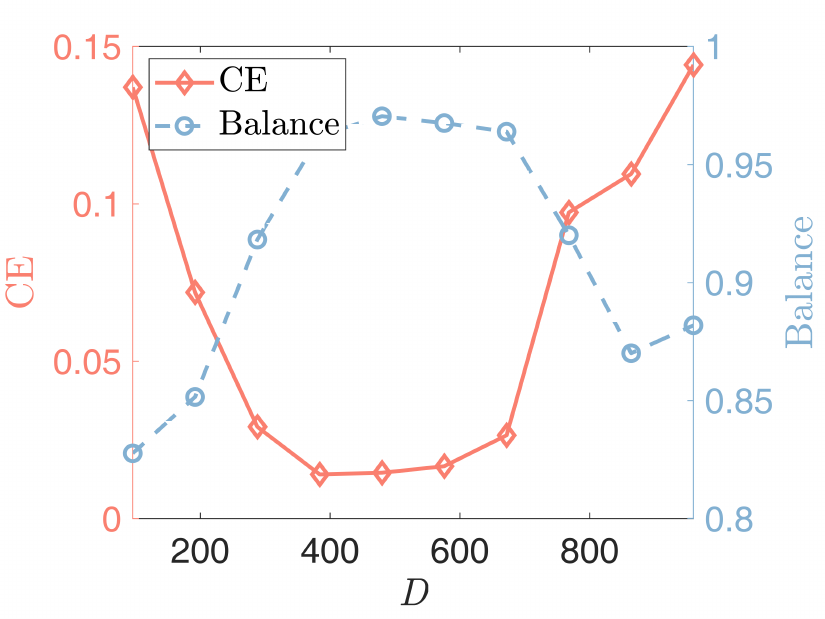} }
    	\caption{The effect of $D$ on (a) graph learning and (b) clustering.}
    	\label{Fig-node_num}
\end{figure}

\textit{4) The sensitivity of parameters: } We let $D = 196, K=4, S=2, N=5000$, and $,  \sigma_i \sim\mathcal{U}(0.4,0.6)$. First, we fix $\mu = 0.01$ and $\gamma = 0.01$ and vary $\beta$ and $\xi$ from $0.001$ to $0.1$. We then fix $\beta = 0.01$ and $\xi =0.05$ and vary $\mu$ and $\gamma$ from $0.001$ to $1$. As shown in Fig.\ref{Fig-sensitivity}, our model can achieve consistent GL and fair clustering performance except when $\beta$ is too small and $\xi$ is too large. Moreover, GL performance is more sensitive to $\mu$ than $\gamma$. There exist combinations of $\mu$ and $\gamma$ that achieve satisfactory $\mathrm{CE}$ and $\mathrm{Balance}$ simultaneously.

\textit{5) The effect of the fairness constraint:} We consider a special case where the real graph contains two clusters, each consisting of samples from the same sensitive group. In this case, the $\mathrm{Balance}$ of real clustering is zero. We then perform clustering using our FSC model and a variant where the fairness constraint is removed. As shown in Fig.\ref{Fig-fairness}, if we remove the fairness constraint, our model can exactly group all samples. However, some samples are misclassified to improve fairness in our model due to the effect of the fairness constraint. We list the corresponding model performance in Table \uppercase\expandafter{\romannumeral3}. Our model achieves a significantly higher $\mathrm{Balance}$ value at the cost of reduced clustering accuracy. Furthermore, the GL performance of our model is also degraded due to the fairness constraint. Thus, if the underlying graph has only one meaningful cluster that is highly unbalanced, fairness constraints may degrade GL and clustering performance.

\begin{table}[t]
\renewcommand{\arraystretch}{1}
	\centering
  
	\begin{threeparttable}
        \tabcolsep = 1.2em
	\caption{The results of removing fairness.}
	{\footnotesize 
	\begin{tabular}{c|c|c|c|c}
	\Xhline{1.05pt}


     & $\mathrm{FS}$ & $\mathrm{EE} $  & $\mathrm{CE}$  & $\mathrm{Bal}$  \\
     
    \hline

    w/o fairness 
     &0.734 &1.211 &0 &0\\  	
     Ours 
     &0.719 &1.353 &0.484 &0.867\\  	

	\Xhline{1.2pt}

	\end{tabular} 
	}

 \end{threeparttable}
 
 \label{table-fairness}
     	\vspace{-2em}
\end{table}

\begin{figure}[t] 
    \centering
	  \subfloat[The ground-truth]{
       \includegraphics[width=0.32\linewidth]{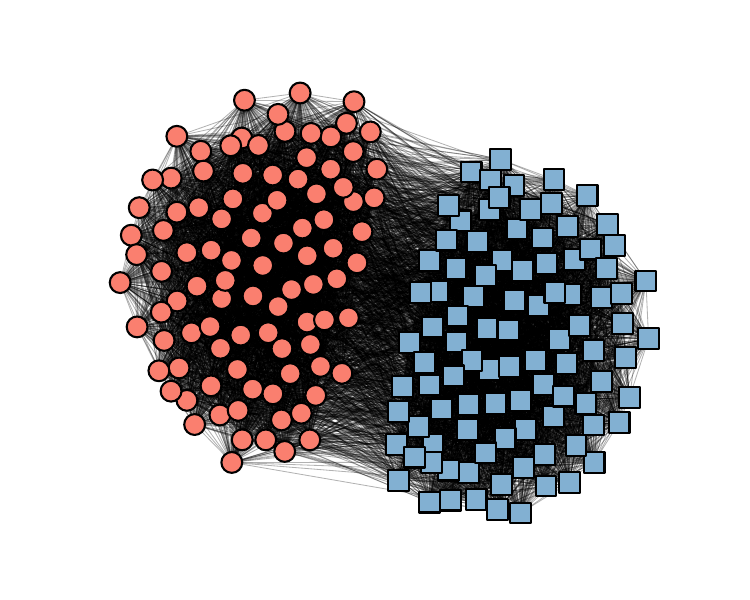}}
       \hspace{-2mm}
       \subfloat[w/o fairness]{
       \includegraphics[width=0.32\linewidth]{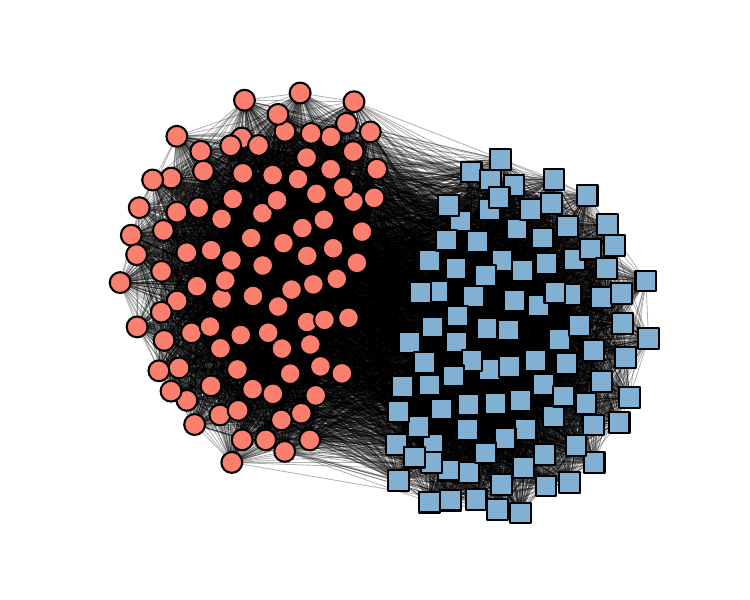}}
       \subfloat[Ours]{
       \includegraphics[width=0.32\linewidth]{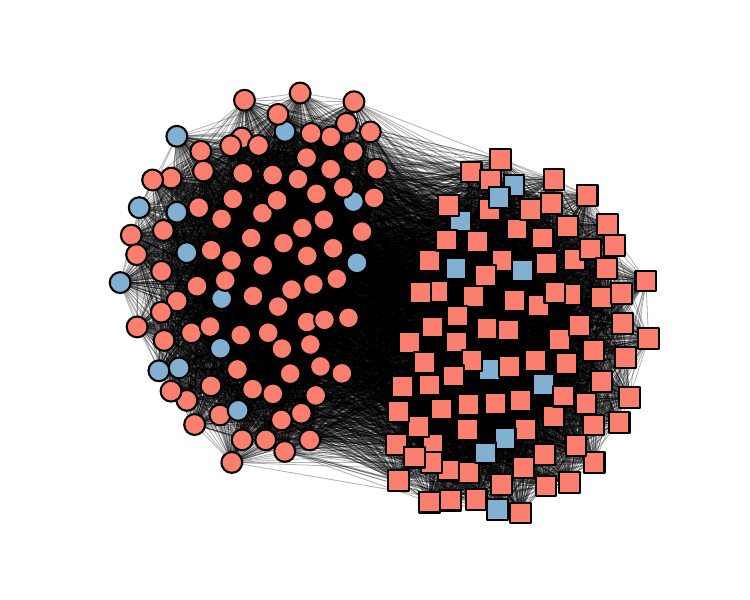}}
    	\caption{The effect of the fairness constraint. Colors represent clusters, while mark shapes represent sensitive groups.}
    	\label{Fig-fairness}
\end{figure}

\textit{6) Ablation study:}  Three cases are taken into consideration. (\romannumeral1)  We construct graphs using \eqref{eq-prelim-1}, conduct fair spectral embedding, and discretize using spectral rotation separately to test the benefit of a unified model (Ours-Sep). (\romannumeral2) We construct graphs using \eqref{eq-prelim-1-1} and conduct fair spectral embedding jointly. After obtaining continuous results, we exploit $k$means as the discretization step to test the benefit of spectral rotation (Ours-$k$means). (\romannumeral3) We remove the denoising module in our model to test the benefit of the node-adaptive graph filter (Ours-noDN). We let $D = 196, K=4, S=2, N=5000$, and $\sigma_i \sim\mathcal{U}(0.4,0.6)$, and the results are listed in Table \uppercase\expandafter{\romannumeral4}. 
Our model outperforms Ours-Sep, demonstrating the benefit of a unified model. Although the graph of Ours-$k$means is well estimated, it obtains the worst $\mathrm{CE}$ due to the poor performance of $k$means. Our model outperforms Ours-noDN because it has a low-pass filter to enhance graph construction.

\begin{table}[t]
\renewcommand{\arraystretch}{1}
	\centering
  
	\begin{threeparttable}
        \tabcolsep = 1.2em
	\caption{The results of ablation studies.}
	{\footnotesize 
	\begin{tabular}{c|c|c|c|c}
	\Xhline{1.05pt}


     & $\mathrm{FS}$ & $\mathrm{EE} $  & $\mathrm{CE}$  & $\mathrm{Bal}$  \\
     
    \hline

    Ours-Sep 
     &0.637 &2.333 &0.250 &0.704\\  	
     Ours-$k$means 
     &0.635 &2.320 &0.549 &0.353\\  	
     Ours-noDN
     &0.623 &2.354 &0.276 &0.694\\  	
     Ours 
     &0.658 &2.203 &0.167 &0.782\\

	\Xhline{1.2pt}

	\end{tabular} 
	}

 \end{threeparttable}
 
 \label{table-ablation}
     	\vspace{-1em}
\end{table}

\begin{figure*}[t] 
    \centering
	  \subfloat[$\mathrm{FS}$]{
       \includegraphics[width=0.23\linewidth]{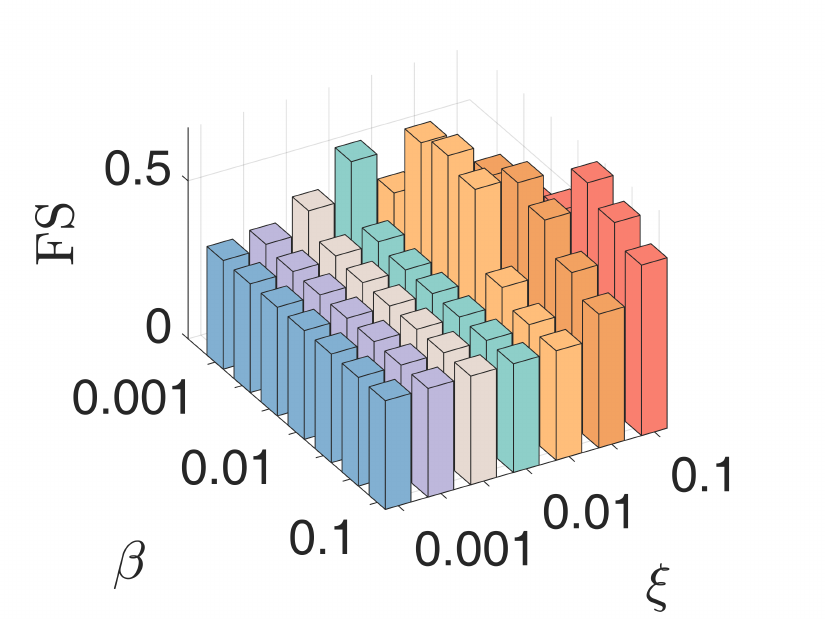}}
       \hspace{-1em}
	  \subfloat[$\mathrm{EE}$]{
       \includegraphics[width=0.23\linewidth]{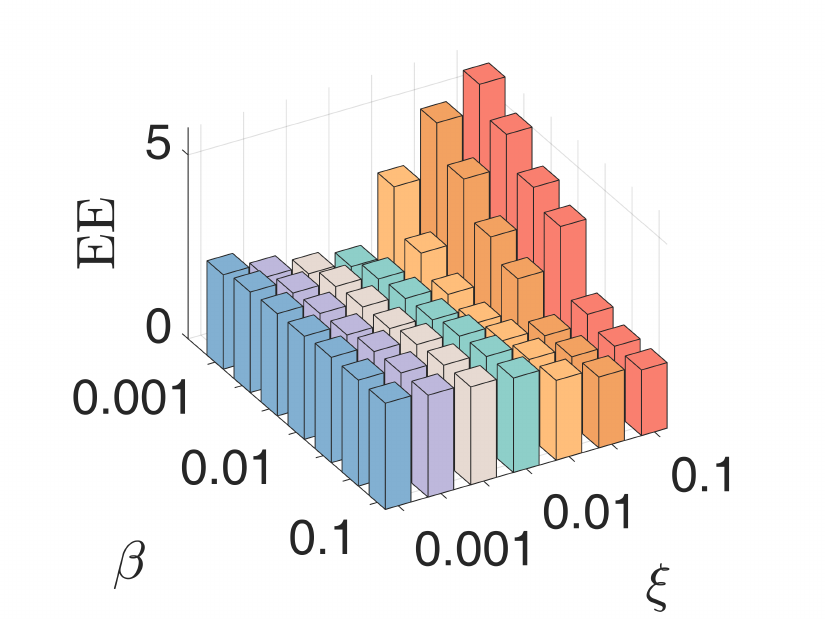}}
       \hspace{-1em}
	  \subfloat[$\mathrm{CE}$]{
       \includegraphics[width=00.23\linewidth]{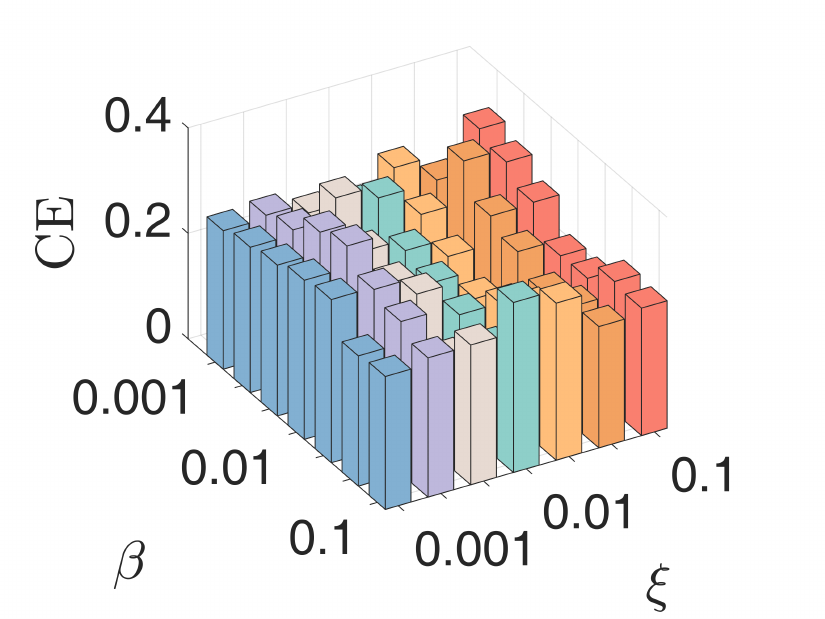}}
       \hspace{-1em}
	  \subfloat[$\mathrm{Balance}$]{
       \includegraphics[width=0.23\linewidth]{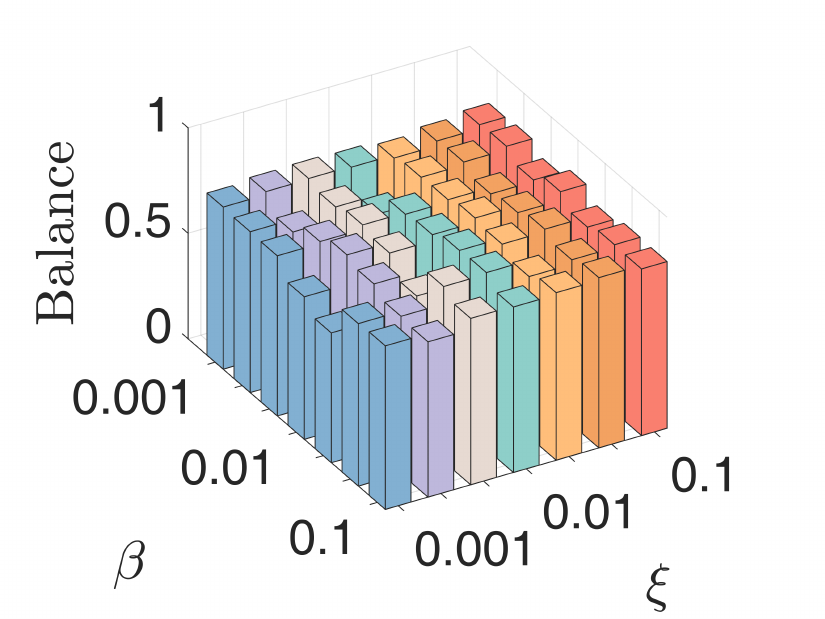}}

       \subfloat[$\mathrm{FS}$]{
       \includegraphics[width=0.23\linewidth]{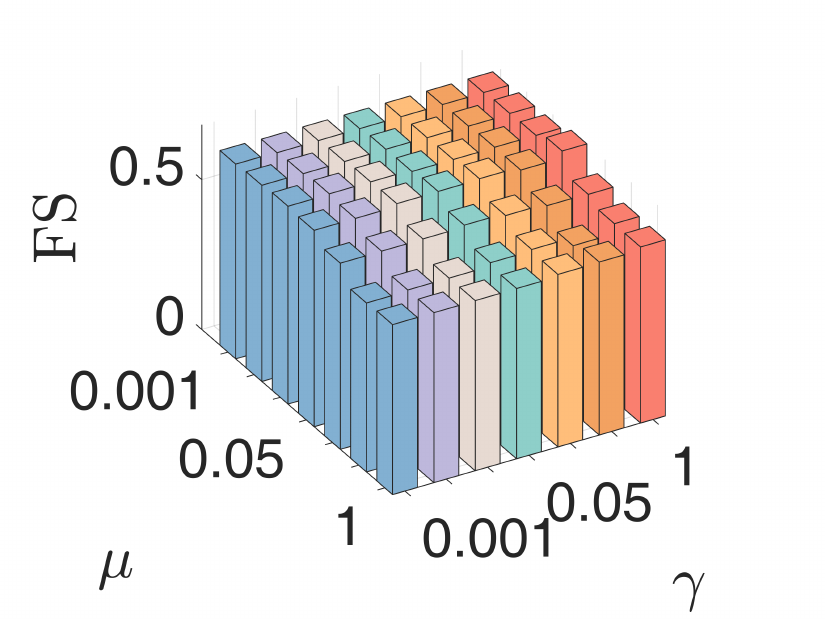}}
       \hspace{-1em}
	  \subfloat[$\mathrm{EE}$]{
       \includegraphics[width=0.23\linewidth]{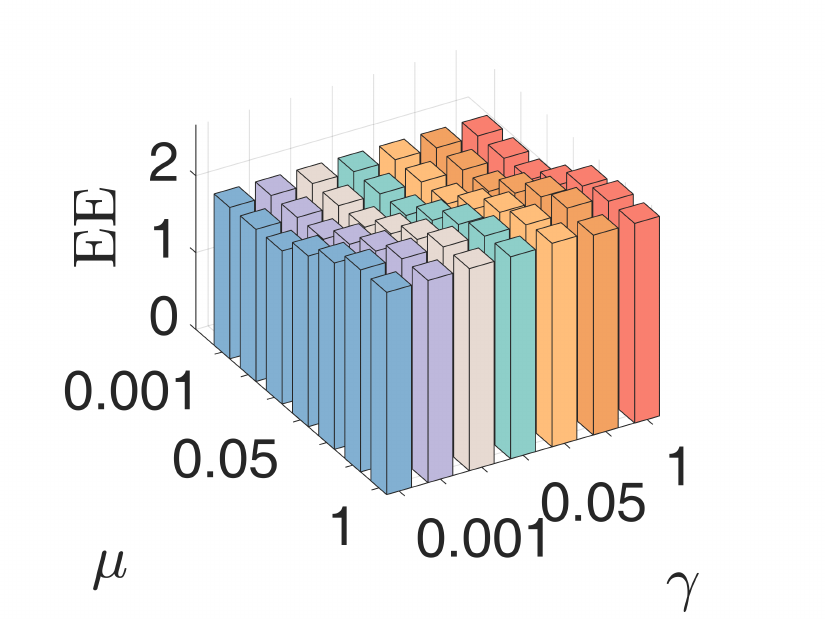}}
         \hspace{-1em}
	  \subfloat[$\mathrm{CE}$]{
       \includegraphics[width=0.23\linewidth]{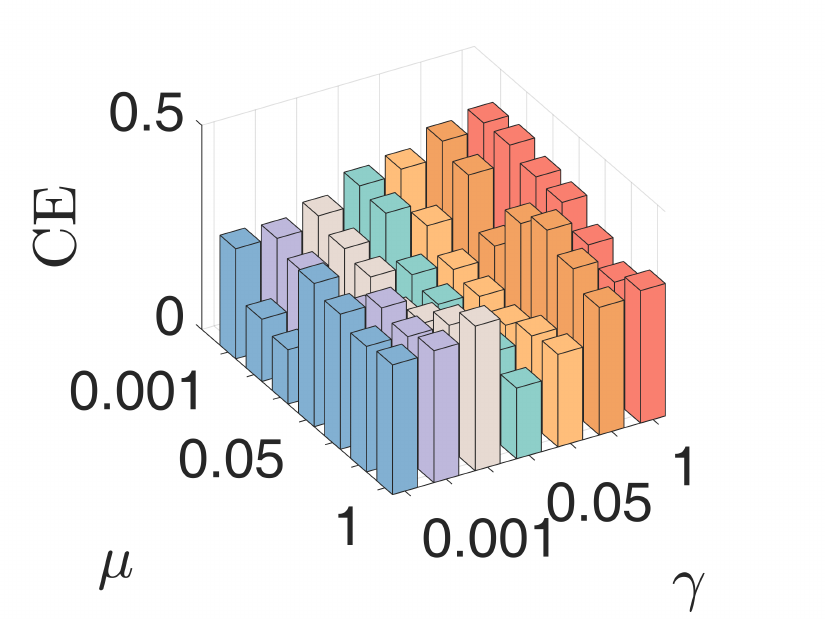}}
       \hspace{-1em}
	  \subfloat[$\mathrm{Balance}$]{
       \includegraphics[width=0.23\linewidth]{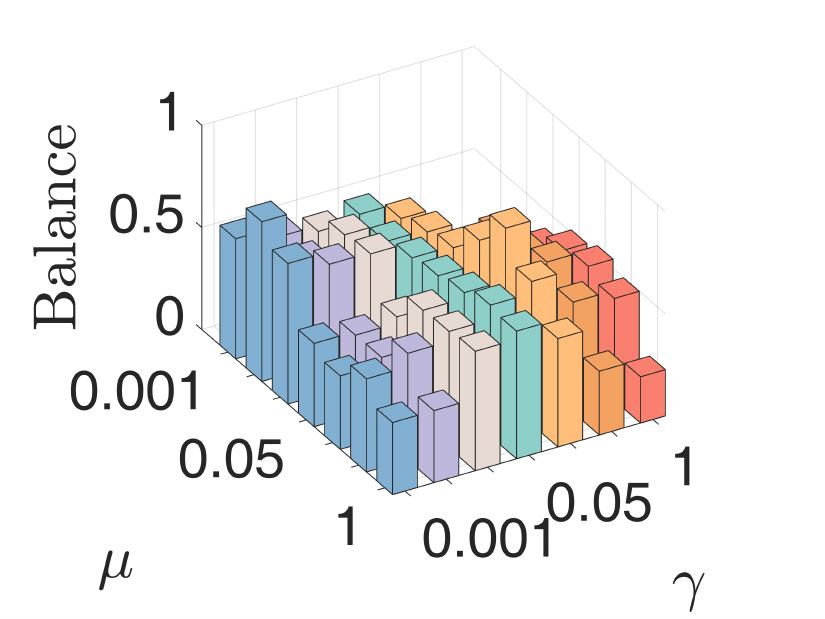}}
    	\caption{The effect of parameter sensitivity. (a)-(d) The results of varying $\xi$ and $\beta$. (e)-(h) The results of varying $\mu$ and $\gamma$. }
    	\label{Fig-sensitivity}
\end{figure*}

\begin{figure}[t] 
    \centering
	  \subfloat[$\sigma_i \sim\mathcal{U}(0,0.2)$]{
       \includegraphics[width=0.485\linewidth]{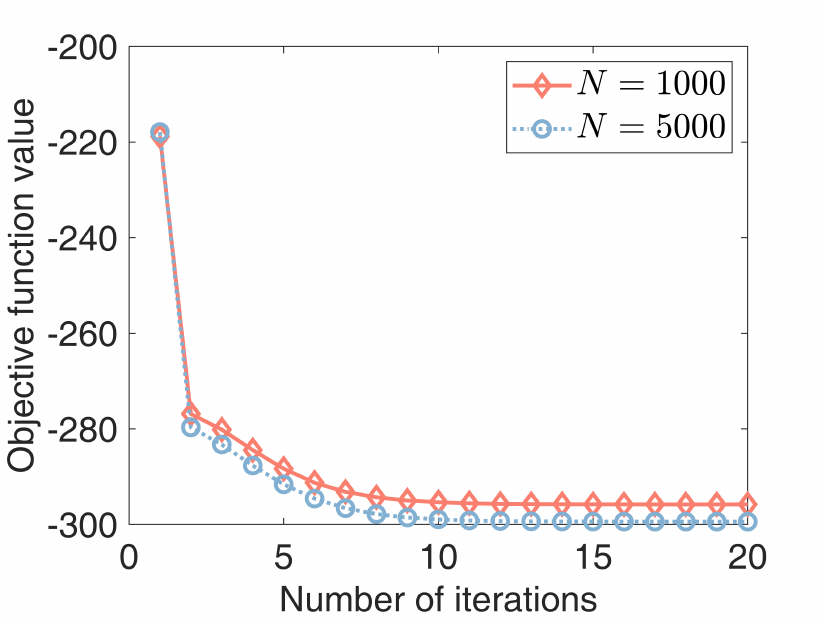}}
       \hspace{-2mm}
       \subfloat[$\sigma_i \sim\mathcal{U}(0.4,0.6)$]{
       \includegraphics[width=0.485\linewidth]{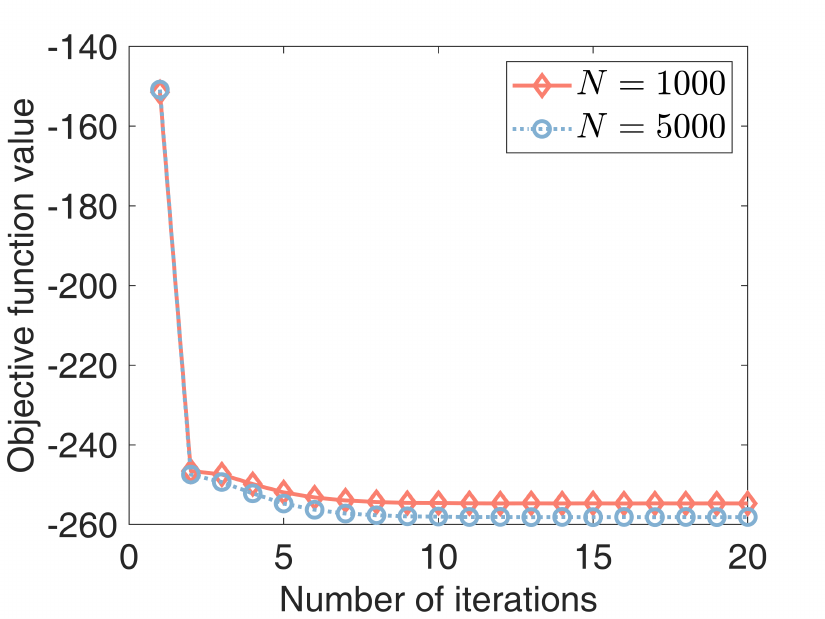}}
    	\caption{The convergence of our algorithm.}
    	\label{Fig-comvergence}
\end{figure}

\begin{figure*}[t] 
    \centering
	 \subfloat[]{\includegraphics[width=0.25\linewidth,height=0.19\linewidth]{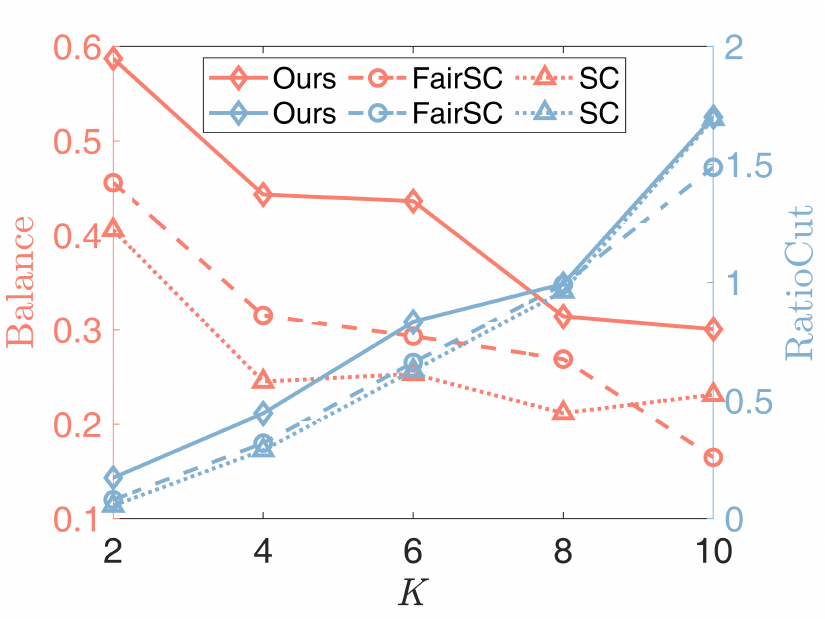}}
	 \subfloat[]{
  \includegraphics[width=0.25\linewidth,height=0.19\linewidth]{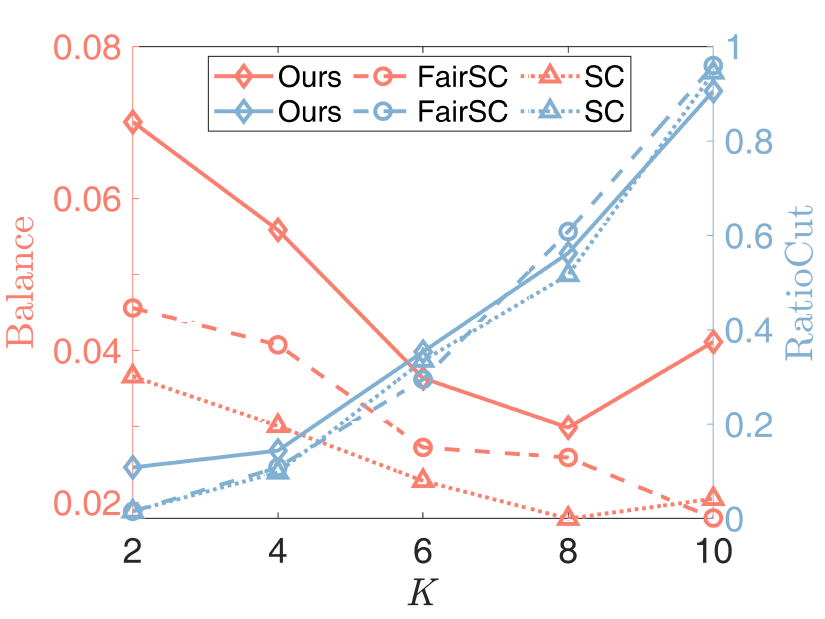}}
  \hspace{-1.5ex}
	   \subfloat[]{ \includegraphics[width=0.25\linewidth,height=0.19\linewidth]{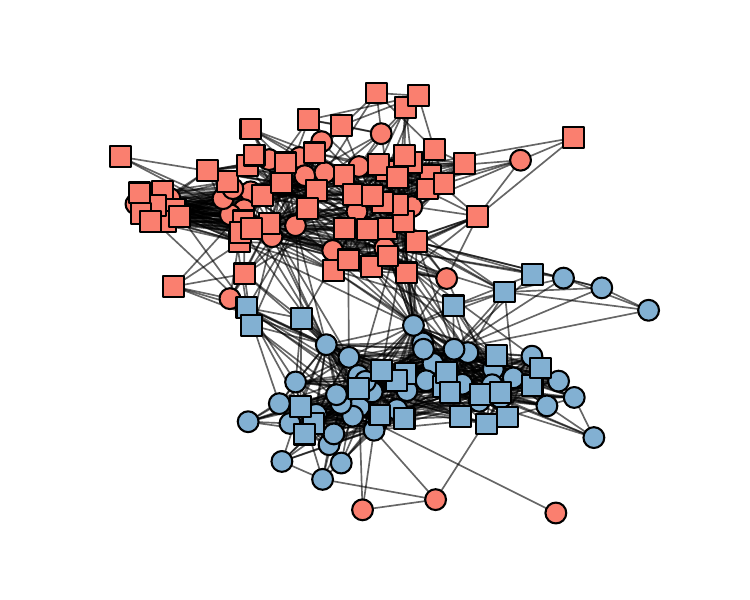}}
           \hspace{-5ex}
    	 \subfloat[]{       \includegraphics[width=0.25\linewidth,height=0.19\linewidth]{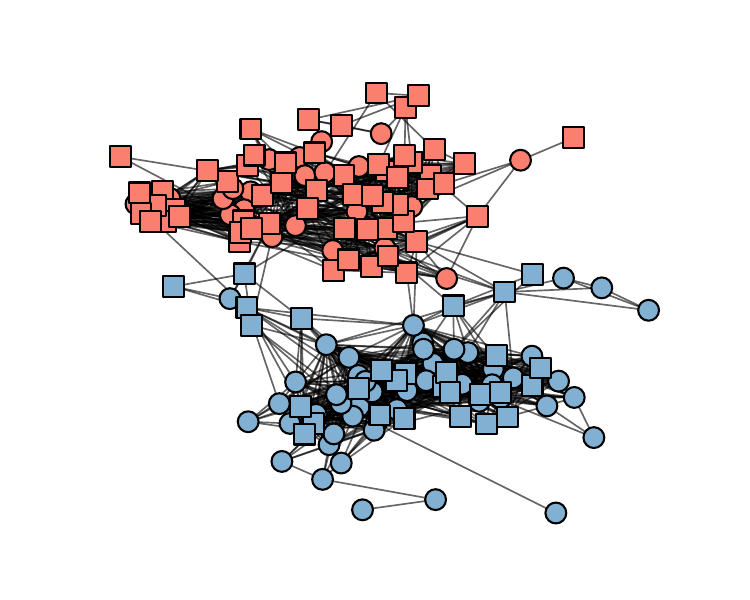}}   
    	\caption{The results of the benchmark datasets. (a)-(b) The fair clustering results of the FACEBOOK and DRUGNET datasets. (c)-(d) The real and the learned FACEBOOK network. Colors represent clusters, while mark shapes represent sensitive groups. }
    	\label{Fig-facebook}
    	\vspace{-1.5em}
\end{figure*}
\textit{7) Convergence: } Finally, we test the convergence of our algorithm. We let $D = 196, K=4, S=2,D = 192$.  As shown in Fig.\ref{Fig-comvergence}, the objective function values monotonically decrease as the number of iterations. Besides, our algorithm converges within a few iterations, indicating its fast convergence.

\subsection{Benchmark Data}
In this section, we test the performance of our model on the commonly used benchmark datasets of FSC \cite{kleindessner2019guarantees}.  The first dataset is a high school friendship network named FACEBOOKNET\footnote{http://www:sociopatterns:org/datasets/high-school-contact-and-friendshipnetworks/}. The dataset contains a graph with vertices representing high school students and edges representing connections between students on Facebook. After data preprocessing, we obtain 155 students split into male and female groups. In this dataset, gender is considered a sensitive attribute. All vertices are divided into two groups, i.e., male and female. The second dataset, DRUGNET, is a network encoding acquaintanceship between drug users in Hartford \footnote{ https://sites:google:com/site/ucinetsoftware/datasets/covert-networks/drugnet}. After data preprocessing, we obtain 193 vertices.  We use ethnicity as a sensitive attribute and split the vertices into three groups: African Americans, Latinos, and others. Note that previous FSC work \cite{kleindessner2019guarantees} is based on a given graph, and the two datasets only contain ground-truth graphs and no observed signals. However, one of the primary advantages of our model is that we can group observed data without the real graph structures.  Thus, we generate data via \eqref{signal-genegration} based on the ground-truth networks. We then use our model to group vertices via the observed data. For comparison, we apply the FSC algorithm in \cite{kleindessner2019guarantees}  (FairSC) and unnormalized spectral clustering (SC) to the real networks to cluster vertices. We aim to demonstrate that our model can achieve competitive fair clustering performance even without real graphs. Referring to \cite{kleindessner2019guarantees}, we use $\mathrm{Balance}$ and $\mathrm{RatioCut}$ as evaluation metrics since we have no real labels.  We let $N = 1000$ and $\sigma_i \sim\mathcal{U}(0,0.2)$. As displayed in \ref{Fig-facebook} (a)-(b), for the two datasets, our model achieves almost the same $\mathrm{RatioCut}$ as FairSC and SC\textemdash which are based on the ground-truth networks\textemdash even though we do not know the underlying graphs. However, compared with FSC and SC, our model can improve $\mathrm{Balance}$ at a moderate sacrifice of  $\mathrm{RatioCut}$. Figures \ref{Fig-facebook} (c)-(d) depict the real FACEBOOKNET graph and the graph learned by our model when $K =2$. Fewer edges are learned between two clusters, suggesting that our model may tend to learn a graph that is more suitable for clustering. Furthermore, two clusters are observed from our learned graph, meaning our model can fairly partition the nodes from the observed data even if we have no real graphs.

\begin{figure*}[t] 
    \centering
	\subfloat[]{\includegraphics[width=0.25\linewidth]{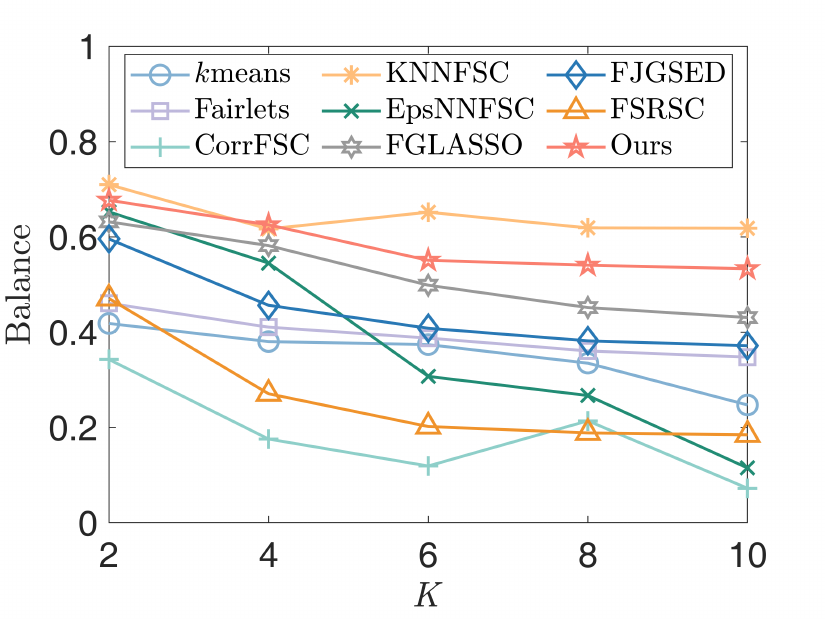}}
     \hspace{1em}
    \subfloat[]{\includegraphics[width=0.25\linewidth]{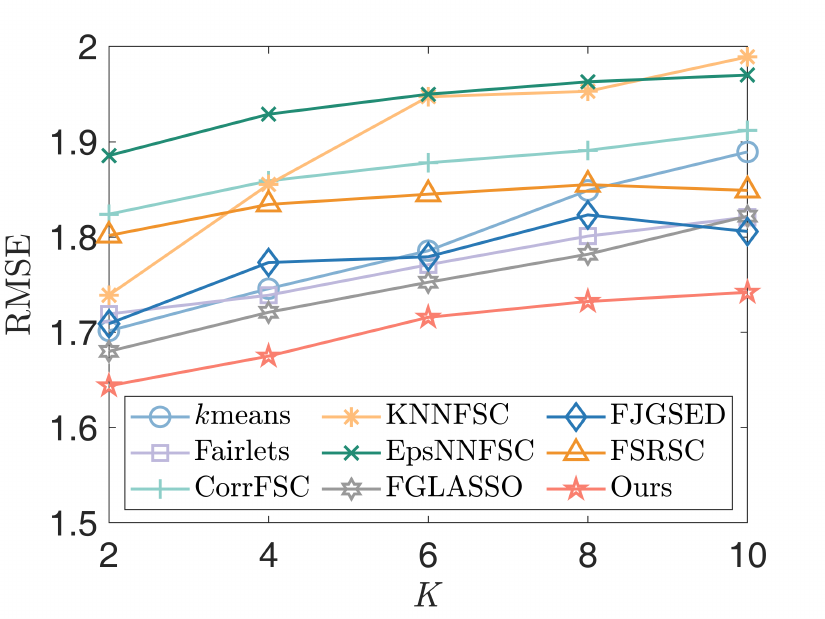}}
     \hspace{1em}
	  \subfloat[]{\includegraphics[width=0.18\linewidth]{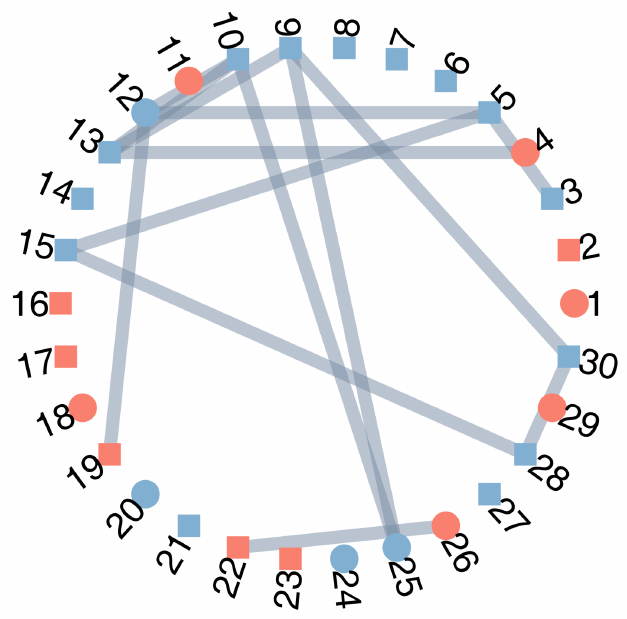}}
     \hspace{1em}
       \subfloat[]{\includegraphics[width=0.18\linewidth]{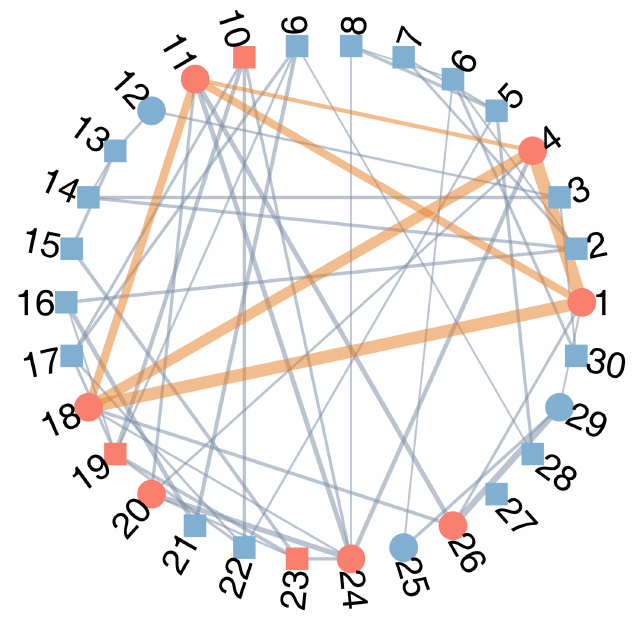}}
    	\caption{The results of the MovieLens dataset. (a)-(b) The fair clustering results of different models.  (c)-(d) The learned sub-graphs of KNNFSC and our model when $K=2$. Colors represent clusters, while mark shapes represent sensitive attributes. }
    	\label{Fig-movie-graph}
\end{figure*}

\subsection{Real Data}

\textit{1) MovieLens 100K dataset: } 
We employ MovieLens 100K dataset\footnote{http://www.grouplens.org} to group movies by their ratings. This dataset contains ratings of 1682 movies by 943 users in the range $[1,5]$, which is sparse as many movies have few ratings. To alleviate the impact of sparsity, we select the top 200 most-rated movies from 1682 movies. Therefore, we have a who-rated-what matrix $\mathbf{X}\in \mathbb{R}^{200\times 943}$. The matrix can be used to construct a movie-movie similarity graph strongly correlated with how users explicitly rate items \cite{wang2015collaborative}. Therefore, we can perform clustering on the similarity graph to group movies with similar attributes. However, as stated in \cite{ tarzanagh2021fair}, old movies tend to obtain higher ratings because only masterpieces have survived. To obtain fair results unbiased by production time, we consider movie year as a sensitive attribute. Movies made before 1991 are considered old, while others are considered new.  To evaluate clustering results, we conduct traditional item-based collaborative filtering (CF) on each cluster, termed cluster CF, to predict user ratings of movies. As claimed in \cite{tarzanagh2021fair}, if the obtained clusters accurately contain a set of similar items,  cluster CF can better predict user ratings of movies. Therefore, we follow \cite{tarzanagh2021fair} and use root mean square error ($\mathrm{RMSE}$) between the predicted and true ratings as an evaluation metric in addition to $\mathrm{Balance}$ \cite{tarzanagh2021fair, wang2015collaborative}. Figures \ref{Fig-movie-graph} (a)-(b) depict fair clustering results of different models. Our model obtains the highest $\mathrm{Balance}$ except KNNFSC.  However,  KNNFSC performs poorly on $\mathrm{RMSE}$, indicating unsatisfactory clustering results. This may be caused by the fact that the graph constructed by KNNFSC hardly characterizes the similarity relationships between movies. In contrast, our model achieves the best $\mathrm{RMSE}$ since it better reveals similarity relationships behind observed data. In Fig.\ref{Fig-movie-graph}(c)-(d), we provide the learned sub-graphs and clustering results of the top 30 rated movies when $K = 2$. The graph learned by KNNFSC has isolated nodes since they are connected to the movies outside the top 30 rated movies. In our graph, nodes 1, 4, 11, and 18 are closely connected because they belong to the Star Wars series. However, in Fig.\ref{Fig-movie-graph}(c), they are not connected. Moreover, our model successfully groups the four nodes into the same cluster.

\textit{2) MNIST-USPS dataset: }The second dataset we employ is  MNIST-USPS dataset \footnote{http://yann.lecun.com/exdb/mnist, https://www.kaggle.com/bistaumanga/usps-dataset}, which contains two sub-datasets, i.e.,  MNIST and USPS. The two sub-datasets contain images of handwritten digits from 0 to 9. We cluster these images and use digits as the ground-truth cluster labels. Specifically, we randomly select 48 images from each sub-dataset, which contains four digits and twelve pictures for each digit. We finally obtain 96 images and resize each image to a $28\times28$ matrix. We take each image as a node in a graph and flatten the corresponding matrix as the node signals. Therefore, the observed data are $\mathbf{X}\in \mathbb{R}^{96\times784}$. We take the domain source of images\textemdash images from MNIST or USPS\textemdash as a sensitive attribute. Thus, we have $S = 2$ and $K = 4$. We use $\mathrm{CE}$ and $\mathrm{Balance}$ as evaluation metrics since we have real labels but no ground-truth graphs. As shown in Fig.\ref{Fig-mnist-usps}, our model achieves the best fair clustering performance for both $\mathrm{CE}$ and $\mathrm{Balance}$, indicating its superiority. The reason for CorrFSC, KNNFSC, EpsFSC, and FSRSC achieving unsatisfactory $\mathrm{CE}$ may be that the corresponding graphs for this dataset cannot reflect the real topological similarity.

\begin{figure}[t] 
    \centering
	  {
       \includegraphics[width=0.9\linewidth]{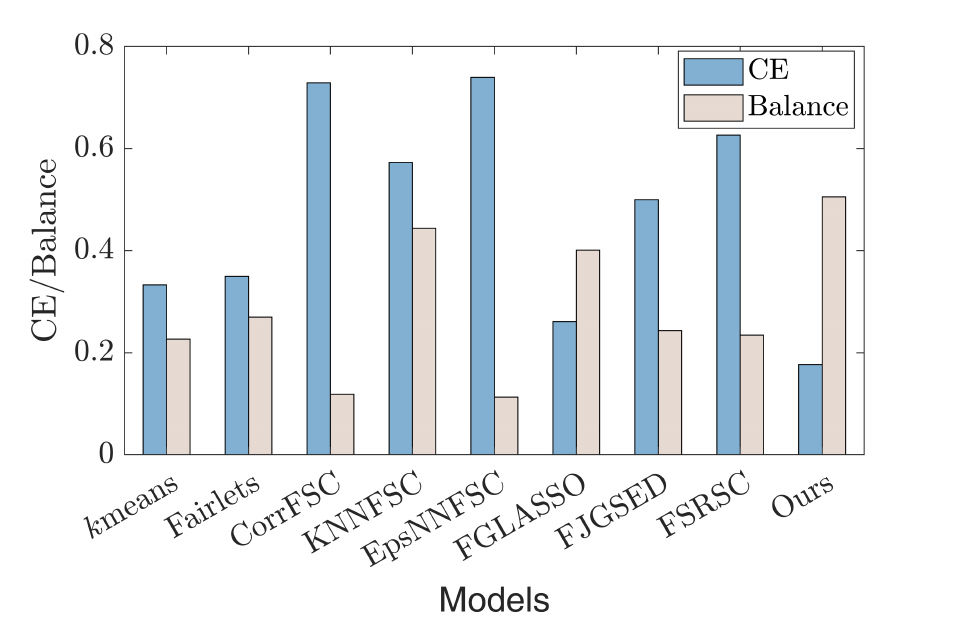}}
    	\caption{The clustering results of the MNIST-USPS dataset.}
    	\label{Fig-mnist-usps}
\end{figure}

\section{Conclusion}
\label{sec:Conclusion}
In this paper,  we theoretically analyzed the impact of similarity graphs on FSC performance. Motivated by the analysis, we proposed a graph construction method for FSC tasks as well as an end-to-end FSC framework. Then, we designed an efficient algorithm to alternately update the variables corresponding to each submodule in our model. Extensive experiments showed that our approach is superior to state-of-the-art (fair) SC models. Future research directions may include developing more scalable FSC algorithms.

\appendices


\section{Proof of Proposition \ref{theo-clustering performance}}
\label{appendix2}
We first provide the following lemma.
\begin{lemma}
For any $\epsilon>0$ and any two matrices $\mathbf{U}, \widehat{\mathbf{U}}\in\mathbb{R}^{D\times K}$ such that $\mathbf{U} =\mathbf{Q}\mathbf{R}$ with $\mathbf{Q} \in\mathcal{I}$ and $\mathbf{R}^{\top}\mathbf{R} = \mathbf{I}$, let $(\widehat{\mathbf{Q}}, \widehat{\mathbf{R}})$ be a $(1+\epsilon)$ approximation of $\widehat{\mathbf{U}}$ using spectral rotation as Assumption \ref{assump-1}, and $\breve{\mathbf{U}} = \widehat{\mathbf{Q}}\widehat{\mathbf{R}}$. Then, for any $\delta_k\geq 0 $, define $\widetilde{\mathcal{M}}_k = \left\{i\in\mathcal{C}_k: \lVert \mathbf{U}_{[i,:]} - \breve{\mathbf{U}}_{[i,:]}\rVert_2\geq \delta_k/2 \right\}$, $k=1,...,K$, and we have
\begin{shrinkfix}
\begin{align}
   \sum_{k=1}^K |\widetilde{\mathcal{M}}_k| \delta_k^2 \leq 4(4+2\epsilon)\lVert  {\mathbf{U}} - \widehat{\mathbf{U}}\rVert_{\mathrm{F}}^2,
    \label{eq-lemma-1}
    \end{align}
\end{shrinkfix}
\label{lemma-1}
   \vspace{-1em}
\end{lemma}
\begin{proof}
First, by the procedure of spectral rotation, we have  
\begin{shrinkfix}
\begin{align}
& \widehat{\mathbf{Q}},  \widehat{\mathbf{R}} = \underset{\mathbf{Q}, \mathbf{R}}{\min}\; \lVert \mathbf{Q} -  \widehat{\mathbf{U}}  {\mathbf{R}}\rVert_{\mathrm{F}}^2\;\;\; \mathrm{s.t.}\;\; \mathbf{R}^{\top} \mathbf{R} = \mathbf{I}, \; \mathbf{Q}\in \mathcal{I} \notag \\
\Leftrightarrow& \widehat{\mathbf{Q}},  \widehat{\mathbf{R}} = \underset{\mathbf{Q}, \mathbf{R}}{\min}\; \lVert \widehat{\mathbf{U}} - \mathbf{Q}   {\mathbf{R}}\rVert_{\mathrm{F}}^2\;\;\; \mathrm{s.t.}\;\; \mathbf{R}^{\top} \mathbf{R} = \mathbf{I}, \; \mathbf{Q}\in \mathcal{I}.
    \label{eq-theory-6}
    \end{align}
\end{shrinkfix}
Then, based on Assumption \ref{assump-1}, we can obtain that 
\begin{shrinkfix}
\begin{align}
    &\lVert \widehat{\mathbf{U}} - \widehat{\mathbf{Q}}   \widehat{\mathbf{R}}\rVert_{\mathrm{F}}^2 \leq (1+\epsilon)  \underset{\mathbf{Q}\in\mathcal{I}, \mathbf{R}^{\top}\mathbf{R} = \mathbf{I}}{\min}\lVert \widehat{\mathbf{U}} - {\mathbf{Q}}  {\mathbf{R}}\rVert_{\mathrm{F}}^2 \notag\\
    \Rightarrow&\lVert \widehat{\mathbf{U}} - \breve{\mathbf{U}} \rVert_{\mathrm{F}}^2 \leq (1+\epsilon) \lVert \widehat{\mathbf{U}} - \mathbf{U}\rVert_{\mathrm{F}}^2. 
 \label{eq-lemma-proof-1}
    \end{align}
\end{shrinkfix}
It is not difficult to obtain the following inequalities
\begin{shrinkfix}
\begin{align}
&\lVert \breve{\mathbf{U}} - \mathbf{U}\rVert_{\mathrm{F}}^2 \leq 2 
 \lVert \breve{\mathbf{U}} - \widehat{\mathbf{U}}\rVert_{\mathrm{F}}^2 + 2\lVert \widehat{\mathbf{U}} - \mathbf{U}\rVert_{\mathrm{F}}^2  \notag\\
 &\leq (4 +2\epsilon)  \lVert \widehat{\mathbf{U}} -{\mathbf{U}}\rVert_{\mathrm{F}}^2.
 \label{eq-lemma-proof-2}
    \end{align}
\end{shrinkfix}
The first inequality holds due to the basic inequality, and the second one holds due to \eqref{eq-lemma-proof-1}. Finally, according to the definition of $\delta_k$, we can obtain the conclusion \eqref{eq-lemma-1}.
    \label{proof-lemma}
       \vspace{-0.5em}
\end{proof}

We start our proof of Proposition \ref{theo-clustering performance}, which is inspired by \cite{kleindessner2019guarantees}. To incorporate the fairness constraint into the objective function of  \eqref{eq-prelim-8}, we let $\widehat{\mathbf{U}} = \mathbf{Z}\widehat{\mathbf{Y}}$, where $\widehat{\mathbf{Y}}$ contains the eigenvectors of $\mathbf{Z}^{\top}\widehat{\mathbf{L}}\mathbf{Z}$ corresponding to the $K$ smallest eigenvalues. Suppose that  $\bar{\mathbf{Y}}$ contains the eigenvectors of $\mathbf{Z}^{\top}\bar{\mathbf{L}}\mathbf{Z}$ corresponding to the $K$ smallest eigenvalues, where $\bar{\mathbf{L}}$ is the expected Laplacian matrix of the graphs generated by the vSBM method. We apply spectral rotation on $\widehat{\mathbf{U}}$ estimated from $\widehat{\mathbf{L}}$ by solving \eqref{eq-prelim-8}. For any $\mathbf{V}\in\mathbb{R}^{K\times K}$ satisfying $\mathbf{V}^{\top} \mathbf{V} = \mathbf{I},  \mathbf{V}\mathbf{V}^{\top} = \mathbf{I}$, it is not difficult to obtain
\begin{shrinkfix}
\begin{align}
\lVert 
&\mathbf{Z}\bar{\mathbf{Y}} - \mathbf{Z}\widehat{\mathbf{Y}} \mathbf{V}
\rVert_{\mathrm{F}}^2 = \mathrm{Tr}\left( (\bar{\mathbf{Y}} -\widehat{\mathbf{Y}} \mathbf{V})^{\top}\mathbf{Z}^{\top}\mathbf{Z}(\bar{\mathbf{Y}} -\widehat{\mathbf{Y}} \mathbf{V}) \right) \notag\\
= &\lVert 
\bar{\mathbf{Y}} - \widehat{\mathbf{Y}} \mathbf{V}
\rVert_{\mathrm{F}}^2.
    \label{eq-theory-9}
    \end{align}
\end{shrinkfix}
Therefore, we have 
\begin{shrinkfix}
\begin{align}
 &\underset{\mathbf{V}^{\top} \mathbf{V} = \mathbf{I},  \mathbf{V}\mathbf{V}^{\top} = \mathbf{I}}{\min}\; \lVert 
\mathbf{Z}\bar{\mathbf{Y}} - \mathbf{Z}\widehat{\mathbf{Y}} \mathbf{V}
\rVert_{\mathrm{F}}  
= \underset{\mathbf{V}^{\top} \mathbf{V} = \mathbf{I},  \mathbf{V}\mathbf{V}^{\top} = \mathbf{I}}{\min}\; \lVert 
\bar{\mathbf{Y}} - \widehat{\mathbf{Y}} \mathbf{V}
\rVert_{\mathrm{F}} \notag\\
\leq &\frac{8\sqrt{2K^3}}{D(c-d)}\lVert 
\mathbf{Z}^{\top}\bar{\mathbf{L}}\mathbf{Z} - \mathbf{Z}^{\top}\widehat{\mathbf{L}}\mathbf{Z} 
\rVert_2  
\leq\frac{8\sqrt{2K^3}}{D(c-d)}\lVert 
\mathbf{Z}^{\top}\bar{\mathbf{L}}\mathbf{Z} - \mathbf{Z}^{\top}\widehat{\mathbf{L}}\mathbf{Z} 
\rVert_{\mathrm{F}}.
    \label{eq-theory-8}
    \end{align}
\end{shrinkfix}
The first inequality holds due to \cite{kleindessner2019guarantees} and how we generate the ground-truth graph, and the second inequality holds due to norm inequality. On the other hand, we have 
\begin{shrinkfix}
\begin{align}
\lVert 
\mathbf{Z}\bar{\mathbf{Y}} - \mathbf{Z}\widehat{\mathbf{Y}} \mathbf{V}
\rVert_{\mathrm{F}} = \lVert 
\mathbf{Z}\bar{\mathbf{Y}}\mathbf{V}^{\top} - \mathbf{Z}\widehat{\mathbf{Y}} 
\rVert_{\mathrm{F}}.
    \label{eq-theory-10}
    \end{align}
\end{shrinkfix}
As in Lemma 6 of \cite{kleindessner2019guarantees}, we can choose $\bar{\mathbf{Y}}$ in such a way that $\mathbf{Z}\bar{\mathbf{Y}} = \mathbf{E}$, where $\mathbf{E}_{[i,:]} = \mathbf{E}_{[j,:]}$ if the vertices $i$ and $j$ are in the same cluster and $\lVert \mathbf{E}_{[i,:]} -\mathbf{E}_{[j,:]}\rVert_2 = \sqrt{2K/D}$ if the
vertices $i$ and $j$ are not in the same cluster. Futhermore,  multiplying $\mathbf{E}$ by $\mathbf{V}^{\top}$ will not change the properties of $\mathbf{E}$ since $\mathbf{V}^{\top}$ is a orthogonal matrix. Finally, according to Lemma \ref{lemma-1}, if we let $\delta_k = \sqrt{2K/D}$, then $\widetilde{\mathcal{M}}_k$ in Lemma \ref{lemma-1} is equivalent to $\mathcal{M}_k$. Furthermore, according to Lemma 5.3 in \cite{lei2015consistency}, if $\frac{4(4+2\epsilon)}{ \delta_k^2 }\lVert  {\mathbf{E}} \mathbf{V}^{\top} - {\mathbf{Z}}\widehat{\mathbf{Y}}\rVert_{\mathrm{F}}^2 \leq \frac{D}{K}$, we have
\begin{shrinkfix}
\begin{align}
    \sum_{k=1}^K |\mathcal{M}_k|&\leq \frac{4(4+2\epsilon)}{ \delta_k^2 }\lVert  {\mathbf{E}} \mathbf{V}^{\top} - {\mathbf{Z}}\widehat{\mathbf{Y}}\rVert_{\mathrm{F}}^2 \notag\\
    &\leq \frac{256(4+2\epsilon)K^2}{D(c-d)^2}\lVert \mathbf{Z}^{\top}\bar{\mathbf{L}}\mathbf{Z} - \mathbf{Z}^{\top}\widehat{\mathbf{L}}\mathbf{Z} 
    \rVert_{\mathrm{F}}^2.
\label{eq-theory-11}
    \end{align}
\end{shrinkfix}
Let $C_1 = \frac{256(4+2\epsilon)K^2}{D(c-d)^2}$, we have 
\begin{shrinkfix}
\begin{align}
 \sum_{k=1}^K |\mathcal{M}_k|&\leq  2C_1\underbrace{ \lVert \mathbf{Z}^{\top}\bar{\mathbf{L}}\mathbf{Z} - \mathbf{Z}^{\top}\mathbf{L}^*\mathbf{Z} 
    \rVert_{\mathrm{F}}^2}_{\mathcal{T}_1} + 2 C_1\underbrace{ \lVert \mathbf{Z}^{\top}\mathbf{L}^*\mathbf{Z} - \mathbf{Z}^{\top}\widehat{\mathbf{L}}\mathbf{Z}
    \rVert_{\mathrm{F}}^2}_{\mathcal{T}_2}.
\label{eq-theory-12}
    \end{align}
\end{shrinkfix}
The first term is the difference between the expected Laplacian matrix and the real matrix, which has been derived in \cite{kleindessner2019guarantees}. Specifically, for any $r_2>0$ and some $r_1>0$ satisfying $a \geq r_1 \ln{D}/D$, with probability at least $1-D^{-r_2}$, we have that there exist a constant $C_2(r_1, r_2)$ such that 
\begin{shrinkfix}
\begin{align}
   \mathcal{T}_1 \leq C_2(r_1, r_2) a D\ln{D}.
\label{eq-theory-13}
    \end{align}
\end{shrinkfix}
The second term $\mathcal{T}_2$ of \eqref{eq-theory-12} is the error between the Laplacian estimated by our model and the real one. Bringing \eqref{eq-theory-13} to \eqref{eq-theory-12}, we finally complete the proof.

\bibliographystyle{IEEEtran}
\bibliography{IEEEabrv, references}

\clearpage
\centering{\large{\textbf{Supplementary Materials}}}
\appendices
\justifying

\subsection{Several Extensions to The Proposed Model}

\textbf{1) Improved spectral rotation: } Improved spectral rotation is an improved version of \eqref{eq-form-1}, which is formulated as \cite{zhong2023self}:
\begin{shrinkfix}
\begin{align}
        &\underset{\mathbf{Q}, \mathbf{R}}{\min}\; \lVert \mathbf{Q}(\mathbf{Q}^{\top}\mathbf{Q})^{-\frac{1}{2}} - \mathbf{U}\mathbf{R}\rVert_{\mathrm{F}}^2\notag\\
        &\mathrm{s.t.}\;\mathbf{R}^{\top}\mathbf{R} = \mathbf{I}, \mathbf{Q}\in \mathcal{I}.
    \label{eq-exten-1}
\end{align}
\end{shrinkfix}
The improved spectral rotation can output a discrete label matrix $\mathbf{Q}$ that are closer to $\mathbf{U}\mathbf{R}$  since $\left(\mathbf{Q}(\mathbf{Q}^{\top}\mathbf{Q})^{-\frac{1}{2}}\right)^{\top}\mathbf{Q}(\mathbf{Q}^{\top}\mathbf{Q})^{-\frac{1}{2}} = (\mathbf{U}\mathbf{R})^{\top}\mathbf{U}\mathbf{R} = \mathbf{I}$, i.e., $\mathbf{Q}(\mathbf{Q}^{\top}\mathbf{Q})^{-\frac{1}{2}}$ and $\mathbf{U}\mathbf{R}$ are in the same space \cite{zhong2023self}. If we  employ the improved spectral rotation in \eqref{eq-formulation-final}, the model becomes 
\begin{shrinkfix}
\begin{align}
       & \underset{\mathbf{X},\mathbf{L}, \bm{\upsilon},\mathbf{Y},\mathbf{R},\mathbf{Q}}{\mathrm{min}}\,\, \frac{1} {N}\lVert 
    {\mathbf{\Upsilon}} (\mathbf{X}_o- \mathbf{X})\rVert_{\mathrm{F}}^2 + \frac{\xi}{N}\mathrm{Tr}(\mathbf{X}^{\top}\mathbf{L}\mathbf{X}) + Reg(\mathbf{L}) \notag\\
       &\;\;\;\;\;\;\;\;\;\;\;\;\;\;\;\;+ \mu \mathrm{Tr}(\mathbf{U}^{\top}\mathbf{L}\mathbf{U})+ \gamma \lVert\mathbf{Q}(\mathbf{Q}^{\top}\mathbf{Q})^{-\frac{1}{2}} - \mathbf{U}\mathbf{R}\rVert_{\mathrm{F}}^2\notag\\
       &\;\;\;\;\;\;\;\;\;\;\;\;\;\;\;\;+ \sum_{i=1}^D\frac{1}{\bm{\upsilon}_{[i]}}\notag\\
    & \mathrm{s.t.}\; \mathbf{L}\in \mathcal{L},\bm{\upsilon}>0, \mathbf{U}^{\top}\mathbf{U} = \mathbf{I}, \mathbf{F}^{\top}\mathbf{{U}} = \mathbf{0}, \mathbf{R}^{\top}\mathbf{R} = \mathbf{I}, \mathbf{Q} \in \mathcal{I}.
    \label{eq-exten-2}
\end{align}
\end{shrinkfix}


\textbf{2) Self-weighted feature importance: } 
To improve clustering performance, some works define feature weights to determine the importance of different features in assigning cluster labels \cite{nie2019semi, chen2018semi}. Specifically, given data matrix $\mathbf{X}$, we define weight matrix $\mathbf{\Psi} = \mathrm{diag}(\bm{\psi})\in \mathbb{R}^{N\times N}$, where $\bm{\psi}\in \mathbb{R}^{N}\geq 0$ and $\bm{\psi}^{\top}\mathbf{1} = 1$. The weighted $i$-th feature is $\mathbf{\Psi} \mathbf{X}_{[i,:]}^{\top}$. The weights $\mathbf{\Psi} $ can be directly learned from data. Thus, our model \eqref{eq-formulation-final} with self-weighted feature importance is formulated as 
\begin{shrinkfix}
\begin{align}
       & \underset{\mathbf{W},\bm{\upsilon},\mathbf{U},\mathbf{R},\mathbf{Q},\mathbf{\Psi},\mathbf{X}}{\mathrm{min}}\,\, \frac{1} {N}\lVert 
    {\mathbf{\Upsilon}} (\mathbf{X}_o- \mathbf{X})\rVert_{\mathrm{F}}^2  + \frac{\xi}{2N}\lVert \mathbf{W} \circ \mathbf{P}_{\psi} \rVert_{1,1} \notag\\
       &\;\;\;\;\;\;\;\;\;\;\;\;\;\;\;\;\;\;\;\;\;+ Reg_W(\mathbf{W}) + \mu \mathrm{Tr}(\mathbf{U}^{\top}\mathbf{L}\mathbf{U})+ \gamma \lVert \mathbf{Q} - \mathbf{U}\mathbf{R}\rVert_{\mathrm{F}}^2 \notag\\
       &\;\;\;\;\;\;\;\;\;\;\;\;\;\;\;\;\;\;\;\;\;+ \sum_{i=1}^D\frac{1}{\bm{\upsilon}_{[i]}}\notag\\
    & \mathrm{s.t.}\; \mathbf{W}\in \mathcal{W}, \mathbf{L} = \mathbf{D} - \mathbf{W},  \mathbf{U}^{\top}\mathbf{U} = \mathbf{I}, \mathbf{F}^{\top}\mathbf{{U}} = \mathbf{0}, \mathbf{R}^{\top}\mathbf{R} = \mathbf{I},  \notag\\
    &\;\;\;\;\;\;\mathbf{Q} \in \mathcal{I},(\mathbf{P}_{\psi})_{[ij]} = \left\lVert \mathbf{\Psi} \mathbf{X}_{[i,:]}^{\top} - \mathbf{\Psi} \mathbf{X}_{[j,:]}^{\top}\right\rVert_2^2,\bm{\upsilon}>0,\notag\\ 
     &\;\;\;\;\;\;\bm{\psi}^{\top}\mathbf{1} = 1, \mathbf{\Psi} = \mathrm{diag}(\bm{\psi}).
    \label{eq-exten-3}
\end{align}
\end{shrinkfix}

\textbf{3) Normalized spectral clustering:} The model \eqref{eq-formulation-final} is a unified framework based on unnormalized SC \cite{hagen1992new}. Here, we extend \eqref{eq-formulation-final} to normalized SC \cite{shi2000normalized}. The standard normalized spectral embedding is
\begin{shrinkfix}
\begin{align}
 \underset{\mathbf{U}}{\min}\; \mathrm{Tr}(\mathbf{U}^{\top}\mathbf{L}\mathbf{U}),\;\;\mathrm{s.t.} \mathbf{U}^{\top}\mathbf{D}\mathbf{U} = \mathbf{I}.  
 \label{eq-formulation-2}
\end{align}
\end{shrinkfix}
The fair constraint $\mathbf{F}^{\top}\mathbf{U} = \mathbf{0}$ also holds for normalized SC \cite{kleindessner2019guarantees}. Thus, our model based on the normalized SC is 
\begin{shrinkfix}
\begin{align}
       &  \underset{\mathbf{X},\mathbf{L}, \bm{\upsilon},\mathbf{Y},\mathbf{R},\mathbf{Q}}{\mathrm{min}}\,\, \frac{1} {N}\lVert 
    {\mathbf{\Upsilon}} (\mathbf{X}_o- \mathbf{X})\rVert_{\mathrm{F}}^2 + \frac{\xi}{N}\mathrm{Tr}(\mathbf{X}^{\top}\mathbf{L}\mathbf{X}) + Reg(\mathbf{L}) \notag\\
       &\;\;\;\;\;\;\;\;\;\;\;\;\;\;\;\;+ \mu \mathrm{Tr}(\mathbf{U}^{\top}\mathbf{L}\mathbf{U}) + \gamma \lVert \mathbf{Q} - \mathbf{U}\mathbf{R}\rVert_{\mathrm{F}}^2 + \sum_{i=1}^D\frac{1}{\bm{\upsilon}_{[i]}}\notag\\
    & \mathrm{s.t.}\; \mathbf{L}\in \mathcal{L}, \bm{\upsilon}> 0 ,\mathbf{U}^{\top}\mathbf{D}\mathbf{U} = \mathbf{I}, \mathbf{F}^{\top}\mathbf{{U}} = \mathbf{0}, \mathbf{R}^{\top}\mathbf{R} = \mathbf{I}, \mathbf{Q} \in \mathcal{I}.
 \label{eq-formulation-3}
\end{align}
\end{shrinkfix}

\textbf{4) Individual fairness:} Our model is based on group fairness, which induces the fairness constraint $\mathbf{F}^{\top}\mathbf{U} = \mathbf{0}$. The work \cite{gupta2021protecting} introduces individual fairness into SC, which induces a new fairness constraint $\mathbf{M}(\mathbf{I} - \frac{1}{D}\mathbf{1}\mathbf{1}^{\top})\mathbf{U} = \mathbf{0}$, where $\mathbf{M}\in\mathbb{R}^{D\times D}$ is a graph representing individual sensitive attributes. Our unified model based on individual fairness is 
\begin{shrinkfix}
\begin{align}
       &  \underset{\mathbf{X},\mathbf{L}, \bm{\upsilon},\mathbf{Y},\mathbf{R},\mathbf{Q}}{\mathrm{min}}\,\, \frac{1} {N}\lVert 
    {\mathbf{\Upsilon}} (\mathbf{X}_o- \mathbf{X})\rVert_{\mathrm{F}}^2 + \frac{\xi}{N}\mathrm{Tr}(\mathbf{X}^{\top}\mathbf{L}\mathbf{X}) + Reg(\mathbf{L}) \notag\\
       &\;\;\;\;\;\;\;\;\;\;\;\;\;\;\;\;+ \mu \mathrm{Tr}(\mathbf{U}^{\top}\mathbf{L}\mathbf{U}) + \gamma \lVert \mathbf{Q} - \mathbf{U}\mathbf{R}\rVert_{\mathrm{F}}^2+\sum_{i=1}^D\frac{1}{\bm{\upsilon}_{[i]}}\notag\\
    & \mathrm{s.t.}\; \mathbf{L}\in \mathcal{L}, \bm{\upsilon}> 0 ,\mathbf{U}^{\top}\mathbf{U} = \mathbf{I}, \mathbf{M}\left(\mathbf{I} - \frac{1}{D}\mathbf{1}\mathbf{1}^{\top}\right)\mathbf{U} = \mathbf{0},\notag\\ &\;\;\;\;\;\;\mathbf{R}^{\top}\mathbf{R} = \mathbf{I}, \mathbf{Q}, \in \mathcal{I}.
 \label{eq-formulation-3-1}
\end{align}
\end{shrinkfix}

\subsection{The Complete Algorithm Flow for Updating  \eqref{eq-opt-2-1}}
We use the algorithm in \cite{saboksayr2021accelerated} to solve problem \eqref{eq-opt-2-1}. The complete algorithm flow is presented in Algorithm \ref{alg:update_W}.

\begin{algorithm}[t] 
\caption{The algorithm for problem \eqref{eq-opt-2-1}} 
\begin{algorithmic}[1] 
\REQUIRE  $\beta, \mathbf{p}$, set $L = \frac{D-1}{2\beta}$\\
\ENSURE 
The learned graph ${\mathbf{w}}$\\
\STATE Initialize $\eta^{(1)} = 1$ and $\bm{\omega}^{(1)} = \mathbf{r}^{(0)} \in\mathbb{R}^D$ at random
\FOR{$t = 1, 2,..., $}

\STATE $\bar{\mathbf{w}}^{(t)} = \max\left(\frac{\mathbf{S}^{\top}\bm{\omega}^{(t)} - 2\mathbf{p}}{4\beta},0\right)$

\STATE $\mathbf{v}^{(t)} = \frac{\mathbf{S}\bar{\mathbf{w}}^{(t)} - L\bm{\omega}^{(t)} + \sqrt{(\mathbf{S}\bar{\mathbf{w}}^{(t)} - L\bm{\omega}^{(t)})^2 + 4 L\mathbf{1}} }{2}$

\STATE $\mathbf{r}^{(t)} = \bm{\omega}^{(t)} - L^{-1}\left( \mathbf{S}\bar{\mathbf{w}}^{(t)} - \mathbf{v}^{(t)}\right)$

\STATE $\eta^{(t+1)} = \frac{1+\sqrt{1+4(\eta^{(t)})^2}}{2}$

\STATE $\bm{\omega}^{(t+1)} = \mathbf{r}^{(t)} + \left(\frac{\eta^{(t)}-1}{\eta^{(t+1)}}\right) \left(\mathbf{r}^{(t)}- \mathbf{r}^{(t-1)}\right)$

\ENDFOR
\RETURN $\mathbf{w} = \max\left(\frac{\mathbf{S}^{\top}\mathbf{r}^{(t)} - 2\mathbf{p}}{4\beta},0\right)$

\end{algorithmic}
\label{alg:update_W} 
\end{algorithm}

\subsection{ The Formulation and  Algorithm for FJGSED}
The model FJGSED is formulated as 
\begin{shrinkfix}
\begin{align}
&\underset{\mathbf{W},\mathbf{U},\mathbf{Q}, \mathbf{R}}{\min}\sum_{i,j = 1}^D \lVert \mathbf{X}_{[i,:]} -\mathbf{X}_{[j,:]}\rVert_{2}^2 \mathbf{W}_{[i,j]} + \beta_{J} \mathbf{W}_{[i,j]}^2 \notag\\ 
&\;\;\;\;\;\;\;\;\;\;\;\;+ {\mu_J} \mathrm{Tr}(\mathbf{U}^{\top}\mathbf{L}\mathbf{U}) + \gamma_J \lVert \mathbf{Q} -  \mathbf{U} \mathbf{R} \rVert_{\mathrm{F}}^2\notag\\
&\mathrm{s.t.}\; \mathbf{W}\mathbf{1} = \mathbf{1}, \mathbf{W} \geq 0, \mathbf{L} = \mathbf{D} - \mathbf{W}, \mathbf{U}^{\top}\mathbf{U} = \mathbf{I}, \mathbf{F}^{\top}\mathbf{U} = \mathbf{0}\notag\\
&\;\;\;\;\; \;\;\mathbf{R}^{\top}\mathbf{R} = \mathbf{I}, \mathbf{Q} \in \mathcal{I}.
\label{eq-JGSCE}
\end{align}
\end{shrinkfix}
The framework of our algorithm for sloving \eqref{eq-JGSCE} is the same as Algorithm \ref{alg:1}, which alternately updates $\mathbf{W}, \mathbf{U}$, $\mathbf{R}$ and $\mathbf{ Q}$. The update of $\mathbf{U}$, $\mathbf{R}$, and $\mathbf{Q}$ is the same as Algorithm \ref{alg:1}. The main difference is updating $\mathbf{W}$/$\mathbf{L}$, and hence we discuss the update of $\mathbf{W}$ here. The corresponding sub-problem is
\begin{align}
\underset{\mathbf{W}}{\min}& \;\;\sum_{i,j}^D \lVert \mathbf{X}_{[i,:]} -\mathbf{X}_{[j,:]}\rVert_{2}^2 \mathbf{W}_{[i,j]} + \beta_{J} \mathbf{W}_{[i,j]}^2 + {\mu_J} \mathrm{Tr}(\mathbf{U}^{\top}\mathbf{L}\mathbf{U}) \notag\\
\mathrm{s.t.}\; & \mathbf{W}_{[i,:]}\mathbf{1} = 1, \mathbf{W}_{[i,:]} \geq 0, \mathbf{L} = \mathbf{D} - \mathbf{W}.
\label{eq-fairJ-2}
\end{align}
We can rewrite the problem as 
\begin{align}
&\underset{\mathbf{W}\mathbf{1} = \mathbf{1}, \mathbf{W} \geq 0}{\min} \sum_{i,j = 1}^D\; \lVert \mathbf{X}_{[i,:]} - \mathbf{X}_{[j,:]} \rVert_2^2 \mathbf{W}_{[ij]} \notag\\
&\;\;\;\;\;\;\;\;\;\;\;\;\;\;\;\;+  \frac{\mu_{J}}{2} \lVert \mathbf{U}_{[i,:]} - \mathbf{U}_{[j,:]} \rVert_2^2 \mathbf{W}_{[ij]}  
+ \beta_J \mathbf{W}^2_{[ij]}   
\label{eq-fairJ-3}
\end{align}
Let $ \mathbf{C}_{[ij]} = \lVert \mathbf{X}_{[i,:]} - \mathbf{X}_{[j,:]} \rVert_2^2 +  \frac{\mu_{J}}{2} \lVert \mathbf{U}_{[i,:]} - \mathbf{U}_{[j,:]} \rVert_2^2$, and the problem \eqref{eq-fairJ-3} can be optimized for each row, i.e. for $i = 1,...,D$,
\begin{align}
&\underset{\mathbf{W}_{[i,:]}}{\min} \sum_{j = 1}^D\;  \mathbf{C}_{[ij]}\mathbf{W}_{[ij]} + \beta_J \mathbf{W}^2_{[ij]}\;\;  \mathrm{s.t.}\; \mathbf{W}_{[i,:]}\mathbf{1} = 1, \mathbf{W}_{[i,:]} \geq 0\notag, \\
\Rightarrow & \underset{\mathbf{W}_{[i,:]}}{\min} \left\lVert \mathbf{W}_{[i,:]} + \frac{1}{2\beta_J} \mathbf{C}_{[i,:]}\right\rVert_2^2  \;\;\mathrm{s.t.}\; \mathbf{W}_{[i,:]}\mathbf{1} = 1, \mathbf{W}_{[i,:]} \geq 0.
\label{eq-fairJ-4}
\end{align}
It defines a squared Euclidean distance on a simplex constraint. Inspired by  \cite{peng2023jgsed}, we update $\mathbf{W}_{[i,:]}$ as 
\begin{align}
\mathbf{W}_{[i,j]} = \max\left(\frac{\mathbf{C}_{[i,l+1]}- \mathbf{C}_{[i,j]}}{l\mathbf{C}_{[i,l+1]} - \sum_{j=1}^l\mathbf{C}_{[i,j]}},0\right) ,
\label{eq-fairJ-5}
\end{align}
where $l$ is a hyper-parameter determining the number of neighbor nodes of the learned graphs. We select $l$ instead of $\beta_J$ as the model parameters.

We iteratively update  $\mathbf{W}$, $\mathbf{U}, \mathbf{R}$, and $\mathbf{Q}$ until convergence. The complete algorithm is shown in Algorithm \ref{alg:FJGSED}

\begin{algorithm}[t] 
\caption{The algorithm for FJGSED} 
\begin{algorithmic}[1] 
\REQUIRE ~~\\ 
 $\mathbf{X}$, the number of clusters $K$, parameters $l, \mu_J, \gamma_J$\\
\ENSURE ~~\\ 
The learned graph $\mathbf{W}$, the cluster indicator matrix  $\mathbf{Q}$
\STATE Initialize  $\mathbf{W}$, $\mathbf{U}$, $\mathbf{Q} $, and $\mathbf{R}$\\

\WHILE{not converged}

\STATE Update $\mathbf{W}$ via \eqref{eq-fairJ-5}

\STATE Update $\mathbf{Y}$ by solving \eqref{eq-opt-4}, and let $\mathbf{U} = \mathbf{Z}\mathbf{Y}$

\STATE Update $\mathbf{R}$ as $\mathbf{R} = \mathbf{\Theta}_{R}\mathbf{\Theta}_{L}^{\top}$

\STATE Update $\mathbf{Q}$ via \eqref{eq-opt-8}

\ENDWHILE
\end{algorithmic}
\label{alg:FJGSED} 
\end{algorithm}

\begin{algorithm}[t] 
\caption{The algorithm for FSRSC} 
\begin{algorithmic}[1] 
\REQUIRE ~~\\ 
 $\mathbf{X}$, the number of clusters $K$, parameters $\gamma_U, \mu_J, \gamma_J$\\
\ENSURE ~~\\ 
The learned graph $\mathbf{W}$, the cluster indicator matrix  $\mathbf{Q}$
\STATE Initialize  $\mathbf{W}$, $\mathbf{U}$, $\mathbf{\Gamma}$, $\mathbf{Q} $, and $\mathbf{R}$\\

\WHILE{not converged}

\STATE Update $\mathbf{A}$ via \eqref{eq-fairUSPC-6}

\STATE Update $\mathbf{W}$ via \eqref{eq-fairUSPC-11}

\STATE $\mathbf{W} = \max(\mathbf{W}, 0) $ and let $\mathrm{diag(}\mathbf{W}) = \mathbf{0}$.

\STATE $\mathbf{W} =\frac{1}{2}(\mathbf{W}^{\top} + \mathbf{W})$.

\STATE Update $\mathbf{\Gamma}$ as $\mathbf{\Gamma} = \mathbf{\Gamma} + \gamma_U(\mathbf{A}- \mathbf{W})$

\STATE Update $\mathbf{U}$ by solving \eqref{eq-opt-4}

\STATE Update $\mathbf{R}$ as $\mathbf{R} = \mathbf{\Theta}_{R}\mathbf{\Theta}_{L}^{\top}$

\STATE Update $\mathbf{Q}$ via \eqref{eq-opt-8}

\ENDWHILE
\end{algorithmic}
\label{alg:FSRSC} 
\end{algorithm}

\subsection{ The Formulation And Algorithm For FSRSC}
The model FSRSC is formulated as 
\begin{shrinkfix}
\begin{align}
&\underset{\mathbf{W},\mathbf{U},\mathbf{Q}, \mathbf{R}} {\min} \lVert \mathbf{X} - \mathbf{W}^{\top} \mathbf{X} \rVert_{\mathrm{F}}^2 + \alpha_{U}  \lVert \mathbf{W} \rVert_{1,1}  + {\mu_U} \mathrm{Tr}(\mathbf{U}^{\top}\mathbf{L}\mathbf{U}) \notag\\ 
&\;\;\;\;\;\;\;\;\;\;\;\;\;+ \gamma_U \lVert \mathbf{Q} -  \mathbf{U} \mathbf{R} \rVert_{\mathrm{F}}^2\notag\\
&\mathrm{s.t.}\; \mathbf{W} \in \mathcal{W},  \mathbf{U}^{\top}\mathbf{U} = \mathbf{I}, \mathbf{F}^{\top}\mathbf{U} = \mathbf{0}, \mathbf{R}^{\top}\mathbf{R} = \mathbf{I}, \mathbf{Q} \in \mathcal{I}.
\label{eq-FSRSC}
\end{align}
\end{shrinkfix}
The framework of our algorithm for \eqref{eq-FSRSC} is the same as Algorithm \ref{alg:1}, which alternately updates $\mathbf{W}, \mathbf{U}$, $\mathbf{R}$ and $\mathbf{ Q}$. The update of $\mathbf{U}$, $\mathbf{R}$, and $\mathbf{Q}$ is the same as Algorithm \ref{alg:1}. The main difference is updating $\mathbf{W}$, and hence we discuss the update of $\mathbf{W}$ here. The corresponding sub-problem is
\begin{align}
\underset{\mathbf{W}} {\min} &\lVert \mathbf{X} - \mathbf{W}^{\top} \mathbf{X} \rVert_{\mathrm{F}}^2 + \alpha_{U}  \lVert \mathbf{W} \rVert_{1,1}  + {\mu_U} \mathrm{Tr}(\mathbf{U}^{\top}\mathbf{L}\mathbf{U}) \notag\\ 
\mathrm{s.t.}\; &\mathrm{diag}(\mathbf{W}) = \mathbf{0}, \mathbf{W}^{\top} = \mathbf{W}, \mathbf{W} \geq 0, \mathbf{L} = \mathbf{D} - \mathbf{W}.
\label{eq-fairUSPC-2}
\end{align}
We use the augmented Lagrange multiplier (ALM) type of method to solve the problem \eqref{eq-fairUSPC-2}. Let us first introduce an auxiliary variable $ \mathbf{A}$ here
\begin{align}
\underset{\mathbf{W}} {\min} &\lVert \mathbf{X} - \mathbf{W}^{\top} \mathbf{X} \rVert_{\mathrm{F}}^2 + \alpha_{U}  \lVert \mathbf{A} \rVert_{1,1}  + {\mu_U} \mathrm{Tr}(\mathbf{U}^{\top}\mathbf{L}\mathbf{U}) \notag\\ 
\mathrm{s.t.}\; &\mathrm{diag}(\mathbf{W}) = \mathbf{0}, \mathbf{W}^{\top} = \mathbf{W}, \mathbf{W} \geq 0, \mathbf{L} = \mathbf{D} - \mathbf{W}, \mathbf{A} = \mathbf{W}.
\label{eq-fairUSPC-3}
\end{align}
The augmented Lagrangian function of the problem is 
\begin{align}
&Lag(\mathbf{W}, \mathbf{A}, \mathbf{\Gamma}) \notag \\
=& \lVert \mathbf{X} - \mathbf{W}^{\top} \mathbf{X} \rVert_{\mathrm{F}}^2 + \alpha_{U}  \lVert \mathbf{A} \rVert_{1,1}  + {\mu_U} \mathrm{Tr}(\mathbf{U}^{\top}\mathbf{L}\mathbf{U}) \notag\\
&+ \frac{\gamma_U}{2} \left\lVert  \mathbf{A} -  \mathbf{W} + {\mathbf{\Gamma}}/{\gamma_U}\right\rVert_2^2.,
\label{eq-fairUSPC-4}
\end{align}
where $\gamma_U$ is the Lagrangian constant. In the ALM algorithm, we update $\mathbf{W}, \mathbf{\Gamma}$, and $\mathbf{A}$ in an alternating manner. We first fix $\mathbf{W}, \mathbf{\Gamma}$ and update $\mathbf{A}$.  Let $\mathbf{J} = \mathbf{W} - \frac{\mathbf{\Gamma}}{\gamma_U}$, the optimization problem is 
\begin{align}
\underset{\mathbf{A}} {\min} \;\;\alpha_{U}  \lVert \mathbf{A} \rVert_{1,1} +   \frac{\gamma_U}{2} \left\lVert  \mathbf{A} -  \mathbf{J}\right\rVert_2^2,
\label{eq-fairUSPC-5}
\end{align}
which can be updated elementwise as 
\begin{align}
\mathbf{A}_{[ij]} = \max\left(|\mathbf{J}_{[ij]}| - \frac{\alpha_{U}}{\gamma_U} ,0\right) \mathrm{sign}\left(\mathbf{J}_{[ij]}\right).
\label{eq-fairUSPC-6}
\end{align}
Then, we fix  $\mathbf{A}, \mathbf{\Gamma}$ and update $\mathbf{W}$. Let $\widetilde{\mathbf{J}} =  \mathbf{A} + \frac{\mathbf{\Gamma}}{\gamma_U}$, and we have
\begin{align}
&\underset{\mathbf{W}}{\min} \; \lVert \mathbf{X} - \mathbf{W}^{\top} \mathbf{X} \rVert_{\mathrm{F}}^2 + {\mu_U} \mathrm{Tr}(\mathbf{U}^{\top}\mathbf{L}\mathbf{U}) + \frac{\gamma_U}{2} \lVert \mathbf{W} - \widetilde{\mathbf{J}} \rVert_{\mathrm{F}}^2, \notag\\
&\mathrm{s.t.}\; \mathrm{diag}(\mathbf{W}) = \mathbf{0}, \mathbf{W}^{\top} = \mathbf{W}, \mathbf{W} \geq 0,
\label{eq-fairUSPC-7}
\end{align}
which is equivalent to 
\begin{align}
&\underset{\mathbf{W}}{\min} \, g(\mathbf{W})\notag\\ 
:=&\underset{\mathbf{W}}{\min} \; 
\mathrm{Tr}\left( \mathbf{W}^{\top}\mathbf{X}\mathbf{X}^{\top}\mathbf{W} - 2 \mathbf{X}\mathbf{X}^{\top}\mathbf{W}^{\top}\right)  + \frac{\mu_{U}}{2} \lVert\mathbf{W} \circ \mathbf{P}_{U} \rVert_{1,1}\notag\\
&+ \frac{\gamma_U}{2} \mathrm{Tr}\left(\mathbf{W}^{\top}\mathbf{W} - 2 \widetilde{\mathbf{J}}^{\top}\mathbf{W} \right)  \notag\\
&\mathrm{s.t.}\; \mathrm{diag}(\mathbf{W}) = \mathbf{0},\mathbf{W}^{\top} = \mathbf{W}, \mathbf{W} \geq 0,
\label{eq-fairUSPC-8}
\end{align}
where $\mathbf{P}_{U} $ is the pair-wise distance matrix of $\mathbf{U}$. For every column of $\mathbf{W}$, we have the following problem 
\begin{align}
&\underset{\mathbf{W}_{[:,i]}}{\min} \, g(\mathbf{W}_{[:,i]})\notag\\
=& \underset{\mathbf{W}_{[:,i]}}{\min} \,\mathbf{W}_{[:,i]}^{\top} \left(\frac{\gamma_U}{2}\mathbf{I} + \mathbf{X}\mathbf{X}^{\top} \right) \mathbf{W}_{[:,i]}\notag\\
&+ \left(\frac{\mu_U }{2} (\mathbf{P}_{U})^{\top}_{[:,i]} - \gamma_U \widetilde{\mathbf{J}}^{\top}_{[:,i]} - 2 (\mathbf{X}\mathbf{X}^{\top})_{[i,:]} \right)\mathbf{W}_{[:,i]}.
\label{eq-fairUSPC-9}
\end{align}
We calculate the derivative of $g(\mathbf{W}_{[:,i]})$ and have 
\begin{align}
&\nabla_{\mathbf{W}_{[:,i]}}g(\mathbf{W}_{[:,i]}) \notag \\
= &2 \left(\frac{\gamma_U}{2}\mathbf{I} + \mathbf{X}\mathbf{X}^{\top} \right)\mathbf{W}_{[:,i]} + \frac{\mu_U}{2} (\mathbf{P}_{U})_{[:,i]} - \gamma_U \widetilde{\mathbf{J}}_{[:,i]} - 2 (\mathbf{X}\mathbf{X}^{\top})_{[:,i]}.
\label{eq-fairUSPC-10}
\end{align}
Let $\nabla_{\mathbf{W}_{[:,i]}}g(\mathbf{W}_{[:,i]}) = \mathbf{0}$, and we obtain 
\begin{align}
\mathbf{W}_{[:,i]} =\left({\gamma_U}\mathbf{I} + 2\mathbf{X}\mathbf{X}^{\top} \right)^{-1}\left( \gamma_U \widetilde{\mathbf{J}}_{[:,i]} + 2 (\mathbf{X}\mathbf{X}^{\top})_{[:,i]} - \frac{\mu_U}{2} (\mathbf{P}_{U})_{[:,i]} \right).
\label{eq-fairUSPC-11}
\end{align}
After updating all columns of $\mathbf{W}$, we project $\mathbf{W}$ into the constraints $\mathrm{diag}(\mathbf{W}) = \mathbf{0}, \mathbf{W}^{\top} = \mathbf{W}, \mathbf{W} \geq 0$. 

Finally, we fix $\mathbf{W}, \mathbf{A}$ and update $\mathbf{\Gamma}$, i.e., $\mathbf{\Gamma} = \mathbf{\Gamma} + \gamma_U(\mathbf{A}- \mathbf{W})$.

After updating  $\mathbf{W}, \mathbf{A}$, and $\mathbf{\Gamma}$, we then update $\mathbf{U}, \mathbf{R}$, and $\mathbf{Q}$ by following Algorithm \ref{alg:1}. We iteratively update  $\mathbf{W}, \mathbf{A}$, $\mathbf{\Gamma}$,  $\mathbf{U}, \mathbf{R}$, and $\mathbf{Q}$ until convergence. The complete algorithm flow is shown in Algorithm \ref{alg:FSRSC}








\end{document}